\documentclass[11pt, letterpaper]{article}

\usepackage{geometry}
\usepackage{amsmath,graphicx}
\usepackage{amsfonts}
\usepackage{amssymb}
\usepackage{mathtools}
\usepackage{amsthm}
\usepackage[utf8]{inputenc} % allow utf-8 input
\usepackage[english]{babel}
\usepackage{multicol}
\usepackage{xcolor}
\usepackage{amsmath}
\usepackage{hyperref}       % hyperlinks
\usepackage{fancyhdr}
\usepackage{tcolorbox}
\definecolor{darkred}{RGB}{150,0,0}
\definecolor{darkgreen}{RGB}{0,150,0}
\definecolor{darkblue}{RGB}{0,0,150}
\usepackage{appendix}
\hypersetup{colorlinks=true, linkcolor=red, citecolor=darkgreen, urlcolor=darkblue}

\newtheorem{theorem}{Theorem}
\newtheorem{lemma}{Lemma}
\newtheorem{corollary}{Corollary}[theorem]
\newtheorem{proposition}{Proposition}[theorem]
\theoremstyle{remark}
\newtheorem{remark}{Remark}
\theoremstyle{definition}
\newtheorem{definition}{Definition}[section]
\theoremstyle{remark}
\newtheorem{assumption}{Assumption}

\theoremstyle{remark}

\usepackage{hyperref}

\def\x{{\boldsymbol{x}}}

\newcommand{\vp}{\vspace{2pt}}

\newcommand{\Sigmab}{\boldsymbol\Sigma}
\newcommand{\Nn}{\mathcal{N}}

\newcommand{\normeta}{\Vert \boldsymbol{\eta} \Vert_2}
\newcommand{\normbeta}{\Vert \boldsymbol{\beta} \Vert_2}
\newcommand{\betab}{\boldsymbol{\beta}}
\newcommand{\etab}{\boldsymbol{\eta}}
\newcommand{\lbdb}{\boldsymbol{\lambda}}

\newcommand{\normlbd}{\Vert \boldsymbol{\lambda} \Vert_1}

\newcommand{\normlbdtwo}{\Vert \boldsymbol{\lambda} \Vert_2}

\newcommand{\binormlbd}{\Vert \boldsymbol{\lambda} \Vert_{-1}}

\newcommand{\etals}{\hat{\boldsymbol{\eta}}_{\rm LS}}

\newcommand{\etareg}{\hat{\boldsymbol{\eta}}_{\tau}}

\newcommand{\etaavg}{\hat{\boldsymbol{\eta}}_{\text{Avg}}}

\newcommand{\etasvm}{\hat{\boldsymbol{\eta}}_{\rm SVM}}

\newcommand{\boldeta}{\hat{\boldsymbol{\eta}}}

\newcommand{\R}{\mathbb{R}}

\newcommand{\Prob}{\mathbb{P}}

\newcommand{\boldu}{{\boldsymbol{U}}}

\newcommand{\boldutau}{\boldsymbol{U}_{\tau}}

\newcommand{\boldutaunoone}{\boldsymbol{U}_{-1,\tau}}

\newcommand{\boldy}{\boldsymbol{y}}

\newcommand{\boldd}{\boldsymbol{d}}

\newcommand{\boldz}{\boldsymbol{z}}

\newcommand{\boldx}{\boldsymbol{x}}

\newcommand{\boldcapx}{\boldsymbol{X}}

\newcommand{\boldq}{\boldsymbol{q}}

\newcommand{\boldqcap}{\boldsymbol{Q}}

\newcommand{\boldetastar}{\boldsymbol{\eta}^{*}}

\newcommand{\suminner}{\sum_{i=1}^p\lambda_i\beta_i^2}

\newcommand{\sumlbdsq}{\sum_{i=1}^p\lambda_i^2}

\newcommand{\dprimen}{d^{'}(n)}

\newcommand{\tildy}{{\boldsymbol{y}_c}}

\newcommand{\tilds}{{s}_c}

\newcommand{\tildh}{{h}_c}

\newcommand{\tildg}{{g}_{c,i}}

\newcommand{\tildetals}{{\hat{\boldsymbol{\eta}}}_{\rm LS}}

\newcommand{\tildetasvm}{{\hat{\boldsymbol{\eta}}}_{\rm SVM}}

\newcommand{\tildtilds}{s_{cc}}

\newcommand{\xxplustau}{\boldsymbol{X}\boldsymbol{X}^T + \tau\boldsymbol{I}}

\newcommand{\evente}{\mathcal{E}}

\newcommand{\newadd}{}
\newcommand{\new}{}

\usepackage[nice]{nicefrac}

\newcommand{\y}{\boldsymbol{y}}
\newcommand{\A}{\boldsymbol{A}}
\newcommand{\Q}{\boldsymbol{Q}}
\newcommand{\X}{\boldsymbol{X}}
\newcommand{\h}{\boldsymbol{h}}
\newcommand{\g}{\boldsymbol{g}}
\newcommand{\w}{\boldsymbol{w}}
\newcommand{\q}{\boldsymbol{q}}
\newcommand{\ub}{\boldsymbol{u}}

\usepackage{natbib}
\begin{document}

% Title.
% ------
\title{Binary Classification of Gaussian Mixtures: \\ Abundance of Support Vectors, Benign Overfitting \\ and Regularization \thanks{A preliminary version of this work \cite{wang2020benign} is presented as ICASSP 2021.}}

%
% Single address.
% ---------------
\author{Ke Wang \thanks{Department of Statistics and Applied Probability, University of California, Santa Barbara, CA 93106 USA ({kewang01@ucsb.edu}).} \and  Christos Thrampoulidis\thanks{Department of Electrical and Computer Engineering, University of British Columbia, Vancouver, BC V6T 1Z4 Canada ({cthrampo@ece.ubc.ca}).
Department of Electrical and Computer Engineering, University of California, Santa Barbara, CA 93106 USA ({cthrampo@ucsb.edu}).}}%and is also affiliated with the University of California Santa Barbara. This research was partially supported by the NSF under Grant Numbers CCF-2009030 and 1934641. The authors would like to thank Dr. Yuan Cao for very helpful comments that helped improve this paper.}}
\date{}
  
%\address{$^{\star}$University of California, Santa Barbara, Department of Statistics and Applied Probability.\\
%$^{\dagger}$University of British Columbia, Vancouver, Department of Electrical and Computer Engineering.}
%
% For example:
% ------------
%\address{School\\
%	Department\\
%	Address}
%
% Two addresses (uncomment and modify for two-address case).
% ----------------------------------------------------------
%\twoauthors
%  {A. Author-one, B. Author-two\sthanks{Thanks to XYZ agency for funding.}}
%	{School A-B\\
%	Department A-B\\
%	Address A-B}
%  {C. Author-three, D. Author-four\sthanks{The fourth author performed the work
%	while at ...}}
%	{School C-D\\
%	Department C-D\\
%	Address C-D}
%
%\ninept
%
%\tenpt
%\begin{titlepage}
\maketitle
%\end{titlepage}
\begin{abstract}
Deep neural networks generalize well despite being exceedingly overparameterized and being trained without explicit regularization. This curious phenomenon has inspired extensive research activity in establishing its statistical principles: Under what conditions is it observed? How do these depend on the data and on the training algorithm? When does regularization benefit generalization? While such questions remain wide open for deep neural nets, recent works have attempted gaining insights by studying simpler, often linear, models. Our paper contributes to this growing line of work by examining binary linear classification under a generative Gaussian mixture model \new{in which the feature vectors take the form $\x=\pm\etab+\boldsymbol{q}$, where for a mean vector $\etab$ and feature noise $\boldsymbol{q}\sim\mathcal{N}(0,\Sigmab)$}. 
%Better understanding the statistical principles behind this so called benign-overfitting phenomenon remains a major open challenge in modern learning theory. 
%Recently there has been remarkable progress towards understanding benign-overfitting in simpler models, such as linear regression and, even more recently, linear classification. This paper 
%Some core questions include how the loss function and regularization affect the classification error. This paper tries to answer those questions and studies benign-overfitting for data generated from a popular binary Gaussian mixtures model (GMM). 
Motivated by recent results on the implicit bias of gradient descent, we study both max-margin SVM classifiers (corresponding to logistic loss) and min-norm interpolating classifiers (corresponding to least-squares loss). 
First, we leverage an idea introduced in [V. Muthukumar et al., arXiv:2005.08054, (2020)] to relate the SVM solution to the \new{min-norm} interpolating solution. Second, we derive novel non-asymptotic bounds on the classification error of the latter. Combining the two, we present novel  sufficient conditions on the covariance spectrum and on the signal-to-noise ratio (SNR) \new{$SNR=\nicefrac{\|\etab\|_2^4}{\etab^T\Sigmab\etab}$}
\new{under which interpolating estimators achieve asymptotically optimal performance as overparameterization increases. Interestingly, our results extend to a noisy model with constant probability noise flips.}
Contrary to previously studied discriminative data models, our results emphasize the crucial role of the SNR and its interplay with the data covariance. \new{Finally, via a combination of analytical arguments and numerical demonstrations we identify conditions under which the interpolating estimator performs better than corresponding regularized estimates.}
% We corroborate our theoretical findings with numerical simulations.
% The logistic loss and square loss are two commonly used loss functions at the training phase for classification. We connect the hard-margin support vector machine (SVM) solution, obtained by using the logistic loss, to the least-square (LS) solution, obtained by using the square loss, in the overparameterized regime for one discriminative model, the Gaussian Mixture Model. We show those two estimators are equivalent when the overparameterization is sufficient and the leaning task is hard enough. Then we derive a non-asymptotic bound for the 0-1 generalization error for the LS solution. Based on those results, we identify the regime that all the data points are support vectors yet still generalizes well.
\end{abstract}
%
%\begin{keywords}
%Overparameterization, SVM, Min-norm interpolation, Generalization error, Separable data, Regularization
%\end{keywords}
%

%\begin{AMS}
%  62H30, 68T10
%\end{AMS}

\section{Introduction} \label{sec-intro}

%\noindent\textbf{Motivation and related works.} 
\subsection{\new{Motivation}}
Deep-learning models are increasingly more complex. They are designed with a huge number of parameters that far exceed the size of typical training data sets and training is often completed without any explicit regularization  \citep{krizhevsky2012imagenet,montufar2014number,poggio2017and,goodfellow2016deep}. As a consequence, after training, the models perfectly fit (or, so called interpolate) the data. Classical statistical wisdom suggests that such interpolating models overfit and as such they generalize poorly, e.g. \citet{hastie2009elements}. But, the reality of modern deep-learning practice is very different: such overparameterized learning architectures achieve state-of-the-art generalization performance despite interpolating the data \citep{zhang2016understanding,belkin2018reconciling,nakkiran2019deep}. Interestingly, similar empirical findings, albeit in much simpler learning settings have been recorded in the  literature even before the era of deep learning \citep{vallet1989linear,opper1990ability,duin2000classifiers}; see discussion in \citet{loog2020brief}. Empirical observations like these raise a series of important questions \citep{duin2000classifiers,zhang2016understanding,belkin2018reconciling,belkin2018understand}: \emph{Why and when are larger models better? What is the role of the training algorithm in this process? Can infinite overparameterization result in better generalization than any finite number of parameters or even training with explicit regularization?} Answering these questions is considered one of the main challenges in modern learning theory and has attracted significant research attention over the past couple of years or so, e.g., \citet{belkin2018reconciling,belkin2018overfitting,mei2019generalization,hastie2019surprises,liang2019risk,liao2020random,deng2019model,ba2019generalization,muthukumar2020classification,yang2020rethinking,chatterji2020does}. %\cite{belkin2018reconciling} called this phenomenon \emph{double descent} 

Among the earliest attempts towards analytically investigating the question ``why do overparameterized models generalize well?" 
%understanding the principles behind the above phenomenon --- termed \emph{benign overfitting} in \citet{bartlett2020benign} --- 
focused on
%unchanged-By now there is rather long list of works explaining the phenomenon of benign overfitting in 
linear-regression including both asymptotic and non-asymptotic analyses \citep{hastie2019surprises,belkin2019two,muthukumar2020harmless,tsigler2020benign}.  While certainly a simplified model, this is a natural first step towards gaining insights about more complex models. 
%See Section \ref{sec:lit} for a discussion of these references, as well as, related work that preceded the deep-learning era \cite{loog2020brief}. 
%\ct{Add a sentence about [2], explaining benign overfitting as introduced in [2].}
\new{Closest to our work, \citet{bartlett2020benign} derived non-asymptotic bounds on the {squared prediction risk} of the min-norm linear interpolator for a  linear regression model with additive Gaussian noise and (sub)-Gaussian covariates.}
%\kw{I think they consider $E[(y-x^T\hat{\theta})^2 - (y-x^T\theta^{*})^2]$ called excess risk, where $\theta^{*}$ is the optimal estimator. I am not quite sure if this is the same as mse. }
\new{They subsequently used these bounds and identified conditions on the spectrum of the data covariance such that the risk asymptotically approaches the optimal Bayes error despite perfectly fitting to noisy data. This behavior was termed ``\emph{benign overfitting}" in their paper and the terminology has already been widely adapted in the literature.} 
%Our work is motivated by \citet{bartlett2020benign}, but otherwise differs in four  important aspects as outlined next.}

% \newadd{\citet{bartlett2020benign} shows that sufficient overparameterization and the distributions of eigenvalues are essential to achieve the best possible prediction accuracy. }

A step further in the direction of understanding generalization in overparameterized regimes is the study of linear classification models, since arguably most deep learning success stories apply to classification settings. 
Classification is not only more relevant, but also typically harder to analyze. The challenge is that even in linear settings, the solution to logistic loss minimization  is \emph{not} given in closed form. This is to be contrasted to the solution to least-squares minimization typically used in regression (e.g. \citep{bartlett2020benign,hastie2019surprises}). 
 As such, central questions have remained largely unexplored until very recently. %There are still many questions need to be answered for binary classification in the overparameterized regime. For example, how loss functions affect the classification risk? Specifically, it is interesting to compare two commonly used classifiers, the support-vector machines (SVM) obtained from the logistic loss, and the least-squares (LS) solutions, obtained from the squares loss. Moreover, one may ask under which conditions we can have benign overfitting for SVM and LS solutions. Other questions include how regularization affects the classification performance? One may want to find the regimes in which the interpolating estimator performs better than the regularized estimators in the overparameterized setting.

%One of the main reasons complicating the analysis is the fact that commonly used classifiers, such as support-vector machines (SVM), cannot be expressed in closed form -- unlike least-squares (LS) solutions in linear regression. 

%Benign overfitting in
{\citet{sur2019modern,salehi2019impact,montanari2019generalization,deng2019model,kini2020analytic,mignacco2020role,kammoun2020precise,salehi2020performance} study overparaterized binary linear classification in the proportional asymptotic regme, where the size $n$ of the training set and the size $p$ of the parameter vector grow large at a fixed rate. These works overcome the aforementioned challenge by relying on powerful tools from modern high-dimensional statistics \citep{stojnic2013framework,thrampoulidis2015regularized,thrampoulidis2018precise} and yield asymptotic error predictions that are sharp, but remain limited to the proportional regime and are expressed in terms of complicated---and often hard to interpret and evaluate---systems of nonlinear equations.}
%. However, the results only apply to the linear asymptotic regime. 

A different approach, resulting in more general non-asymptotic, albeit non-sharp, bounds was initiated by \citet{muthukumar2020classification} who studied a 'Signed' classification model with Gaussian features. Their key observation, that drives their analysis, is that the max-margin classifier  linearly interpolates the data given sufficient overparameterization. This allowed the authors to establish a tight link between the (hard to directly analyze) SVM and the (amenable to analysis) LS solutions. In turn, this resulted in identifying sufficient conditions on the covariance spectrum needed for benign overfitting. 
%The results in \cite{muthukumar2020classification} hold for a `Signed' data model. 
While this paper was being prepared, a follow-up work \citet{hsu2020proliferation} has extended their analysis to binary classification under generalized linear models (including the `Signed' model as a special case) \new{and to subGaussian/Haar-distributed features.} \new{Motivated by these works, we investigate the following related open questions: {\em Does  the max-margin classifier interpolate data that are generated from generative (rather than discriminative) models? If so, under what conditions? How do optimally tuned regularized estimators compare to interpolating classifiers? Are there settings in which the latter perform better? How does label noise affect any interpolating properties of the max-margin classifier? What does this imply for benign overfitting?} }

% Here, in contrast to the discriminative classification models studied in \citet{muthukumar2020classification,hsu2020proliferation}, we undertake a study of an alternative generative classification model: the, so called, Gaussian mixture model (GMM). In the GMM, the features are Gaussian vectors  centered around the mean vector $\pm\etab$ of their respective class $y=\pm 1$ with some covariance matrix $\Sigmab$.  {We outline our multiple contributions
% %to the study of overparameterized classification of GMM data 
% below.}
%For this different model, we first show that the SVM solution interpolates the data under sufficient effective overparameterization. Compared 
\subsection{Contributions \new{and novelty}} We answer the questions above by focusing on the popular Gaussian mixture model (GMM). Unlike discriminative classification models, the GMM specifies the feature conditional distribution $\x|y$, setting it to be a multivariate Gaussian %distribution
%studied in \citet{muthukumar2020classification,hsu2020proliferation} we undertake a study of an alternative generative classification model: the, so called, Gaussian mixture model (GMM). In the GMM, 
%the features are Gaussian vectors  
that is centered around a mean vector $y\etab$ (of their respective class $y=\pm 1$) and has covariance matrix $\Sigmab$ (Section \ref{sec-model} for details).  
We outline our contributions
%to the study of overparameterized classification of GMM data 
below and then highlight the novelties compared to prior work.

% For the binary supervised Gaussian mixture model, we perform a detailed non-asymptotic study of both SVM and (regularized) least-squares solutions answering a series of questions motivated by the phenomenon of benign overfitting. 

\vp
\noindent~ \emph{(i)} \emph{Abundance of support vectors (Section \ref{sec-link}):} \new{We show for the first time that the max-margin classifier \emph{linearly interpolates} GMM data given sufficient overparameterization. Notably, our analytic sufficient conditions for this to happen involve not only the covariance spectrum, but also the problem's signal-to-noise-ratio (SNR), which we define as $SNR=\nicefrac{\|\etab\|_2^4}{\etab^T\Sigmab\etab}.$ Thus, we uncover a key difference compared to discriminative data (e.g. Signed model \citep{muthukumar2020classification,hsu2020proliferation}). We complement our sufficient conditions with numerical results that suggest their tightness.}

%We derive sufficient conditions on the extent of overparameterization, on the spectral structure of $\Sigmab$ and on the features' signal-to-noise ratio (SNR) so that the SVM solution \emph{linearly interpolates} the data generated from the GMM. We present numerical results that suggest the tightness of our theoretical conditions.

\vp
\noindent~ \emph{(ii)} \emph{Non-asymptotic  bounds for min-norm estimators (Section \ref{sec-testerror}):} We derive novel \emph{non-asymptotic} error bounds for the min-norm linear interpolator. Our bounds explicitly capture the effect of the overparameterization ratio, of the covariance spectrum and of the SNR. 
%Thanks to the established relation in (i) above, our bounds further characterize the performance of SVM under sufficient overparameterization. 

\vp
\noindent~ \emph{(iii)} 
%\emph{Sufficient conditions for benign-overfitting:}
\new{\emph{Interpolators' risk under high overparameterization (Section \ref{sec-benign}):}}
\new{Combining our findings above, we derive sufficient conditions on the spectrum of $\Sigmab$ and on the SNR that guarantee both the SVM and the LS solutions (a) perfectly interpolate the data, and, (b) achieve asymptotically optimal risk as overparameterization increases.
%infinite overparameterization leads the gradient-descent iterations to converge (as number of iterations increases) to a solution that (a) perfectly interpolates the data, and, (b) achieves asymptotically optimal classification performance.
}
Our conditions improve upon the state of the art \citep{chatterji2020finite} \new{in the noiseless case (see discussion below)}.

% \citet{soudry2018implicit,ji2019implicit} has  that the solution of gradient descent on logistic loss minimization converges in direction to the SVM solution when data are separable . Combining this with our results in (i) and (ii) above, \new{we identify sufficient conditions on the data covariance and on the SNR guaranteeing that infinite overparameterization leads the gradient-descent iterations to converge (as number of iterations increases) to a solution that (a) perfectly interpolates the data, and, (b) achieves asymptotically optimal classification performance.}
% %with best generalization accuracy over any finite number of parameters. 
% %\new{Specifically, our conditions ensure that the classification error drives to zero as the number of model parameters increases to infinity. With some abuse of terminology (since our model and regime is different than that of \citet{bartlett2020benign}), we use the term benign overfitting to refer to this favorable behavior of classification error in the highly overparameterized regime.}
% %hen this happens we say that benign overfitting occurs.}
% Our conditions improve upon existing ones in the literature \citep{chatterji2020finite} \ct{in the noiseless case (see discussion below). }

\vp
\noindent~ \emph{(iv)} \emph{The effect of regularization (Section \ref{sec-reg}):} We study the effect of ridge-regularization on the risk. Interestingly, we identify regimes that the interpolating estimator (corresponding to zero regularization) outperforms regularized estimates in the overparameterized regime.

\vp
\noindent~ {\emph{(v)} \emph{Interpolation and benign overfitting in noisy models (Section \ref{sec-labelnoise}):} We extend our findings to a \emph{noisy} isotropic Gaussian mixture model, where labels are corrupted with constant probability. First, we find that the favorable interpolating property of SVM continues to hold, but under stronger conditions due to the label corruptions. 
%show that the SVM solution still perfectly interpolates the data (as in the noiseless case) provided sufficient overparameterization compared to the SNR. However, our sufficient conditions are stronger compared to the noiseless case. 
Second, in the regime of interpolation, we upper bound the risk of the minimum-norm interpolator and use this result to identify regimes of benign overfitting, i.e. regimes where the SVM risk asymptotically approaches the Bayes risk despite perfectly fitting the data.}

\vp
\new{On the technical front, while our analysis uses tools similar to those in \citet{bartlett2020benign,muthukumar2020classification},  there are key differences in the GMM, which further complicate the analysis and impose new challenges. %beyond previous works. 
This can be illustrated at a high-level as described below (see also Section \ref{sec-pfoutline}).}
%A more detailed discussion see Section \ref{sec-pfoutline} for details. 
%Let $\mathbf{X}\in\mathbb{R}^{n\times d}$ be the features matrix and $\mathbf{y}\in\{\pm1\}^n$ the label vector of the training set. 
We will show that at the heart of our analysis lies the challenge of upper/lower-bounding  quadratic forms such as $\boldy^T(\boldsymbol{X}\boldsymbol{X})^{-1}\boldy$, where $\boldy$ is the label vector and $\boldsymbol{X}$ is the feature matrix of the training set. Under the GMM, and unlike in linear regression and discriminative classification models, the matrix $\boldsymbol{X}$ ``includes" both the label vector $\boldy$ and the mean vector $\etab$. Hence, considering $\boldy$ and $\boldsymbol{X}$ separately as in \citet{bartlett2020benign,muthukumar2020classification} leads to sub-optimal bounds. \new{Instead, 
%our first key observation is that with an appropriate application of the matrix inversion lemma, it 
we first show that it 
is possible to decompose the original quadratic form of interest into several more primitive quadratic forms on inverse-Wishart %or Haar-distributed 
matrices (rather than on the original Gram matrix). This decomposition is central to our proof technique, but the technical challenge remains because: (a) the decomposition involves the new quadratic forms in a convoluted way requiring us to establish both lower and upper bounds for each one of them and then combine them carefully, and, (b) while more primitive, the desired bounds for the new quadratic forms do \emph{not} follow from previous works.
%and then to carefully put them together so as to arrive at the desired overall bound. 
Besides, as mentioned above, a particular distinguishing feature of GMM compared to previous works is that in the process of doing the above we need to carefully capture the impact of not only the covariance spectrum, but also of the model's SNR.}
%Instead, leveraging the Woodbury matrix inverse identity, we expand the quadratic forms under consideration and carefully analyze each one of the resulting terms. 
\new{Compared to previous works, we also complement our analysis with numerical results validating the tightness of our findings.}
%effectiveness of our approach in accurately capturing the effect of both the covariance spectrum and the SNR.} 
Also, we study the effect of regularization and identify regimes in which interpolating estimators have optimal performance. \new{Compared to \citet{muthukumar2020classification,hsu2020proliferation} we also extend our results to a noisy model with constant probability label corruptions.}

\new{The most closely related work in terms of problem setting and results is the recent paper by \citet{chatterji2020finite}, which thus deserves its own discussion.  \citet{chatterji2020finite} are the first to derive non-asymptotic risk bounds for overparameterized binary mixture models and use them to characterize benign-overfitting conditions. Notably, their bounds hold for sub-Gaussian features and for an adversarial noisy model that is more general than ours. On the other hand, in the special case of GMM, our results improve upon theirs as follows. In the noiseless case, we significantly relax the conditions under which interpolating estimators  asymptotically attain Bayes optimal performance with increasing overparameterization. Also, our risk bounds capture the key role of the covariance structure unlike theirs. In the noisy case, our benign overfitting conditions are the same, but our risk bounds on the min-norm interpolator hold under relaxed scaling assumptions. It is worth mentioning that our proof strategy towards upper bounding the risk of SVM is also entirely different compared to \citet{chatterji2020finite}. In comparison to \citet{chatterji2020finite}, we are also the first to establish interpolating conditions for the SVM solution under GMM data. Finally, our risk bounds also hold for regularized least-squares. }

\new{A more elaborate discussion on the above closely related works, as well as, a comparison to classical margin-based bounds is deferred to Section \ref{sec-discussion} due to space limitations. The Appendix includes detailed proofs of all our results.}

%Finally, our findings on conditions under which  benign overfitting occurs improve upon those of \citet{chatterji2020finite}.% and also extend to the correlated case.  % We further compare our results to those of the most related works \cite{bartlett2020benign,muthukumar2020classification,chatterji2020finite} in the remaining of the paper.

%\subsection{Related work}
%\kw{Would we add a new section here or put all related work in Section 8.}

% \subsection{\new{Organization}} In Section \ref{sec-model}, we formally define the Gaussian mixture model and the classification algorithms that we focus on.  In Section \ref{sec-link}, we establish a link between the min-norm interpolating solution and the SVM solution. In Section \ref{sec-testerror}, we derive upper bounds on the classification error of the min-norm interpolator and the regularized estimator. Based on the results in Sections \ref{sec-link} and \ref{sec-testerror}, \new{in Section \ref{sec-benign} we find SNR regimes such that interpolating solutions asymptotically achieve Bayes optimal error as overparameterization increases}. We then discuss how regularization affects the classification error in Section \ref{sec-reg}. \new{We extend our results to a noisy GMM in Section \ref{sec-labelnoise}, where we also derive sufficient conditions for benign overfitting.} A proofs outline is given in Section \ref{sec-pfoutline}. \kw{We end the paper with a detailed comparison to the most closely related works in the literature in SM \ref{sec-discussion} and some ideas for future work in Section \ref{sec-future}.}

\noindent\textbf{Notation.} For a vector $\boldsymbol{v}\in \mathbb{R}^p$ , let $\Vert \boldsymbol{v} \Vert_2 = \sqrt{\sum_{i=1}^p v_i^2}$, $\Vert \boldsymbol{v} \Vert_1 = {\sum_{i=1}^p |v_i|}$, $\Vert \boldsymbol{v} \Vert_{-1} = {\sum_{i=2}^p |v_i|}$, $\Vert \boldsymbol{v} \Vert_{\infty} = \max_{i}\{|v_i|\}$ and $\boldsymbol{e}_i$ denotes the $i$-th standard basis vector. For a matrix $\boldsymbol{M}$, $\Vert \boldsymbol{M} \Vert_2$ denotes its operator norm. [$n$] denotes the set $\{1,2,...,n\}$. We also use standard ``Big O" notations $\Theta(\cdot)$, $\omega(\cdot)$, e.g., see \citet[Chapter 3]{cormen2009introduction}. Finally, we write $\Nn(\boldsymbol{\mu},\Sigmab)$ for the (multivariate) Gaussian distribution of mean $\boldsymbol{\mu}$ and covariance matrix $\Sigmab$, and, $Q(x)=\mathbb{P}(Z>x),~Z\sim\Nn(0,1)$ for the Q-function of a standard normal. Throughout, `constants' refer to numbers that do \emph{not} depend on the problem dimensions $n$ or $p$.

\section{Learning Model} \label{sec-model}

\subsection{Data model} \label{subsec-feat}
Consider the following supervised binary classification problem under a \emph{Gaussian mixtures model} (GMM). Let $\boldsymbol{x} \in \mathbb{R}^p$ denote the feature vector and $y \in \{-1, +1\}$ its class label. 
%We investigate the binary classification under the Gaussian mixture models (GMM).
The class label $y$ takes one of the values $\{ \pm 1 \}$ with probabilities $\pi_{\pm 1}$ such that $\pi_{+1} + \pi_{-1} = 1$. The class-conditional probability $p(\boldsymbol{x}|y)$ follows Gaussian distribution. Specifically, conditional on $y = \pm 1$, the feature vector $\boldsymbol{x}$ is a Gaussian vector with mean vector $\pm \boldsymbol{\eta} \in \mathbb{R}^p$ and an invertible covariance matrix $\Sigmab$. Summarizing, the data pair $(\boldsymbol{x},y)$ is generated such that
\begin{equation}
\label{eq-GM}
    y = \begin{cases} 1, & \text{w.p.}~~ \pi_{+1} \\
    -1, & \text{w.p.}~~ 1-\pi_{+1}     \end{cases}
    \qquad\text{and}\qquad \boldsymbol{x} | y \sim \Nn(y  \boldsymbol{\eta}, \boldsymbol{\Sigma}).
    %\iff \boldsymbol{x} \sim N(\pm  \boldsymbol{\eta}, \boldsymbol{\Sigma}).
\end{equation}
We denote the eigenvalues of $\Sigmab$ by  $\boldsymbol{\lambda} := [\lambda_1, \cdots, \lambda_p]$, with $\lambda_1 \ge \lambda_2 \ge \cdots \ge \lambda_p$, and write  the eigendecomposition of $\boldsymbol{\Sigma}$  as
$
    \Sigmab = \sum_{i=1}^p\lambda_i\boldsymbol{v}_i\boldsymbol{v}_i^T = \boldsymbol{V}\boldsymbol{\Lambda}\boldsymbol{V}^T,
$
where $\boldsymbol{\Lambda}$ is a diagonal matrix whose diagonal elements are eigenvalues of $\Sigmab$ and the columns of matrix $\boldsymbol{V}$ are eigenvectors of $\Sigmab$. Using the eigenvecotrs of $\Sigmab$ as a basis, the mean vector $\etab$ can be expressed as $\boldsymbol{\eta} = \boldsymbol{V}\betab$, where $\betab \in \mathbb{R}^p$. Note that $\normeta = \normbeta$.

Consider training set $\{(\boldsymbol{x}_i,y_i)\}_{i=1}^n$ composed of $n$ IID data pairs generated according to the GMM in \eqref{eq-GM}. Let $\boldsymbol{X} = [\boldsymbol{x}_1, \boldsymbol{x}_2,\cdots,\boldsymbol{x}_n]^{T} \in \mathbb{R}^{n \times p}$ denote the feature matrix and $\boldsymbol{y} = [y_1,\cdots,y_n]^T$ denote the class-label vector. Following \eqref{eq-GM}, the data matrix $\boldsymbol{X}$ can be expressed as follows for a ``noise matrix" $\boldsymbol{Q} \in \mathbb{R}^{n \times p}$ with independent $\Nn(\mathbf{0},\Sigmab)$ rows, $$\boldsymbol{X} = \boldsymbol{y}\boldsymbol{\eta}^T + \boldsymbol{Q}.$$

\subsubsection{Data covariance structure}
\label{sec-covdef}
One of our contributions is demonstrating how the classification performance on data from the GMM \newadd{depends} crucially on the structure of the data covariance. To explicitly capture this dependency, we consider two ensembles for the spectrum of the data covariance $\Sigmab$.
\begin{definition}[Balanced ensemble]
\label{def-balancedef}
No eigenvalues of $\Sigmab$ are significantly larger than others. Specifically, there exists a constant $b>1$ such that
\begin{align}
    \label{def-balance}
    bn\lambda_1 \le \binormlbd,
\end{align}
where $\binormlbd = \sum_{i=2}^p \lambda_i$. An example of special interest is the isotropic case $\Sigmab = \boldsymbol{I}$ with sufficient overparameterization, i.e., $p > Cn$, for some constant $C > 1$.
\end{definition}
\begin{definition}[Bi-level ensemble] \label{def-bileveldef}
One eigenvalue of $\Sigmab$ is much larger than others. Specifically, there exist constants $b_1, b_2 > 1$ such that
\begin{align}
    \label{def-bilevel}
    b_1 n\lambda_1 \ge \binormlbd \ \ \text{and} \ \  b_2 n\lambda_2 \le \sum_{i=3}^p \lambda_i.
\end{align}
\end{definition}

The different nature of the two models leads to different conclusions on how the covariance structure affects our key results on abundance of support vectors and benign overfitting. 
Similar data covariance structures were considered in \citet{muthukumar2020classification}, but for the discriminative model $y_i=\rm{Sign}(\boldsymbol{x}_i^T\etab), i\in[n]$ with features $\boldsymbol{x}_i\sim\mathcal{N}(0,\Sigmab).$ \new{The two ensembles above are also related to the notions of effective ranks introduced by \citet{bartlett2020benign} in the study of benign overfitting for linear regression (see Section \ref{sec-comparebartlett} for details). }
%\newadd{We can also relate the two ensembles above to the effective ranks defined in \cite{bartlett2020benign}: $r_k := {(\sum_{i>k}^p \lambda_i)}/{\lambda_{k+1}}$ and $R_k = (\sum_{i>k}^p\lambda_i)^2/(\sum_{i>k}^p\lambda_i^2)$. To see the relationship, we further define $\tilde{r}_k := {(\sum_{i>k+1}^p \lambda_i)}/{\lambda_{k+1}}$ and clearly $r_k = \tilde{r}_k + 1$. In the balanced ensemble, we have $\tilde{r}_0 \ge bn$ and this implies $r_0 \ge bn$. Actually, for large enough $n$, $r_0 > b^{'}n$ also implies $\tilde{r}_0 \ge bn$. In the bi-level ensemble, we first have $\tilde{r}_0 \le bn$ and this implies $r_0 \le b^{'}n$ for large enough $n$. Additionally, in the bi-level ensemble, we have $\tilde{r}_1 \ge b_1n$ and this implies $r_1 \ge b_1n$. As we noted before, in addition to the covariance itself, the interplay between the data covariance and the mean $\etab$ is more intricate and plays important role for the GMM; see Sections \ref{sec-link} and \ref{sec-benign}.}

% \begin{equation}
% \label{eq-datax}
%     \boldsymbol{X} = \boldsymbol{y}\boldsymbol{\eta}^T + \boldsymbol{Q}.
% \end{equation}

\subsubsection{Key summary quantities}
As mentioned, our results naturally depend on the spectrum of $\Sigmab$. Specifically, we will identify $\normlbd$ and $\normlbdtwo$ as two key relevant summary quantities. But as hinted by \eqref{eq-GM} the data covariance $\Sigmab$ is expected to interplay with the mean vector $\etab$ in the results. We will show that this interplay is captured by the \emph{the signal strength in the direction of $\Sigmab$}, which we denote 
$$\sigma^2 :=  \Vert \etab \Vert_{\Sigmab}^2 := \etab^T\Sigmab\etab= \betab^T\boldsymbol{\Lambda}\betab.$$ 
Finally, the signal strength $\Vert{\etab}\Vert_2$ will also be important. Note that the two quantities $\sigma^2$ and  $\Vert{\etab}\Vert_2$ define a natural notion of signal-to-noise ratio (SNR) for the GMM. {\newadd{To better see this, take inner products of both sides of \eqref{eq-GM} with $\etab$ to express the label-feature relation as
%\begin{equation}
%\label{eq-GM02}
$ \boldsymbol{x} = y \boldsymbol{\eta} + \boldsymbol{q} \implies y=\frac{\boldsymbol{\eta}^T\boldsymbol{x}}{\Vert \boldsymbol{\eta} \Vert_2^2} - \frac{\boldsymbol{\eta}^T\boldsymbol{q}}{\Vert \boldsymbol{\eta} \Vert_2^2}$}, where $  \boldsymbol{q} \sim N(\boldsymbol{0}, \boldsymbol{\Sigma})$.
 %\end{equation}
% In this paper, we assume that $\boldsymbol{\Sigma} = \boldsymbol{I}$.
% \begin{remark}
% \label{rmk-snr}
% \normalfont{Even though (\ref{eq-GM02}) looks like regression, analysis of (\ref{eq-GM}) is significantly different from the analysis of linear regression in \cite{bartlett2020benign}. First, here we consider the classification error. Second, the noise term $\frac{\boldsymbol{\eta}^T\boldsymbol{q}}{\Vert \boldsymbol{\eta} \Vert_2^2}$ depends on the true signal $\boldsymbol{\eta}$ unlike typical regression assumption that noise is independent of signal.}
% \end{remark}
Then, following the standard definition in random-design regression and noting that 
%\begin{align*}
     %\label{eq-snr}
$     \frac{\rm Var(\boldsymbol{\eta}^T\boldsymbol{x})}{\rm Var(\boldsymbol{\eta}^T\boldsymbol{q})}=\frac{c\Vert \boldsymbol{\eta} \Vert_2^4+\boldsymbol{\eta}^T\boldsymbol{\Sigma}\boldsymbol{\eta}}{\boldsymbol{\eta}^T\boldsymbol{\Sigma}\boldsymbol{\eta}} = \frac{c\Vert \boldsymbol{\eta} \Vert_2^4}{\boldsymbol{\eta}^T\boldsymbol{\Sigma}\boldsymbol{\eta}}+1$,
%\end{align*}
for $0 \le c \le 1$ depending on $\pi_{+1}$, we let  $\rm{SNR} := \frac{\Vert \boldsymbol{\eta} \Vert_2^4}{\boldsymbol{\eta}^T\boldsymbol{\Sigma}\boldsymbol{\eta}}=\frac{\Vert \boldsymbol{\beta} \Vert_2^4}{\boldsymbol{\beta}^T\boldsymbol{\Lambda}\boldsymbol{\beta}}$; \new{Lemma \ref{lem-miserror} bounds the classification error in terms of the same quantity, which further validates its role as the SNR.}}
% Additionally, we denote \emph{the signal strength in the direction of $\Sigmab$} as $$\sigma^2 :=  \Vert \etab \Vert_{\Sigmab}^2 := \etab^T\Sigmab\etab= \betab^T\boldsymbol{\Lambda}\betab.$$ 
% %Note that SNR is inversely proportional to $\sigma^2$. 
% As we will see, all our results are given in terms of the following individual summary quantities for $\etab$ and $\Sigmab$: $\normeta$, $\normlbd$, $\normlbdtwo$, but also with respect to $\sigma^2$ which summarizes both. %Using some abuse of terminology, $\normbeta^4$ measures the `signal strength' and $\sigma^2$ measures the `noise strength'. 
% Note that when $\Sigmab = \boldsymbol{I}$, both SNR and $\sigma^2$ equal $\normeta^2$.}
% %When $\boldsymbol{\Sigma} = \lambda\boldsymbol{I}$, the SNR becomes $\frac{\normeta^2}{\sigma}$. In this case, $\normeta^2$ can be viewed as a measure of the signal strength and the diagonal element $\sigma$ can be used to measure the noise strength.

\subsection{Training algorithm}
\label{subsec-train}
Given access to the training set, we train a  linear classifier $\hat{\boldsymbol{\eta}}$ by minimizing the empirical risk $\hat{{\mathcal{R}}}_{\rm emp}(\boldsymbol{w}):={\frac{1}{n}\sum_{i=1}^n {\ell (y_i \cdot\boldsymbol{w}^T \boldsymbol{x}_i)}}$, 
%for a 
% \begin{equation}
%     \label{eq-emprisk}
%     \min_{\boldsymbol{\eta}}\hat{{\mathcal{R}}}_{emp}(\boldsymbol{\eta}):={\frac{1}{n}\sum_{i=1}^n {\ell (y_i \boldsymbol{\eta}^T \boldsymbol{x}_i)}},
% \end{equation}
%loss function $\ell: \mathbb{R} \to \mathbb{R}$.
%Then given a new sample of feature vector $\boldsymbol{x}$, the label is: $\hat{y} = \text{sign}(\hat{\boldsymbol{\eta}}^T\boldsymbol{x})$. 
%We focus on two popular choices of $\ell$, namely,
where the loss function $\ell$ is chosen as: (i) Least-squares (LS): $\ell(t) = (1-t)^2$, or, (ii) Logistic: $\ell(t) = \log{(1+e^{-t})}$.
% \begin{itemize}
%     \item Least-squares (LS): $\ell(t) = (1-t)^2$, or,
%     \item Logistic: $\ell(t) = \log{(1+e^{-t})}$.
% \end{itemize}
Throughout, we focus on the \emph{overparameterized} regime $p>n$. As is common, we run gradient descent (GD) on the empirical risk.
%{A common approach to minimizing $\hat{{\mathcal{R}}}_{emp}(\boldsymbol{\eta})$ is to apply gradient descent (GD). 
The following results characterizing the \emph{implicit bias} of GD for the square and logistic losses in the overparameterized regime are well-known. For one, when data can be linearly interpolated (i.e., $\exists \betab\in\mathbb{R}^p$ such that $y_i=\boldsymbol{x}_i^T\betab,~\forall i\in[n]$), then GD on square loss \newadd{with sufficiently small step size} converges (as the number of iterations grow to infinity) to the solution of \emph{min-norm interpolation}, e.g. \citet{hastie2019surprises}:
\begin{equation}
    \label{eq-minsolution}
    \hat{\boldsymbol{\eta}}_{\rm LS} = \text{arg} \min_{\boldsymbol{w}} \Vert \boldsymbol{w} \Vert_2~~ \text{subject to} \ y_i = \boldsymbol{w}^T\boldsymbol{x}_i, \forall i\in[n].
\end{equation}
Second, when data are linearly separable (i.e., $\exists \betab\in\mathbb{R}^p$ such that $y_i(\boldsymbol{x}_i^T\betab)\geq 1,~\forall i\in[n]$), then the normalized iterates of GD on logistic loss converge in direction \footnote{Precisely, convergence is in the sense of the normalized GD iterations $\etab^t$, i.e. $\Vert \frac{\etab^t}{\Vert \etab^t\Vert_2} - \frac{\etasvm}{\Vert \etasvm^t\Vert_2} \Vert _2 \overset{t \to \infty}{\to} 0$.} to the solution of \emph{hard-margin SVM} \citep{soudry2018implicit,ji2019implicit} (see also \citet{rosset2003margin} for earlier similar results):
\begin{equation}
\label{eq-svmsolution}
    \hat{\boldsymbol{\eta}}_{\rm SVM} = \text{arg} \min_{\boldsymbol{w}} \Vert \boldsymbol{w} \Vert_2~~ \text{subject to} \ y_i\boldsymbol{w}^T\boldsymbol{x}_i \ge 1, \forall i\in[n].
\end{equation}

\new{Now, specializing to data from the GMM, it can be shown that when $p>n+2$, then the data can be linearly interpolated with high probability (whp.). In turn, this easily implies that data are also linearly separable. See Appendix \ref{sec:lin_sep} for a formal statement and  proof of these claims.} Combining those, in the overparameterized regime, whp.,  GD on data from the GMM converges to either \eqref{eq-minsolution} or \eqref{eq-svmsolution} for a square and logistic loss, respectively. 

The behavior above holds when no explicit regularization is used. To see the role of regularization, we also consider the ridge estimator given by
\begin{align}
\label{eq-ridgeest}
    \etareg &= \arg{\min_{\boldsymbol{w}}\{\Vert \boldy - \boldsymbol{X}\boldsymbol{w} \Vert_2^2 + \tau \Vert \boldsymbol{w} \Vert_2^2 \}} = \boldsymbol{X}^T(\boldsymbol{X}\boldsymbol{X}^T + \tau \boldsymbol{I})^{-1}\boldy.
\end{align}
Note that $\etals$ can be obtained from \eqref{eq-ridgeest} by setting $\tau = 0$ ($\boldsymbol{X}\boldsymbol{X}^T$ is non-singular whp. for $p>n$, e.g., \citet{vershynin2018high}).
%Note that $\lim_{\tau \to 0}\etareg = \etals$.
%where we assume that the matrix $\boldsymbol{X}\boldsymbol{X}^T + \tau \boldsymbol{I}$ is non-degenerate.

Henceforth, we focus on the classifiers in \eqref{eq-svmsolution}, \eqref{eq-minsolution}, \eqref{eq-ridgeest}. With some abuse of terminology, we often refer to the minimum-norm interpolator in \eqref{eq-minsolution} as LS solution for brevity.

% For the logistic loss, when the training data are linearly separable, the normalized GD iterates converge
% Unlike the least squares min-norm solution, there is no known simple expression for $\hat{\boldsymbol{\eta}}_{SVM}$ as a function of the training data. This is one of our motivations to investigate the connection between the LS min-norm solution and SVM solution. 

\subsection{Classification error}
For a new sample $(\boldsymbol{x},y)$, the classifier $\hat{\boldsymbol{\eta}}$ classifies $\boldsymbol{x}$ as $\hat y = \rm{sign}(\boldeta^T \boldsymbol{x})$. Then, the classification error is measured by the expected 0-1 loss risk
\begin{equation}
    \label{eq-testrisk}
    \mathcal{R}(\hat{\boldsymbol{\eta}}) = \mathbb{E}[\mathbb{I}(\hat{y} \ne y)]=\Prob(\hat{\boldsymbol{\eta}}^T(y\boldsymbol{x})<0),
\end{equation}
where the expectation is over the distribution of  $(\boldsymbol{x},y)$ generated as in \eqref{eq-GM}. The following simple lemma gives an upper bound on $\mathcal{R}(\hat{\boldsymbol{\eta}})$.
\begin{lemma}
\label{lem-miserror}
%Suppose $\hat\eta\cdot\eta>0$. 
%Then, t
Under the Gaussian-mixtures model, the classification error of a classifier $\hat{\boldsymbol{\eta}}$ satisfies,
$\mathcal{R}(\hat{\boldsymbol{\eta}}) =  Q(\frac{\boldeta^T\boldsymbol{\eta}}{\sqrt{\hat{\boldsymbol{\eta}}^T \boldsymbol{\Sigma}\hat{\boldsymbol{\eta}}}}). $ In particular, if $\boldeta^T\boldsymbol{\eta}>0$, then
$\mathcal{R}(\hat{\boldsymbol{\eta}})\le \exp{\big(-\frac{(\hat{\boldsymbol{\eta}}^T\boldsymbol{\eta})^2}{2\hat{\boldsymbol{\eta}}^T \boldsymbol{\Sigma}\hat{\boldsymbol{\eta}}}\big)}.$
%In the special case $\Sigma=\mathbf{I}$, it holds that $\mathcal{R}(\hat{\boldsymbol{\eta}}) = Q(\frac{\hat\eta\cdot\eta}{\|\hat\eta\|_2^2}).$
\end{lemma}
\begin{proof}
For a new draw $\boldsymbol{x}, y$, using  $\boldsymbol{x} = y \boldsymbol{\eta} + \boldsymbol{\Sigma}^{1/2}\boldsymbol{z} , \boldsymbol{z}\sim\Nn(\boldsymbol{0},\boldsymbol{I})$ and symmetry of the Gaussian distribution, it can be easily checked that $\mathcal{R}(\hat{\boldsymbol{\eta}}) 
=\Prob(\hat{\boldsymbol{\eta}}^T(y\boldsymbol{q}) < -\hat{\boldsymbol{\eta}}^T\boldsymbol{\eta})=\Prob({\boldsymbol{\Sigma}}^{1/2}\hat{\boldsymbol{\eta}}^T\boldsymbol{z} > \hat{\boldsymbol{\eta}}^T\boldsymbol{\eta})$.
%where the last equality comes from (\ref{eq-GM02}). 
Now, $\boldsymbol{\Sigma}^{1/2}\hat{\boldsymbol{\eta}}^T\boldsymbol{z}$ is a zero-mean Gaussian random variable with variance $\boldeta^T\boldsymbol{\Sigma}\boldeta$. Thus, the advertised bounds follow directly: the first, by definition of the Q-function, and, the second, by the Chernoff bound for the Q-function, e.g., \citet[Ch.~2]{wainwright2019high}.
\end{proof}

Thanks to the lemma above, our goal of upper bounding the classification error, reduces to that of lower bounding the ratio $\frac{(\hat{\boldsymbol{\eta}}^T\boldsymbol{\eta})^2}{2\hat{\boldsymbol{\eta}}^T \boldsymbol{\Sigma}\hat{\boldsymbol{\eta}}}$. We do this in Section \ref{sec-testerror} for the classifiers $\etareg$ and $\etals$. In large, this is possible because these estimators can be conveniently written in closed forms (see \eqref{eq-ridgeest}). In contrast, the SVM solution \emph{cannot} be expressed in closed form. To get around this challenge, Section \ref{sec-link} establishes sufficient conditions under which the SVM-solution $\hat{\boldsymbol{\eta}}_{\rm SVM}$ linearly interpolates the data, thus, it coincides with the LS solution. 

%Henceforth, we focus on the isotropic case $\Sigmab=\mathbf{I}$. We leave the analysis of general covariance matrices to future work.

\section{Link between SVM and linear-interpolation}
\label{sec-link}
% When the logistic loss is used in training, the GD iterates converge to the solution to hard-margin SVM, defined in (\ref{eq-svmsolution}). However, u

This section establishes a link between the SVM solution in \eqref{eq-svmsolution} and the LS solution in \eqref{eq-minsolution} for general $\Sigmab$. Specifically, Theorem \ref{thm-linkgen} below identifies sufficient conditions under which all training data points become support vectors, i.e., $\hat{\boldsymbol{\eta}}_{\rm SVM}$ linearly interpolates the data: $\boldsymbol{x}_i^T\boldeta_{\rm SVM}=y_i,~\forall i\in[n].$

% Unlike the min-norm linear interpolator in \eqref{eq-minsolution} which can be conveniently written in closed form  obtaining an exact expression for $\hat{\boldsymbol{\eta}}_{SVM}$ is not possible.
% Inspired by \cite{muthukumar2020classification,hsu2020proliferation}, we explore the link between the SVM solution $\hat{\boldsymbol{\eta}}_{SVM}$ and the min-norm least squares solution $\hat{\boldsymbol{\eta}}$ under GMM. The following Theorem states that all the training data points become support vectors with high probability if the overparameterization is sufficient and the task is 'hard' enough. 
\begin{theorem}
\label{thm-linkgen}
% For sufficiently large number of training points n and large positive constant $C_i$'s,
Assume $n$ training samples following the GMM defined in Section \ref{sec-model}. There exist constants $C_1, C_2 >1$ such that, if the following conditions on the eigenvalues of $\Sigmab$ and on the signal strength in the direction of $\Sigmab$ defined as $\sigma^2\overset{\Delta}{=}\suminner$ hold:
\begin{align}
\label{eq-linkgen01}
\normlbd &> 72\big(\Vert \boldsymbol{\lambda}\Vert_2\cdot n\sqrt{\log{n}} + \Vert \boldsymbol{\lambda}\Vert_{\infty}\cdot n\sqrt{n}\log{n} + 1\big),\\
\label{eq-linkgen02}
\normlbd &> C_1n\sqrt{\log(2n)}\, \sigma,
\end{align}
then, the SVM-solution $\hat{\boldsymbol{\eta}}_{\rm SVM}$  satisfies the linear interpolation constraint in \eqref{eq-minsolution} with probability at least $(1 - \frac{C_2}{n})$.
% for the following conditions on the number of features $p$ and the mean-vector $\boldsymbol{\eta}$:
% \begin{align}
% p > 10n\log n + n -1 \quad\text{and} \quad
% p > C_2 n \Vert \boldsymbol{\eta} \Vert_2. \label{eq-link02}
% \end{align}
\end{theorem}
%\cite[Theorem 1]{wang2020benign} shows the conditions for the isotropic case.
For the isotropic case, condition \eqref{eq-linkgen01} can be sharpened as shown in the following theorem.
\begin{theorem}
\label{thm-linkiso}
% For sufficiently large number of training points n and large positive constant $C_i$'s,
Assume $n$ training samples following the GMM with $\Sigmab=\mathbf{I}$. There exist constants $C_1,C_2>1$ such that, if the following conditions on the number of features $p$ and the mean-vector $\boldsymbol{\eta}$ hold:
\begin{align}
p > 10n\log n + n -1 \quad\text{and} \quad
p > C_1n\sqrt{\log(2n)}\normeta, \label{eq-link02}
\end{align}
then, the SVM-solution $\hat{\boldsymbol{\eta}}_{\rm SVM}$  satisfies the linear interpolation constraint in \eqref{eq-minsolution} with probability at least $(1 - \frac{C_2}{n})$.
% for the following conditions on the number of features $p$ and the mean-vector $\boldsymbol{\eta}$:
% \begin{align}
% p > 10n\log n + n -1 \quad\text{and} \quad
% p > C_2 n \Vert \boldsymbol{\eta} \Vert_2. \label{eq-link02}
% \end{align}
\end{theorem}

The theorems establish two sufficient conditions each for all training samples to become support vectors. In the isotropic setting, the first condition requires that the number of features $p$ is significantly larger than the number of observations $n$. \newadd{For the anisotropic case, the corresponding condition is related the the effective ranks $r_0$ and $R_0$ \citep{bartlett2020benign,muthukumar2020classification}, i.e. $r_k := {(\sum_{i>k}^p \lambda_i)}/{\lambda_{k+1}}$ and $R_k = (\sum_{i>k}^p\lambda_i)^2/(\sum_{i>k}^p\lambda_i^2)$. The condition requires that the covariance spectrum has sufficiently slowly decaying eigenvalues (corresponding to sufficiently large $R_0$), and that it is not too ``spiky" (corresponding to sufficiently large $r_0$). \citet[Remark 4]{muthukumar2020classification} provides a detailed discussion on how the effective ranks relate to different spectrum regimes.} Specifically, the bi-level ensemble (Definition \ref{def-bileveldef}) does \emph{not} satisfy \eqref{eq-linkgen01}. To see that, \eqref{eq-linkgen01} implies $\normlbd > 72 n\sqrt{n}(\log{n})\lambda_1$, meaning that $n\lambda_1$ should not be large compared to the sum of other eigenvalues. In contrast, the bi-level ensemble requires $b_1n\lambda_1 > \binormlbd$. The second conditions in the two theorems above are the same to each other, since $\sigma = \normeta$ in the isotropic setting. These latter conditions relate to the SNR and constrain the signal strength in the direction of $\Sigmab$. %We also remark that these conditions are new compared to the Signed model studied in \cite[Theorem 1]{muthukumar2020classification}.
%\begin{theorem}
%\label{thm-linkiso}
% For sufficiently large number of training points n and large positive constant $C_i$'s,
%Assume $n$ training points following the GMM with $\Sigmab = \boldsymbol{I}$. Let $n$ be large enough. There exist positive constants $C_1,C_2>1$ such that, if the following conditions on the number of features $p$ and the mean-vector $\boldsymbol{\eta}$ hold:
%\begin{align}
%p > 10n\log n + n -1 \quad\text{and} \quad
%p > C_2 n \Vert \boldsymbol{\eta} \Vert_2, \label{eq-link02}
%\end{align}
%then, the SVM-solution %$\hat{\boldsymbol{\eta}}_{\rm SVM}$ in  %\eqref{eq-svmsolution} satisfies the linear %interpolation constraint in \eqref{eq-minsolution} %with probability at least $(1 - \frac{C_1}{n})$.
% for the following conditions on the number of features $p$ and the mean-vector $\boldsymbol{\eta}$:
% \begin{align}
% p > 10n\log n + n -1 \quad\text{and} \quad
% p > C_2 n \Vert \boldsymbol{\eta} \Vert_2. \label{eq-link02}
% \end{align}
%\end{theorem}
%The theorem establishes two sufficient conditions for all training points to become support vectors. The first condition requires sufficient overparameterization: $p$ be sufficiently large compared with $n$. The second one restricts the Euclidean norm of the mean vector to be no larger than a constant multiple of the overparameterization ratio $p/n$. Note from \eqref{eq-GM} that for $\Sigmab=\mathbf{I}$, the squared-norm $\|\etab\|_2^2$ plays the role of SNR. Thus, the second condition requires that the SNR is not too high. 

\begin{figure}[t!]
\begin{minipage}[b]{1\linewidth}
  \centering
  \centerline{\includegraphics[width=8.5cm]{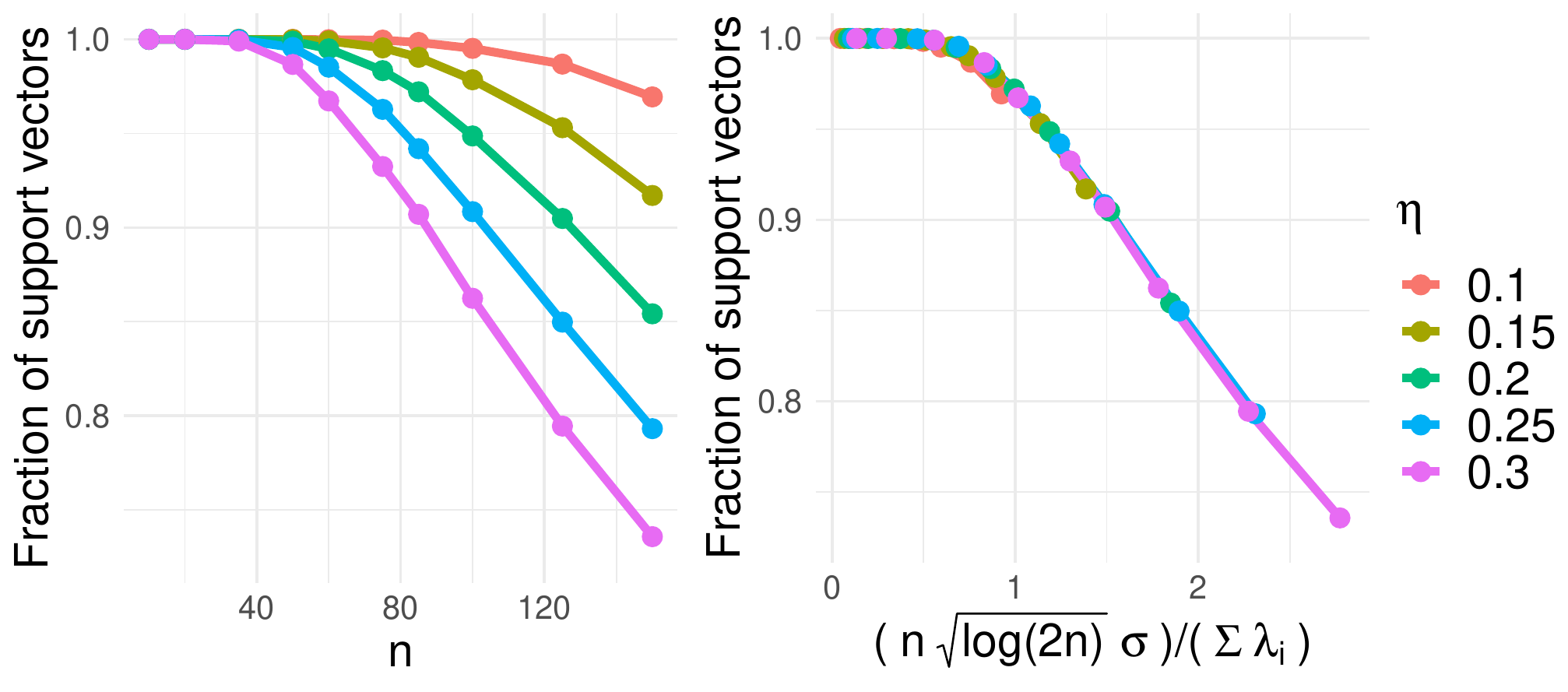}}
%  \vspace{2.0cm}
  %\centerline{(a) Result 1}\medskip
\end{minipage}
\caption{Proportion of support vectors  for various values of $\sigma^2$. Note that the five curves nearly overlap when plotted versus $n\sqrt{\log(2n)}\sigma^2/\normlbd$ as predicted by \eqref{eq-linkgen02} in our Theorem \ref{thm-linkgen} confirming its tightness. See text for details on choices of $\etab, \Sigmab$ and $p$.
%With sufficient overparameterization $p=1500$, the proportion of support vectors gets smaller if the ratio c is small compared with $n$. The right plot uses $n/\frac{p}{\Vert \boldsymbol{\eta} \Vert_2}$ as x-axis and we observe that all the curves move close to each other.
}
\label{fig-linksvm}
\end{figure}

To better interpret the result of the two theorems we show corresponding numerical results in Figure \ref{fig-linksvm}. As explained, the figure also confirms the tightness of our theoretical prediction. In all our simulations throughout this paper, we fix $\pi_{+} = 0.5$ and plot averages over $300$ Monte-Carlo realizations. For simplicity, we choose diagonal $\Sigmab$; thus, $\etab = \betab$. In Fig.~\ref{fig-linksvm}, we guarantee \eqref{eq-linkgen01} by setting $p=1500$ and varying $n$ up to $150$. For the eigenvalues of $\Sigmab$, we set $\lambda_1 = 7.5$, $\lambda_2 = \cdots = \lambda_{p-1} = 1$ and $\lambda_p = 0.2$. For $\etab$, we chose $\eta_1= \cdots =\eta_p = \eta$, where $\eta = 0.1, 0.15, 0.2, 0.25$ or $0.3$. Fig.~\ref{fig-linksvm}(Left) shows how the fraction of support vectors changes with $n$ for different $\eta$. Smaller $\eta$ results in higher proportion of support vectors. In order to verify the second condition in \eqref{eq-linkgen02}, Fig.~\ref{fig-linksvm}(Right) plots the same curves over a re-scaled axis $n\sqrt{\log(2n)}\sigma/\normlbd$ (as suggested by \eqref{eq-linkgen02}). Note that the 5 curves corresponding to different settings overlap in this new scaling, which agrees with the prediction of Theorem \ref{thm-linkgen}.

Next, we explain how Theorems \ref{thm-linkgen} and \ref{thm-linkiso} are useful for our purpose of studying the classification error of $\hat\etab_{\rm SVM}$. {Suppose \eqref{eq-linkgen01} and \eqref{eq-linkgen02} (or \eqref{eq-link02} in the  isotropic case) hold.  Then $\hat\etab_{\rm SVM} = \hat\etab_{\rm LS} =  \mathbf{X}^T(\mathbf{X}\mathbf{X}^T)^{-1}\mathbf{y}$}. Thus, under these conditions we can analyze the classification error of \eqref{eq-svmsolution}, by studying the simpler LS solution in \eqref{eq-minsolution}. This observation was recently first exploited in \citet{muthukumar2020classification} and sharpened in \citet{hsu2020proliferation}, but for a different data model. To see why the above statement is true, note that when \eqref{eq-linkgen01} and \eqref{eq-linkgen02} (or \eqref{eq-link02}) hold, then $\hat\etab_{\rm SVM}$ satisfies the linear interpolation constraints; thus, it is feasible in \eqref{eq-minsolution}. Consequently, $\boldeta_{\rm SVM}$ is in fact optimal in \eqref{eq-minsolution}. To see the latter, assume for the sake of contradiction that $\|\hat\etab_{\rm LS}\|_2<\|\hat\etab_{\rm SVM}\|_2$. But, for all $i\in[n]$, $y_i(\hat\etab_{\rm LS}^T\boldsymbol{x}_i) = y_i^2 \geq 1$; thus, $\hat\etab_{\rm LS}$ is feasible in \eqref{eq-svmsolution}, which contradicts our assumption. We will rely on this observation in Section \ref{sec-benign} to study benign overfitting of SVM.

% Observe that the proportion of 
% In Fig. \ref{fig-linksvm}(Left)

% For each curve, we fix the ratio $\frac{p}{\Vert \boldsymbol{\eta} \Vert_2}$ to be $100, 50$ and $33$, respectively. We observe that the proportion of support vectors gets smaller if $\frac{p}{\Vert \boldsymbol{\eta} \Vert_2}$ is small compared with $n$. According to Theorem \ref{thm-linkiso}, the right figure uses $n/\frac{p}{\Vert \boldsymbol{\eta} \Vert_2}$ as x-axis and we observe that all the curves move close to each other, implying the effectiveness of our theoretical results.

Finally, we compare our result to  \cite{muthukumar2020classification} that established similar conditions to Theorem \ref{thm-linkgen}, but for a `Signed' model:  $y_i={\rm sign}(\boldsymbol{x}_i^T\etab)$ with $\boldsymbol{x}_i\sim\Nn(\mathbf{0},\Sigmab)$. Interestingly, \citet{muthukumar2020classification} obtained sufficient conditions that are identical to the first conditions in Theorems \ref{thm-linkgen} and \ref{thm-linkiso}. More recently, \citet{hsu2020proliferation} sharpened the overparameterization condition \eqref{eq-linkgen01} to $\Vert \lbdb \Vert_1 \ge C_1 \sqrt{n} \Vert \lbdb \Vert_2$ and $\Vert \lbdb \Vert_1 \ge C_2 n \log n \Vert \lbdb \Vert_{\infty}$ with large constants $C_1$ and $C_2$ for the anisotropic case. While their proof technique does not appear to be easily extended to the analysis of GMM, sharpening \eqref{eq-linkgen01} can be an interesting future work. %\newadd{The concurrent work \cite{cao2021risk} obtained a sharper result $\Vert \lbdb \Vert_1 \ge \max\{n\sqrt{n}\Vert \lbdb \Vert_\infty, n\Vert \lbdb \Vert_2\}$ for the overparameterization condition by applying different concentration inequalities. }
\emph{The second conditions related to SNR are tailored to the GMM.} %\newadd{and the same as \cite[Theorem 3.1]{cao2021risk}.} 
Intuitively, this is explained since in the `Signed' model, the data are insensitive to the value of the signal strength $\|\etab\|_2^2$; what matters is only the direction of $\etab$. In contrast, both the direction and the scaling of the mean vector $\etab$ are important in the GMM as apparent from \eqref{eq-GM}. Our analysis captures this in a concrete way. Note that the first condition in Theorem \ref{thm-linkiso} is sharper than in Theorem \ref{thm-linkgen}. This is because, in the isotropic case, we can leverage special properties of Wishart matrices; see Section \ref{sec-proofoutline02} for more details.

% We observe that all the data points become support vectors only when the overparameterization is sufficient, i.e. $p$ is sufficiently large compared with $n$, and the learning task is 'hard' enough, i.e. the signal strength $\normeta$ has an upper bound. \textcolor{blue}{This is different from \cite{muthukumar2020classification}[Theorem 1], where they only need the sufficient overparameterization condition. The reason is that different from the generative model, the mean vector $\boldsymbol{\eta}$ plays a much more important role in GMM, hence it appears in the second line condition of Theorem \ref{thm-linkiso}. Based on this result, we show in Section \ref{sec-testerror}, in the sufficiently overparameterized regime, the good generalization is still possible even when all the points become support vectors.}

\section{Classification error}
\label{sec-testerror}
This section includes upper bounds on the classification error of the unregularized min-norm LS solution $\etals$ and $\ell_2$-regularized LS solution $\etareg$ for the isotropic, balanced and bi-level ensembles. The implications of our bounds on $\etals$ and $\etareg$ are discussed later in Sections \ref{sec-benign} and \ref{sec-reg}. \new{The bounds that we provide can be achieved with probability $1-\delta$ over the randomness of the training set. We will assume throughout that $0\leq\delta\leq1/C$ for some universal constant $C.$}
% We start from bounding the expected misclassification error in terms of $\hat{\boldsymbol{\eta}}\cdot \boldsymbol{\eta}$, the inner product between the least squares min-norm solution $\hat{\boldsymbol{\eta}}$ and the true weight vector. Lemma \ref{lem-miserror} shows it is important to ensure $\hat{\boldsymbol{\eta}}\cdot\boldsymbol{\eta} \ge 0$ with high probability. We now show the condition that makes $\hat{\boldsymbol{\eta}}\cdot\boldsymbol{\eta} \ge 0$ with high probability.

\subsection{Balanced ensemble}
\label{subsec-lsest}
Recall from Lemma \ref{lem-miserror} that $\hat{\boldsymbol{\eta}}^T\boldsymbol{\eta} > 0$ is needed to ensure that $\mathcal{R}(\hat{\boldsymbol{\eta}})<1/2$. The following lemma shows that this favorable event occurs with high probability provided sufficiently large overparameterization and high SNR.

\begin{lemma}
\label{lem-posicorr}
Assume the balanced $\Sigmab$ ensemble (Definition \ref{def-balancedef}). Fix $\delta\in(0,1)$ and suppose $n$ is large enough such that $n>c\log(1/\delta)$ for some $c>1$. 
Then, there exist constants $C_1,C_2>1$ such that with probability at least $1-\delta$, $\etareg^T\boldsymbol{\eta} > 0$ provided that
\begin{equation}
\label{eq-posicorr02}
   \Vert\etab\Vert_2^2 > \frac{C_1n\sigma^2}{\tau + \normlbd} + C_2\sigma.
\end{equation}
\end{lemma}
% The proof of Lemma \ref{lem-posicorr} is closely related to the proof of Theorem \ref{thm-eqvar01} and is in the Appendix.

We are now ready to state our main result for the balanced ensemble.

\begin{theorem}
\label{thm-eqvar01}
Assume the balanced $\Sigmab$ ensemble (Definition \ref{def-balancedef}). Fix $\delta\in(0,1)$ and suppose large enough $n>c\log(1/\delta)$ for some $c>1$. Further assume that \eqref{eq-posicorr02} holds for constants $C_1$ and $C_2$ $>1$.
% Further let $p > bn$ for a large constant $b$ and condition \eqref{eq-posicorr02} is satisfied. 
Then, there exists constants $C_3, C_4 >1$ such that with probability at least $1-\delta$, 
%the expected 0-1 loss of the least squares ridge estimator $\etareg$ is upper bounded by
\begin{align}
\label{eq-thm01}
 \mathcal{R}(\etareg) \leq   \exp \bigg(\frac{-\Big(\normeta^2 - \frac{C_1n\sigma^2}{\tau + \normlbd} - C_2\sigma\Big)^2}{C_3\max\{1,\frac{n^2\sigma^2}{(\tau+\normlbd)^2}\}\normlbdtwo^2 + C_4\sigma^2} \bigg).
\end{align}
\end{theorem}
The bound for the unregularized LS estimator $\etals$ can be obtained from \eqref{eq-thm01} by setting $\tau = 0$. Thus, with probability at least $1-\delta$, $\mathcal{R}(\etals)$ is upper bounded by
\begin{align}
\label{eq-testls}
    \exp \bigg(\frac{-\Big(\normeta^2 - \frac{C_1n\sigma^2}{\normlbd} - C_2\sigma\Big)^2}{C_3\max\{1,\frac{n^2\sigma^2}{\normlbd^2}\}\normlbdtwo^2 + C_4\sigma^2} \bigg).
\end{align}
%\textcolor{blue}{
By \eqref{eq-testls} we notice that the classification error depends on $\normeta^2$, $\normlbdtwo^2$ and $\sigma^2$. Specifically, increasing $\normeta^2$ and/or decreasing either $\normlbdtwo^2$ or $\sigma^2$ can make the bound smaller. Increasing overparameterization can also help the bound decrease. To see that, consider for example the case  $\lambda_1 = \lambda_2 = \cdots =\lambda_p$. Then,  $\frac{n\sigma^2}{\normlbd}=\frac{n}{p}\normeta^2$ is directly related to the overparameterization ratio $p/n$ and the numerator becomes $(\normeta^2(1-C_1\frac{n}{p}) - C_2\sigma)^2$. %\newadd{Under the conditions in Theorem \ref{thm-linkgen}, decomposing the expression for $\etals$ differently, \cite[Theorem 3.1]{cao2021risk} got that $\mathcal{R}(\etals)$ (also $\mathcal{R}(\etasvm)$) is upper bounded by $\exp\left(\frac{-C\Vert \etab \Vert_2^4}{\Vert \lbdb \Vert_\infty + (\normlbd^2/n) + \sigma^2)}\right)$. Under the same assumptions, the numerator of \eqref{eq-testls} can be simplified to the same as the result in \cite{cao2021risk}, while the denominators are slightly different, in the sense that \cite{cao2021risk} has $\Vert \lbdb \Vert_2^2/n$ instead of $\Vert \lbdb \Vert_2^2$ and an additional $\Vert \lbdb \Vert_\infty$ term.}
%This indicates that $\frac{n\sigma^2}{\normlbd}$ is related to the sufficiency of overparameterization and sufficient overparameterization can help make the bound on the classification error smaller.}
% \begin{proof}
% In the high-SNR regime, $C n \Vert \boldsymbol{\eta} \Vert_2^2 \ge p$. Thus for the denominator of \eqref{eq-thm01}: $\frac{p}{n} \le C\normeta^2$ concludes the result.
% In the low-SNR regime, for the denominator: $\normeta^2 \le \frac{p}{Cn}$ gives the result.
% \end{proof}

\subsection{Isotropic ensemble}
We have a slightly sharper bound on the classification error of the unregularized estimator in the isotropic regime, which is also easier to interpret. For simplicity, we only state the result for the min-norm interpolating solution (aka $\tau=0$).

% \begin{lemma}
% \label{lem-posicorriso}
% Assume that $\Sigmab = \boldsymbol{I}$. Fix $\delta\in(0,1)$ and suppose $n$ is large enough such that $n>c\log(1/\delta)$ for some $c>1$. 

% % Further let $p > b\cdot n$ for a large constant $b$. Then, there exists constant $C>1$ such that with probability at least $1-\delta$, $\hat{\boldsymbol{\eta}}\cdot\boldsymbol{\eta} \ge 0$ provided that:
% % \begin{equation}
% % \label{eq-posicorr02}
% %     (1-\frac{n}{p})\Vert \boldsymbol{\eta}\Vert_2 > C.
% % \end{equation}
% \end{lemma}
% The proof of Lemma \ref{lem-posicorr} is closely related to the proof of Theorem \ref{thm-eqvar01} and is in the Appendix.
%Now we state the theorem for classification error for the isotropic ensemble.
\begin{theorem}
\label{thm-eqvariso}
Assume $\Sigmab = \boldsymbol{I}$. Fix $\delta\in(0,1)$ and suppose large enough $n>c\log(1/\delta)$ for some $c>1$.
There exist constants $C,b>1$ such that with probability at least $1-\delta$, $\etals^T\boldsymbol{\eta} > 0$ provided that
$
    p>b\cdot n$ and $(1-\frac{n}{p})\Vert \boldsymbol{\eta}\Vert_2 > C.
$
Further assume that these two conditions hold for $C,b>1$.  
% Further let $p > bn$ for a large constant $b$ and condition \eqref{eq-posicorr02} is satisfied. 
Then, there exist constants $C_1,C_2>1$ such that with probability at least $1-\delta$:
\begin{equation}
\label{eq-thmiso}
    \mathcal{R}(\hat{\boldsymbol{\eta}}_{\rm LS}) \le \exp \Big(-\Vert \boldsymbol{\eta} \Vert_2^2 \,\frac{\big((1-\frac{n}{p})\Vert \boldsymbol{\eta} \Vert_2 - C_1 \big)^2}{C_2 ( \frac{p}{n} + \normeta^2)} \Big).
\end{equation}
% \begin{equation}
% \label{eq-thm01}
%     \mathcal{R}(\hat{\boldsymbol{\eta}}_{\rm LS}) \le \exp \bigg(-\Vert \boldsymbol{\eta} \Vert_2^2 \frac{\big((1-\frac{n}{p})\Vert \boldsymbol{\eta} \Vert_2 - C_1  \big)^2}{C_2\frac{p}{n} + C_3\normeta^2} \bigg)
% \end{equation}

\end{theorem}
% Theorem \ref{thm-eqvar01} shows that the expected classification test error is small when the signal strength $\Vert \boldsymbol{\eta} \Vert_2^2$ is large enough compared to $p$, because the numerator includes terms such as $\normeta^4$ whereas the denominator includes $\normeta^2$.
%\newadd{Assuming $\normeta \ge C_1$ and after some simplification, the bound on $ \mathcal{R}(\hat{\boldsymbol{\eta}}_{\rm LS})$ is the same as the risk bound in \cite[Corollary 3.3]{cao2021risk}.} 
The bound depends on the overparameterization ratio $p/n$ and the SNR $\|\etab\|_2^2$ when $\Sigmab = \boldsymbol{I}$. To clarify the dependence, it is instructive to  consider separately the following two regimes.
% The bound in Theorem \ref{thm-eqvar01} can be further simplified with different assumptions. We notice that two quantities, $\Vert \boldsymbol{\eta} \Vert_2^2$ and $\frac{p}{n}$, which measure the signal strength and noise strength respectively, play important role (the $p$ here is the sum of the eigenvalues of $\boldsymbol{\Sigma}$, thus measures the noise strength). We thus consider two settings:
    (a) High-SNR regime: $\Vert \boldsymbol{\eta} \Vert_2^2 > \frac{p}{n}$. \ \ \ (b) Low-SNR regime: $\Vert \boldsymbol{\eta} \Vert_2^2 \le \frac{p}{n}$.

The following is an immediate corollary of Theorem \ref{thm-eqvariso} specialized to the two regimes
% In the high-SNR and low-SNR regimes, Theorem \ref{thm-eqvar01} implies the following result.
\begin{corollary}
\label{cor-eqvariso}
Let the same assumptions of Theorem \ref{thm-eqvariso} hold. 
%Fix $\delta\in(0,1)$ and suppose $n$ is large enough such that $n>c\log(1/\delta)$ for some $c>1$. 
Then, there exists constants $C_1>1,C_2>0$ such that with probability at least $1-\delta$, in the high-SNR regime:
\begin{equation}
\label{eq-isolarge}
    \mathcal{R}(\etals) \le \exp \Big(-C_2\cdot\Vert \boldsymbol{\eta} \Vert_2 ^2\cdot{\big((1-\frac{n}{p}) - C_1  \frac{1}{\Vert \boldsymbol{\eta} \Vert}_2\big)^2}\Big),
\end{equation}
and, in the the low-SNR regime:
\begin{equation}
\label{eq-isosmall}
     \mathcal{R}(\etals) \le \exp \Big(-C_2\cdot\Vert \boldsymbol{\eta} \Vert_2 ^4\frac{((1-\frac{n}{p}) - C_1  \frac{1}{\Vert \boldsymbol{\eta} \Vert}_2)^2}{p/n }\Big).
\end{equation}
\end{corollary}
We use  simulations to validate the above bounds. 
In Fig.~\ref{fig-isotesterror}(Left) we fix $n=100$ and plot the  classification error (in log-scale) as a function of $p$ for four different SNR values $3,5,8$ and $10$. Observe that $-\log{\mathcal{R}(\hat\etab_{\rm LS}})$ initially increases until it reaches its maximum at some value of $p>n$ and then decreases as $p$ gets even larger. This ``increasing/decreasing"  pattern is explained by the transition from the high-SNR to the low-SNR regime as per Corollary \ref{cor-eqvariso}. On one hand, the negative of the exponent of the high-SNR bound \eqref{eq-isolarge} is increasing with $p$ for $\normeta^2$. On the other hand, as $p$ increases, and we move in the low-SNR regime, the negative of the exponent in \eqref{eq-isosmall} decreases with $p$ when $p$ is large enough. Additionally, in Figs.~\ref{fig-isotesterror}(Middle,Right), we plot re-normalized values $-\log{\mathcal{R}(\hat\etab_{\rm LS})}/\Vert \boldsymbol{\eta} \Vert_2 ^2$ and $-\log{\mathcal{R}(\hat\etab_{\rm LS})}/\Vert \boldsymbol{\eta} \Vert_2 ^4$. Notice that after appropriate normalization the curves become almost parallel to each other and almost overlap for large values of $\|\etab\|_2^2$, as suggested by \eqref{eq-isolarge} and \eqref{eq-isosmall}. 

\begin{figure}[t!]
\begin{minipage}[b]{1\linewidth}
  \centering
  \centerline{\includegraphics[width=13cm]{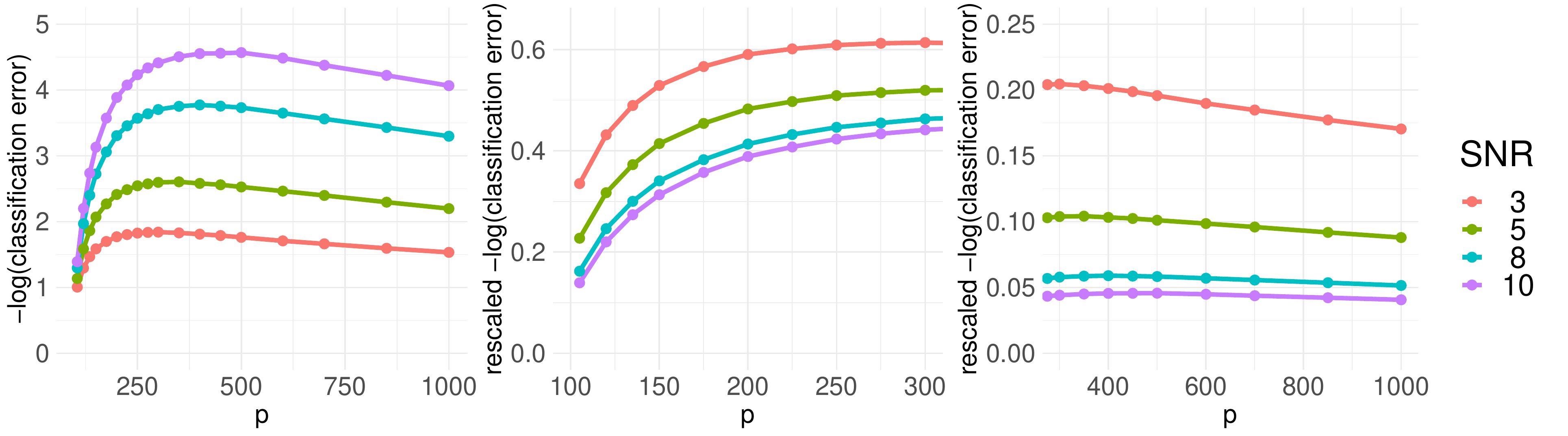}}
%  \vspace{2.0cm}
  %\centerline{(a) Result 1}\medskip
\end{minipage}
\caption{The left plot depicts $-\log{(\text{ classification error})}$ for $\normeta^2= 3,5,8,10$ as a function of $p$. The middle and right figures depict $-\log{(\text{test error})}/\Vert \boldsymbol{\eta} \Vert_2 ^2$ for small $p$ (aka High-SNR regime) and $-\log{(\text{test error})}/\Vert \boldsymbol{\eta} \Vert_2 ^4$ for large $p$ (aka Low-SNR regime),  respectively. The rescalings are as suggested by our bounds \eqref{eq-isolarge} and \eqref{eq-isosmall}, respectively. Note that, after rescaling, the error curves indeed become almost parallel as suggested by Corollary \ref{cor-eqvariso}.}
\label{fig-isotesterror}
\end{figure}

\subsection{Bi-level ensemble}
In this section we study the classification error under the bi-level ensemble in Definition \ref{def-bileveldef}, i.e. when one eigenvalue of $\Sigmab$ is much larger than the rest. Compared to the balanced ensemble, the analysis here depends on a more intricate way on the interaction between the mean vector and the spectrum of $\Sigmab$. To better understand this interaction  we will assume $\betab$ is one-sparse, i.e., the signal is concentrated in one direction. We will also assume, this time without loss of generality, \footnote{Recall $\boldsymbol{\Sigma} = \boldsymbol{V}\boldsymbol{\Lambda} \boldsymbol{V}^T$ and $\boldsymbol{\eta}= \boldsymbol{V}\boldsymbol{\beta}$. Thus,  $\boldsymbol{X}=\boldsymbol{y}\boldsymbol{\eta}^T + \boldsymbol{Q}= (\boldsymbol{y}\boldsymbol{\beta}^T+ \boldsymbol{Z}\boldsymbol{\Lambda}^{\frac{1}{2}})\boldsymbol{V}^T =: \widetilde{\boldsymbol{X}}\boldsymbol{V}^T,
$
where $\boldsymbol{Z}\in\mathbb{R}^{n \times p}$ has IID standard normal entries. 
%By Lemma \ref{lem-miserror}, we know it is critical to bound $\frac{(\hat{\boldsymbol{\eta}}^T\boldsymbol{\eta})^2}{\hat{\boldsymbol{\eta}}^T \boldsymbol{\Sigma}\hat{\boldsymbol{\eta}}}$ and 
With this, it is not hard to check that
$
%\label{eq-ratiolowerbound}
    \frac{(\hat{\boldsymbol{\eta}}^T\boldsymbol{\eta})^2}{\hat{\boldsymbol{\eta}}^T \boldsymbol{\Sigma}\hat{\boldsymbol{\eta}}} = \frac{(\boldsymbol{y}^T(\boldsymbol{X}\boldsymbol{X}^T + \tau\boldsymbol{I})^{-1}\boldsymbol{X} \boldsymbol{\eta})^2}{ \boldsymbol{y}^T(\boldsymbol{X}\boldsymbol{X}^T + \tau\boldsymbol{I})^{-1}\boldsymbol{X}\boldsymbol{\Sigma}\boldsymbol{X}^T (\boldsymbol{X}\boldsymbol{X}^T + \tau\boldsymbol{I})^{-1}\boldsymbol{y}} \notag
    = \frac{(\boldsymbol{y}^T(\widetilde{\boldsymbol{X}}{\widetilde{\boldsymbol{X}}}^T + \tau\boldsymbol{I})^{-1}\widetilde{\boldsymbol{X}} \boldsymbol{\beta})^2}{ \boldsymbol{y}^T(\widetilde{\boldsymbol{X}}{\widetilde{\boldsymbol{X}}}^T + \tau\boldsymbol{I})^{-1}\widetilde{\boldsymbol{X}}\boldsymbol{\Lambda}{\widetilde{\boldsymbol{X}}}^T (\widetilde{\boldsymbol{X}}{\widetilde{\boldsymbol{X}}}^T + \tau\boldsymbol{I})^{-1}\boldsymbol{y}}. \label{eq-ratiolowerboundalt}
$
Hence, after a change of basis, we can equivalently analyze the simplified model with diagonal covariance:
$
    \widetilde{\boldsymbol{x}} = y \boldsymbol{\beta} + \widetilde{\boldsymbol{q}} 
$, $\widetilde{\boldsymbol{q}} \sim N(\boldsymbol{0}, \boldsymbol{\Lambda})$.} 
that $\Sigmab$ is diagonal; thus, $\betab = \etab$. Hence, taking $\betab$ to be one-sparse with (say) the $k$-th element non-zero, the SNR becomes $\frac{\beta_k^2}{\lambda_k}=\frac{\eta_k^2}{\lambda_k}$. Specifically, $k=1$ corresponds to the smallest SNR, for which we expect highest classification risk among all other choices of $k$. For better classification performance, large signal and noise  components should \emph{not} be in the same direction. 
% . We focus on the model with diagonal $\Sigmab$ whose diagonal elements follow the bi-level ensemble (Definition \ref{def-bileveldef}). In this case, $\betab = \etab$. We notice that the model with non-diagonal $\Sigmab$ can be expressed as a model with a diagonal covariance matrix.
  This motivates the following assumption.
\begin{assumption} \label{ass-onesparse}The covariance matrix $\Sigmab$ is diagonal and its diagonal elements follow the bi-level structure in Definition \ref{def-bileveldef}. $\etab$ is one-sparse with nonzero  $k$-th element $\eta_k$ and $k \ne 1$.
\end{assumption}
Under Assumption \ref{ass-onesparse}, the signal strength in the direction of $\Sigmab$ is $\sigma^2 = \lambda_k\eta_k^2$ and the ratio needed to be lower bounded $\frac{(\hat{\boldsymbol{\eta}}^T\boldsymbol{\eta})^2}{\hat{\boldsymbol{\eta}}^T \boldsymbol{\Sigma}\hat{\boldsymbol{\eta}}}$ becomes $\frac{(\hat{\eta}_k\eta_k)^2}{\sum_{i=1}^p\lambda_i\hat{\eta}_k^2}$. The following theorem establishes an upper bound on the classification risk for this setting.
% \begin{lemma}
% \label{lem-posicorrbilevel}
% Let Assumption \ref{ass-onesparse} hold. Fix $\delta\in(0,1)$ and suppose $n$ is large enough such that $n>c\log(1/\delta)$ for some $c>1$. 
% Then, there exist constants $C_1,C_2>1$ such that with probability at least $1-\delta$, $\etareg^T\boldsymbol{\eta} > 0$ provided that
% \begin{equation}
% \label{eq-posicorrbilevel}
%   \eta_k^2 > \frac{C_1n\sigma^2}{\tau + \binormlbd} + C_2{\sigma}.
% \end{equation}
% \end{lemma}
\begin{theorem}
\label{thm-bilevel01}
Let Assumption \ref{ass-onesparse} hold. Fix $\delta\in(0,1)$ and  large enough  $n>c\log(1/\delta)$ for some $c>1$.
Let Assumption \ref{ass-onesparse} hold.
Then, there exist constants $c_1,c_2>1$ such that with probability at least $1-\delta$, $\etareg^T\boldsymbol{\eta} > 0$ provided that
%\begin{equation}
%\label{eq-posicorrbilevel}
$\eta_k^2 > \frac{c_1n\sigma^2}{\tau + \binormlbd} + c_2{\sigma}.$
%\end{equation}
% Further let $p > bn$ for a large constant $b$ and condition \eqref{eq-posicorr02} is satisfied. 
Further assuming the above condition holds, there exist constants $C_i$'s $>1$ such that with probability at least $1-\delta$,
\begin{align}
%\label{eq-bilevel01}
\label{eq-bilevel02}
    \mathcal{R}(\etareg) \leq \exp \Big(\frac{-\Big(\eta_k^2(1 - \frac{C_1n\lambda_k}{\tau + \binormlbd}) - C_2\sigma\big)^2}{A + B + C_6(\lambda_k^2+\sigma^2)} \Big)
    \end{align}
    %&A = C_3\lambda_1^2 \bigg(\frac{\tau + \binormlbd)}{\tau + n \lambda_1 +\binormlbd}\bigg)^2\Big(1+\frac{C_4n\sqrt{\lambda_k}\eta_k}{\tau + \binormlbd}\Big)^2,\notag\\
with $A = C_3\lambda_1^2 \big(\frac{\tau + \binormlbd+C_4n\sigma}{\tau  +\binormlbd + n \lambda_1}\big)^2$ and $
    B = C_5\big(\sum_{i \ne 1,k}\lambda_i^2\big)\big(1+\frac{C_4n\sigma}{\tau + \binormlbd}\big)^2.\notag
$
\end{theorem}
A bound for unregularized estimator $\etals$ can be obtained by setting $\tau = 0$. %Thus, with probability at least $1-\delta$, $\mathcal{R}(\etals)$ is upper bounded by
%\begin{align}
%\label{eq-corbilevel}
%    &\exp \bigg(\frac{-\Big(\eta_k^2(1 - \frac{C_1n\lambda_k}{\binormlbd}) - C_2\sigma\Big)^2}{A + B + C_6(\lambda_k^2+\sigma^2)} \bigg), \ \ \\
%    \text{with} \ &A = C_3\lambda_1^2 \bigg(\frac{\binormlbd+C_4n\sigma}{\binormlbd+n\lambda_1}\bigg)^2, 
%    B = C_5\Big(\sum_{i \ne 1,k}\lambda_i^2\Big)\Big(1+\frac{C_4n\sigma}{\binormlbd}\Big)^2.\notag
%\end{align}
Recall the SNR under Assumption \ref{ass-onesparse} is $\frac{\eta_k^4}{\sigma^2}=\frac{\eta_k^2}{\lambda_k}$. We observe that the bound above depends not only on the SNR, but also on $\lambda_i$, for $i \ne k$, i.e., the spectrum of $\Sigmab$ in \emph{every} direction. Note that similar to previous sections, in \eqref{eq-bilevel02}, the term $\frac{n\lambda_k}{\tau+\binormlbd}$ on the numerator is related to the sufficiency of overparameterization. As we will see, the role of regularization in the bi-level ensemble is more subtle compared to the balanced ensemble and will be discussed in Section \ref{sec-reg}.

%\section{Benign overfitting}
\section{SVM generalization under high overparameterization}
\label{sec-benign}
Now that we have captured the classification error of the min-norm LS estimator $\etals$ in \eqref{eq-testls} and \eqref{eq-thmiso}, and we have established  conditions ensuring $\etals = \etasvm$ in Theorem \ref{thm-linkgen} and Theorem \ref{thm-linkiso}, we establish sufficient conditions under which the classification error of hard-margin SVM vanishes as the overparameterization ratio $p/n$ increases. Note that the bi-level ensemble will \emph{not} satisfy the first condition in Theorem \ref{thm-linkgen}, hence we focus on the balanced and isotropic ensembles. For later use, define the term in \eqref{eq-linkgen01} as $\lambda_{*} := 72\big(\Vert \boldsymbol{\lambda}\Vert_2\cdot n\sqrt{\log{n}} + \Vert \boldsymbol{\lambda}\Vert_{\infty}\cdot n\sqrt{n}\log{n} + 1\big)$. We first focus on a special case where $\betab = \begin{bmatrix}\beta &\beta &... &\beta\end{bmatrix}^T$ for simplicity.

\begin{corollary}
\label{cor-benign01}
Let the same assumption as in Theorem \ref{thm-eqvar01} hold with $\tau = 0$ and sufficiently large $n > {C}/{\delta}$ for some $C>1$. Also let $\betab = \begin{bmatrix}\beta &\beta &... &\beta\end{bmatrix}^T$.  Then, for large  $C_i$'s $>1$, with probability at least $(1 - \delta)$, $\etasvm$ linearly interpolates the data and the classification error $\mathcal{R}(\etasvm)$ approaches $0$ as $p \to \infty$  provided %either of 
the two following sets of conditions on $\normlbd$ hold:
%\noindent(1). If $\Vert \lbdb \Vert_2^2 \le \beta^2\normlbd$,
\begin{align*}
      \normlbd > \max\{\lambda{_{*}}, C_1\beta^2n^2\log(2n)\}~~ \text{and}~~ \max\{\beta^{-2}\Vert \lbdb \Vert_2^2, \normlbd\} \le C_2 \beta^2p^{\alpha}, \ \text{for} \ \alpha <2.
\end{align*}
\end{corollary}
%\noindent(2). If $\Vert \lbdb \Vert_2^2 > \beta^2\normlbd$, 
%\begin{align*}
%      \normlbd > \max\{\lambda{_{*}}, C_1\beta^2n^2\log(2n)\},\ \ \text{and} \ \beta\sqrt{\normlbd} < \Vert \lbdb \Vert_2  \le C_3 \beta^2p^{\alpha}, \ \ \text{for} \ \ \alpha < 1.
%\end{align*}

%\subsection{Isotropic ensemble}
%\label{sec-isobound}
%In this section, we show an upper bound on the classification error of $\etals$ as well as the conditions that guarantee benign overfitting for $\etasvm$ for the isotropic case. Those results are derived in \cite{wang2020benign}. We start from the condition that makes $\etals^T\etab > 0$ with high probability.
\newadd{The first condition above requires sufficient overparameterization and the second one a large enough SNR. To see that, note for the setting of Corollary \ref{cor-benign01} that $\text{SNR} = p^2\beta^2/\normlbd$. Thus, the second condition imposes $\text{SNR} \ge cp^{2-\alpha}$ implying that $\text{SNR} \ge cp^{\epsilon}$ for some $\epsilon > 0$. } 

Corollary \ref{cor-benign01} assumes that $\betab$ has equal elements. Now we allow the mean vector $\etab$ to have different entry values but let $\Sigmab = \boldsymbol{I}$, then we have the following result.
\begin{corollary}
\label{cor-linkandtester}
Let the same assumptions as in Theorem \ref{thm-eqvariso} hold and $n$ sufficiently large such that $n > {C}/{\delta}$ for some $C > 1$, thus $\Sigmab = \boldsymbol{I}$. Then, for large enough positive constant $C_i$'s $> 1$, $\hat{\boldsymbol{\eta}}_{\rm SVM}$ linearly interpolates the data and the classification error $\mathcal{R}(\hat{\boldsymbol{\eta}}_{\rm SVM})$ approaches zero as $(p/n) \to \infty$ with probability at least ($1 - \delta$) provided either of the two following sets of conditions on the number of features $p$ and mean-vector $\boldsymbol{\eta}$ hold:

\noindent (1). High-SNR regime 
\begin{align*}
    \frac{1}{C_1}n\normeta^2 > p > \max\{10n\log n+n-1, C_2n\sqrt{\log(2n)}\normeta\}.
\end{align*}
\noindent (2). Low-SNR regime
\begin{align*}
    p > \max\{10n\log n +n-1, C_3n\sqrt{\log(2n)}\normeta, n\normeta^2\}, \ \text{and} \ \normeta^4 \ge C_4(\frac{p}{n})^{\alpha}, \ \text{for} \  \alpha >1.
\end{align*}
\end{corollary}

We first compare Corollaries \ref{cor-benign01} and \ref{cor-linkandtester} assuming both $\Sigmab = \boldsymbol{I}$ and $\betab = \begin{bmatrix}\beta &\beta &... &\beta\end{bmatrix}^T$. Then $\normlbd = \Vert \lbdb \Vert_2^2 = p$ and $\normeta^2 = \normbeta^2 = p\beta$. It is not hard to check that under those assumptions, they both require
%\begin{align*}
    $p > Cn^2\log(2n)$,
%\end{align*}
for sufficiently large constant $C$. One might expect that a sharper condition can be obtained by Corollary \ref{cor-linkandtester} when $\Sigmab = \boldsymbol{I}$. Unfortunately, that is not the case because although the first condition in Theorem \ref{thm-linkiso} is sharper than that of Theorem \ref{thm-linkgen}, the second conditions become equivalent when $\Sigmab = \boldsymbol{I}$ and $\betab = \begin{bmatrix}\beta &\beta &... &\beta\end{bmatrix}^T$ and are stronger than the first condition.

\begin{figure}[t]
\begin{minipage}[b]{1\linewidth}
  \centering
  \centerline{\includegraphics[width=8.5cm]{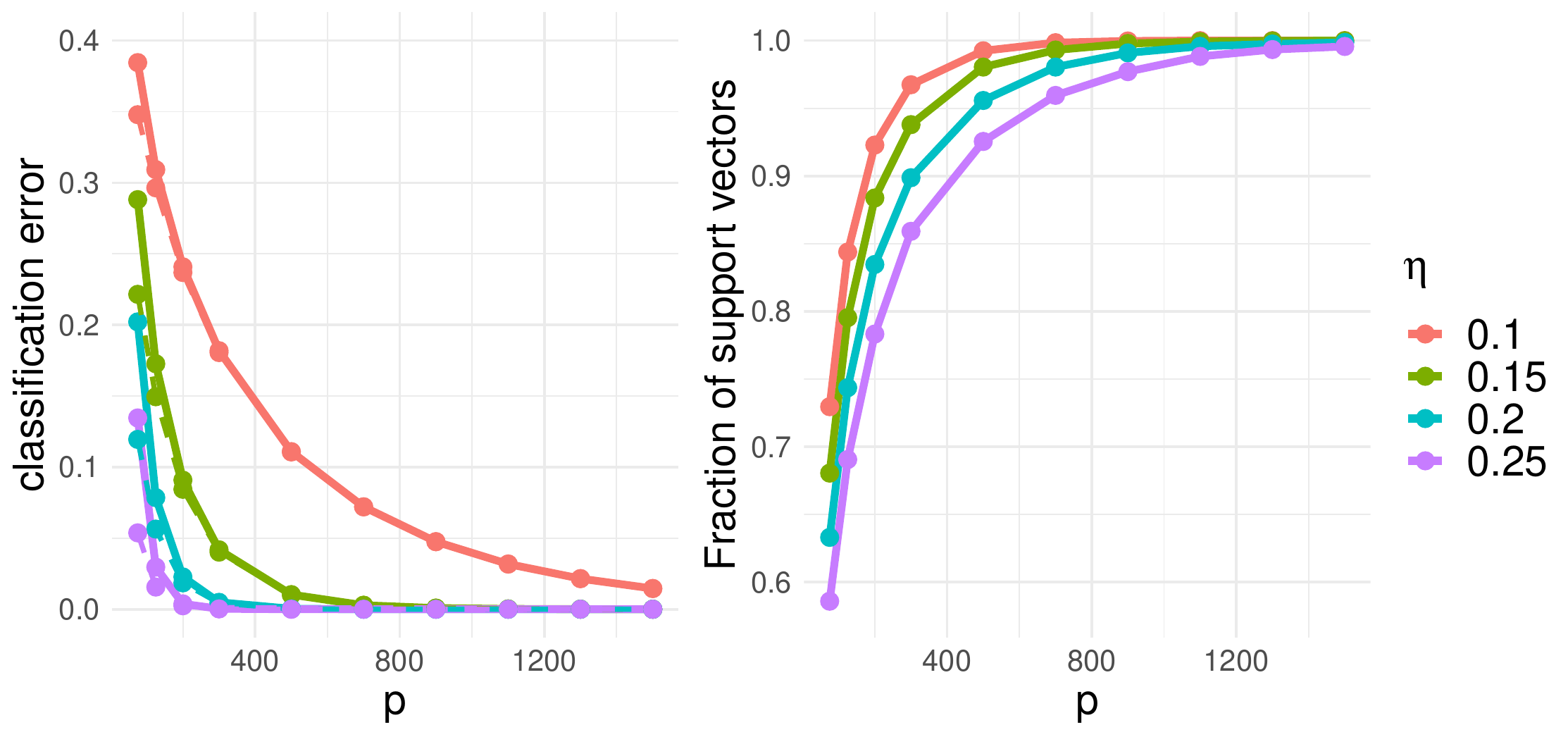}}
% \vspace{0.0cm}
  %\centerline{(a) Result 1}\medskip
\end{minipage}
\caption{Numerical demonstration of benign overfitting for the GMM. The left plot shows the  classification error with $n=50$ and mean vector $\etab$ with entries  $\eta_1 = \ldots = \eta_p = \eta$. The solid lines correspond to the LS estimates and the (almost overlapping) dashed lines show the SVM solutions. The error vanishes with  $p\rightarrow\infty$ indicating benign overfitting as predicted by Corollary \ref{cor-benign01}. The right plot illustrates the proportion of support vectors in the same setting.
%With sufficient overparameterization $p=1500$, the proportion of support vectors gets smaller if the ratio c is small compared with $n$. The right plot uses $n/\frac{p}{\Vert \boldsymbol{\eta} \Vert_2}$ as x-axis and we observe that all the curves move close to each other.
}
\label{fig-benign01}
\end{figure}

\begin{remark}[Comparison of noiseless conditions to \citet{chatterji2020finite}]\label{rem:cha-noiseless}
\new{Using different tools to directly analyze $\hat\etab_{\rm SVM}$ (see Section  \ref{sec-comparechatterji}), \citet[Thm.~3.1]{chatterji2020finite} proved that for noisy mixtures with possibly adversarial corruptions and with subGaussian features} 
\begin{align}
     p > C_1\max\{n\normeta^2, n^2\log(n)\}
     ~~\text{and}~~
     \|\etab\|_2^4 \geq C_2 p^\alpha, ~\alpha>1, \label{eq:Cha}
     %\Vert \boldsymbol{\eta} \Vert_2^4 =\omega(p) 
\end{align}
suffice for benign overfitting, \new{i.e., for making the classification error asymptotically approach the noise level as $p/n \to \infty$. Our corollary \ref{cor-linkandtester} holds for the special case of Gaussian features and \emph{noiseless} labels. Since labels are not corrupted, the noise floor is zero. In this special case, our result relaxes significantly the sufficient conditions for which the risk approaches zero compared to a direct application of their result.
 To see this note that} condition \eqref{eq:Cha} is reminiscent of our `low-SNR regime' condition (2) in Corollary \ref{cor-linkandtester}. First, our condition relaxes the requirement on overparameterization from $p>Cn^2\log(n)$ in \eqref{eq:Cha} to $p>Cn\sqrt{\log(2n)}$. Second, our condition $\|\etab\|_2^4=\omega(p/n)$ on the SNR can be equivalent to theirs $\|\etab\|_2^4=\omega(p)$, for example in a setting of constant $n$. 
%This result is also close to \cite{jin2009impossibility}, who studied the $s$-sparse mean vector and thus $\normeta^2 = Cs$. \cite{jin2009impossibility} shows that the model learning is impossible when $s = \Omega(\sqrt{p})$.} 
% However, we note that our analysis yields additional sufficient conditions for benign overfitting as presented in condition (1) of Corollary \ref{cor-linkandtester}. 
In order to better understand different conditions, consider a somewhat concrete setting in which $n$ is fixed and only $p$ and $\normeta$ grow large. Then for the classification error to go to $0$ as $p \to \infty$, \citet{chatterji2020finite} requires (see \eqref{eq:Cha}) that $\normeta = \Theta(p^\beta)$ for $\beta \in (\frac{1}{4}, \frac{1}{2}]$. Instead, our Corollary \ref{cor-linkandtester} requires that $\normeta = \Theta(p^\beta), \beta\in(1/4,1/2]$ (low-SNR) or $\normeta = \Theta(p^\beta)$ for $\beta\in(1/2,1)$ (high-SNR).  \newadd{We repeat that this improvement is for zero label noise.} \new{In Section \ref{sec-labelnoise}, where we study a noisy GMM, we show that our \emph{sufficient} conditions can indeed change in the noisy case.}
% The results of \citet{chatterji2020finite} apply more generally also in the case of a bounded number of possibly adversarial label corruptions. As discussed also in \kw{Section 8.2.3}, our analysis is of rather different nature compared to \citet{chatterji2020finite}. In Appendix \ref{sec-labelnoise}, we apply our method to a simple probabilistic label-noise model for which we derive benign overfitting conditions that match these of \citet{chatterji2020finite}.}
\end{remark}

%\newadd{As we discussed before, for the isotropic ensemble, our result in Theorem \ref{thm-eqvariso} is the same as the risk bound in \cite{cao2021risk}, hence the benign overfitting conditions are matching for finite $n$.}

%Note that the conditions above imply that the balanced ensemble requirement is satisfied, hence the averaging estimator performs better than $\etals$ and $\etasvm$. 
Finally, we present numerical illustrations validating Corollary \ref{cor-benign01}. In Fig.~\ref{fig-benign01}, we let $\eta_1 = \cdots = \eta_p = \eta$ with $\eta = 0.1, 0.15, 0.2$ or $0.25$. Thus, $\Vert \boldsymbol{\eta} \Vert_2^2 = \eta^2{p}$. We also fix $n=50$. The eigenvalues of $\Sigmab$ are generated as follows: $\lambda_1 = 0.005p$, $\lambda_p = 0.2\cdot\frac{0.995p}{p-1}$ and $\lambda_2 = \cdots = \lambda_{p-1}=\frac{p-\lambda_1-\lambda_{p-1}}{p-2}$. This setting is different from the isotropic case, but ensures $\normlbd \le C_1p$, $\Vert \lbdb \Vert_2 \le C_2p^{1/2}$ and conditions in Corollary \ref{cor-benign01} are satisfied. In Fig.~\ref{fig-benign01}(Left), we plot the classification error as a function of $p$ for both LS estimates (solid lines) and SVM solutions (dashed lines). The solid and dashed curves almost overlap, so it can be hard to distinguish in the figure. We verify that as $p$ increases, the  classification error decreases towards zero. Similarly, Fig.~\ref{fig-benign01}(Right) reaffirms that all the data points become support vectors for sufficiently large $p$ (cf. Theorem \ref{thm-linkiso}). In addition, Fig. \ref{fig-benign01}(Left) shows that the classification error of SVM solutions is slightly better than that of LS estimates when $p$ is small. The error becomes the same for large $p$, since then the SVM solutions are the same as LS solutions. Another observation is that the classification error goes to zero very fast when SNR is high (e.g., purple curves), but the probability of interpolation increases at a slow rate. In contrast, when the SNR is low (e.g., red curves), the probability of interpolation increases fast, but the classification error decreases slowly. Intuitively, the harder the classification task (aka lower SNR), the larger the classification error and the more data points become support vectors. 

\section{On the role of regularization}
\label{sec-reg}
{In this section, we discuss how the $\ell_2$-regularization affects the classification error of $\etareg$ under the balanced and bi-level ensembles. For convenience, we start with a brief summary of our findings.}

\noindent (a). {Balanced ensemble:
\begin{enumerate}
    \item The classification error is decreasing with $\tau$. Thus, it is minimized as $\tau \to +\infty$.
    \item Our bounds verify that in the limit $\tau \to +\infty$, $\etareg$ has the same error as the so-called averaging estimator $\etaavg=\frac{1}{n}\sum_{i\in[n]}y_i\mathbf{x}_i$, where $\mathbf{x}_i^T$ is the $i$-th row of $\boldsymbol{X}$.
    \item The averaging estimator is the best among the ridge-regularized estimator and the LS interpolating estimator.
\end{enumerate}
\noindent (b). Bi-level ensemble:
\begin{enumerate}
    \item Our upper bound on the classification error is \emph{not} monotonically decreasing with $\tau$. Hence regularization might \emph{not} be helpful and the averaging estimator is \emph{not} optimal.
    \item There are regimes where $\tau = 0$ is optimal. Specifically, the interpolating estimator performs the best when $\lambda_1$ is large enough compared to other eigenvalues of $\Sigmab$ and overparameterization is sufficient.
\end{enumerate}
These observations are illustrated in Figures \ref{fig-testerror}, \ref{fig-bilevel01} and \ref{fig-bilevel02} which are discussed in detail in the next sections.
}
\subsection{Balanced ensemble}
%Recall that by Theorem \ref{thm-eqvar01}, with high probability, $\mathcal{R}(\etareg)$ is upper bounded by
%\begin{align}
%\label{eq-thm01b}
%    \exp \bigg(\frac{-\Big(\normeta^2 - \frac{C_1n\sigma^2}{\tau + \normlbd} - C_2\sigma\Big)^2}{C_3\max\{1,\frac{n^2\sigma^2}{(\tau+\normlbd)^2}\}\sumlbdsq + C_4\sigma^2} \bigg).
%\end{align}
We first analyze the bound in \eqref{eq-thm01}. Observe that both the terms $\frac{C_1n\sigma^2}{\tau + \normlbd}$ in the numerator and $\frac{n^2\sigma^2}{(\tau+\normlbd)^2}$ in the denominator decrease as the regularization parameter $\tau$ becomes larger. %thus leading to smaller upper bound on the classification error.
This suggests that, under the balanced ensemble, increasing regularization always helps decrease the error. The remaining terms, $\sigma$ and $\normlbdtwo^2$ in \eqref{eq-thm01}, that are not affected by changing $\tau$ reflect the intrinsic structure of the model and characterize the difficulty of the learning task. 
As $\tau \to +\infty$, the ``regularization-sensitive" terms vanish and only those ``regularization-insensitive" terms remain. Specifically, the upper bound on classification error becomes
\begin{align}
\label{eq-testreginfty}
    \exp \Big(-{\big(\frac{\normeta^2}{\sigma} -  C_2\big)^2}\big/\big({C_3\frac{\normlbdtwo^2}{\sigma^2} + C_4}\big) \Big).
\end{align}
%The bound above is closely related to the SNR $\frac{\Vert \boldsymbol{\eta} \Vert_2^4}{\sigma^2}$. 
%The term $\sigma$ in the numerator and $\normlbdtwo^2$ in the denominator characterize how the noise affects the performance of estimation. More importantly, we show 
In Appendix \ref{subsec-avgest} we show that the bound in \eqref{eq-testreginfty} is the same as the bound for the so-called averaging estimator which simply returns
\begin{align}
\label{eq-defavg}
    \etaavg = \boldsymbol{X}^T\boldy/n.
\end{align}
%The averaging estimator can be understood as the empirical mean of features multiplied by the labels and is a natural estimator for the GMM. 
Therefore, under the balanced ensemble, the classification performance of the averaging estimator is superior to that of the ridge and interpolating estimators. A similar finding was recently reported in \citet{mignacco2020role}, but in an asymptotic setting and only for the isotropic case. %Due to the space limit, we defer the results for the averaging estimator in Section \ref{subsec-avgest} of the supplementary material.

\begin{figure}[t]
\begin{minipage}[b]{1\linewidth}
  \centering
  \centerline{\includegraphics[width=13cm]{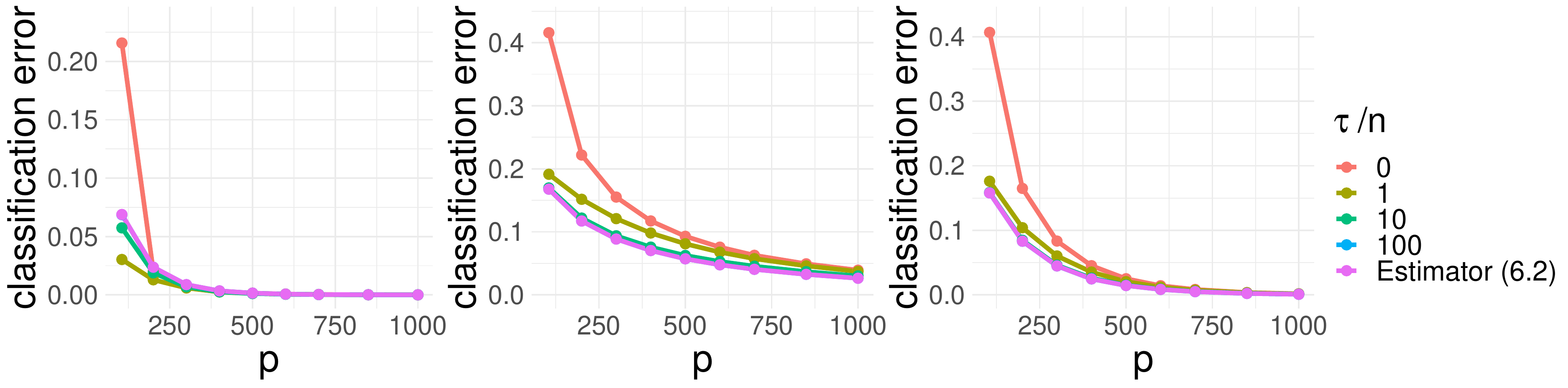}}
%  \vspace{2.0cm}
  %\centerline{(a) Result 1}\medskip
%  \setlength{\abovecaptionskip}{5pt plus 1pt minus 10pt}
\end{minipage}
\caption{Classification error as a function of $p$ under $3$ model setups with different regularization parameter values $\tau$. In the right plot all $\eta_i$'s are the same. The middle/left ones correspond to the extreme cases of largest/smallest values $\sigma^2$; see text for details. Also plotted (in magenta) the averaging estimator defined in \eqref{eq-defavg}. As predicted by our theory, for fixed $\normeta$ and $\normlbdtwo^2$, larger $\tau$ and smaller $\sigma^2$ lead to better performance and $\etareg$ has the same performance as $\etaavg$ when $\tau$ is large.}
\label{fig-testerror}
\end{figure}

%\subsection{Averaging estimator}
We now use numerical simulations to validate the above claims. In our simulations in Fig.~\ref{fig-testerror}, we fix $n = 100$ and vary $p$. To check \eqref{eq-thm01}, for each $p$, we set $\normlbdtwo$ to be $p$ and $\lambda_1 = \sqrt{0.0125p}$, $\lambda_p=\sqrt{0.000125p}$ and all the rest $\lambda_i$'s are $\sqrt{(p-\lambda_1^2-\lambda_p^2)/(p-2)}$. This setup makes $\lambda_1$ slightly larger than other $\lambda_i$'s and $\lambda_p$ slightly smaller. For example, when $p = 1000$, then $\lambda_1 = 3.53$, $\lambda_p = 0.35$ and all other $\lambda_i$'s are $0.99$. Note that although $\lambda_i$'s are not equal, those settings still satisfy the requirements of the balanced ensemble. Then, we look at different signals $\etab$ with the same strength $\normeta^2 = (0.125^2)p$. To make $\sigma^2$ in $\eqref{eq-thm01}$ different, we consider $3$ cases: all $\eta_i$'s are the same, only $\eta_1$ nonzero and only $\eta_p$ nonzero. The right plot in Fig.~\ref{fig-testerror} shows the classification error of the same-$\eta_i$ case, the middle one shows the nonzero-$\eta_1$ case and the left plot shows the nonzero-$\eta_p$ case. To see the role of regularization, we look at the classification error with different $\tau$ values and also include the averaging estimator. We can see that among the three plots, the nonzero-$\eta_p$ case (left) has the smallest classification error and the nonzero-$\eta_1$ case (middle) has the largest classification error. This is \newadd{in agreement} with the fact that the nonzero-$\eta_p$ case has the smallest $\sigma^2$ and the the nonzero-$\eta_1$ case has the largest $\sigma^2$. For large $p$, regularization always helps reduce the classification error. When $\tau$ is large, the performance of $\etareg$ becomes the same as that of $\etaavg$. All those observations are consistent with Theorem \ref{thm-eqvar01}.

\subsection{Bi-level ensemble}
\label{sec-bilevel}
We have seen  that regularization is always useful in reducing the classification risk in the balanced ensemble. For the bi-level ensemble, the story is quite different: the classification error is no longer monotonically decreasing as $\tau$ increases. Recall that under Assumption \ref{ass-onesparse}, with high probability, $\mathcal{R}(\etareg)$ is upper bounded by \eqref{eq-bilevel02}.
%\begin{align}
%\label{eq-bilevel01}
%    &\exp \bigg(\frac{-\Big(\eta_k^2(1 - \frac{C_1n\lambda_k}{\tau + \binormlbd}) - C_2\sigma\Big)^2}{A + B + C_6(\lambda_k^2+\sigma^2)} \bigg), \label{eq-bilevel02b}\\
    %&A = C_3\lambda_1^2 \bigg(\frac{\tau + \binormlbd)}{\tau + n \lambda_1 +\binormlbd}\bigg)^2\Big(1+\frac{C_4n\sqrt{\lambda_k}\eta_k}{\tau + \binormlbd}\Big)^2,\notag\\
%    \text{with} \ &A = C_3\lambda_1^2 \bigg(\frac{\tau + \binormlbd+C_4n\sigma}{\tau  +\binormlbd + n \lambda_1}\bigg)^2, \ B = C_5\Big(\sum_{i \ne 1,k}\lambda_i^2\Big)\Big(1+\frac{C_4n\sigma}{\tau + \binormlbd}\Big)^2.\notag
%\end{align}
Moreover, when $\tau$ goes to infinity, %the bound becomes
%\begin{align}
%    \label{eq-bilevelinfty}
%    \exp\bigg(-\frac{\Big(\eta_k^2 - C_2\sigma \Big)^2}{C_3\sum_{i= 1}^k\lambda_i^2 + C_4\sigma^2}\bigg).
%\end{align}
it is not hard to check that this bound matches the corresponding for the averaging estimator (see Appendix \ref{subsec-avgest}). Thus, in this case the averaging estimator is \emph{not} optimal.

\begin{figure}[t]
\begin{minipage}[b]{\linewidth}
  \centering
  \centerline{\includegraphics[width=17.5cm]{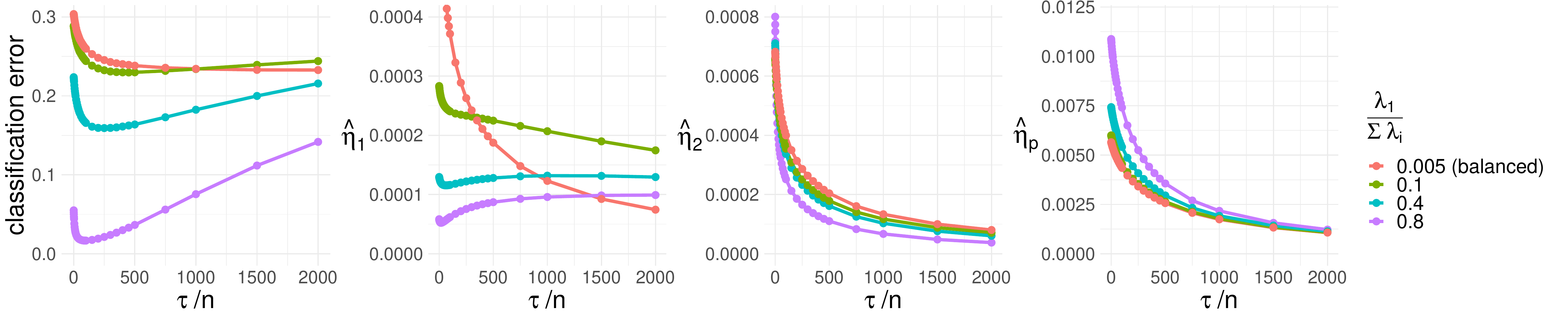}}
 %\vspace{0.05cm}
  %\centerline{(a)}\medskip
\end{minipage}
\caption{The left-most plot shows the classification error for different $\lambda_1/\binormlbd$ ratios with $n=100$ and mean vector $\etab$ with entries  $\eta_1 = \cdots = \eta_{p-1} = 0$ and $\eta_p = \sqrt{200}$. Other plots show how $\hat{\eta}_i$'s vary with $\tau$. As predicted by the theory, under the balanced ensemble, $\hat{\eta}_1$ and $\hat{\eta}_2$ decrease at similar rates, but have different behaviors when $\Sigmab$ has a highly spiky eigen-structure. See text for details.}
\label{fig-bilevel01}
\end{figure}
To see why \eqref{eq-bilevel02} is no longer monotonically decreasing in $\tau$, the term $\frac{\tau + \binormlbd+C_4n\sigma}{\tau  +\binormlbd + n \lambda_1}$  in $A$ is increasing in $\tau$ and thus $A$ is increasing in $\tau$ when $\lambda_1 > C_4\sigma = C_4\sqrt{\lambda_k}\eta_k$, i.e., when $\lambda_1$ is large enough compared to $\lambda_k$ and $\eta_k$. Note that \eqref{eq-bilevel02} is obtained by lower bounding $\frac{(\eta_k\hat{\eta}_k)^2}{\sum_{i=1}^p\lambda_i\hat{\eta}_i^2}$ and $A$ is related to the term $\lambda_1\hat{\eta}_1^2$, i.e., the estimate in the direction of $\lambda_1$. %hence larger $\tau$ may not help estimate $\eta_1$ very much and sometimes can even degrade the estimation.
Since $\lambda_1$ is much larger than others, even if the regularization is useful in other directions, the performance won't keep improving as $\tau$ increases, because it won't help in the direction with the largest ``noise". Term $B$ in \eqref{eq-bilevel02}, on the other hand, is related to $\lambda_i\hat{\eta}_i^2$, for $i \ne 1$ or $k$, and it becomes smaller as $\tau$ becomes larger, thus regularization is useful in these directions. $B$ becomes less important than $A$ if $\lambda_1$ becomes larger than other $\lambda_i$'s, hence the regularization becomes less helpful in this case. Another observation is that in the numerator of \eqref{eq-bilevel02}, the term $\frac{n\lambda_k}{\tau + \binormlbd}$ decreases as $\tau$ increases. Note that when $\lambda_2 =\cdots= \lambda_p$, $\frac{n\lambda_k}{ \binormlbd} = \frac{n}{p-1}$, hence this term measures the sufficiency of overparameterization. When the overparameterization is sufficient, i.e., $p$ is much larger than $n$, $\frac{n\lambda_k}{\binormlbd}$ is already very small, hence $\frac{n\lambda_k}{\tau + \binormlbd}$ won't be much smaller than $\frac{n\lambda_k}{\binormlbd}$ even for large $\tau$. In other words, strong regularization won't help very much. Summarizing all those observations, we conclude that the regularization becomes less useful in reducing the classification error when $\lambda_1$ is large enough relative to other eigenvalues and when overparameterization is sufficient. Under those conditions, $\tau = 0$ minimizes \eqref{eq-bilevel02}, therefore, the interpolating estimator $\etals$ has better performance than the regularized estimators. Since small or zero regularization can provide the best estimation in the bi-level setting with Assumption \ref{ass-onesparse} in the overparameterization regime, it seems that the model structure itself provides the implicit regularization. This phenomenon is also discussed in \citet{bartlett2020benign,muthukumar2020harmless,kobak2018optimal,muthukumar2020classification,tsigler2020benign}.

The following numerical experiments validate our analysis. First in Fig.~\ref{fig-bilevel01}, we illustrate how the ratio $\lambda_1/\binormlbd$ affects the classification error and the role of regularization. In our simulation, $n = 100$, $p=200$. $\etab \in \mathbb{R}^{200}$ is one-sparse and only the last element is non-zero, i.e., $\etab^T = [0,0,\cdots,0, \sqrt{200}]$. For the eigenvalues of $\Sigmab$, in the balanced ensemble, the diagonal elements are all equal, i.e., $\lambda_1 =\cdots =\lambda_{200} = 150$. In the bi-level ensemble, we fix $\normlbd = 200\cdot 150$ and let $\lambda_1 = \alpha \normlbd$, with $\alpha = 0.1, 0.4$ and $0.8$. Then $\lambda_2 = \cdots = \lambda_p = (1-\alpha)\cdot \normlbd/(p-1)$. Note that larger $\alpha$ makes $\lambda_1/\binormlbd$ higher and that $\alpha =0.005$ in the balanced ensemble. Fig.~\ref{fig-bilevel01} illustrates how classification error and $\hat{\eta}_i$'s change with the regularization parameter $\tau$. Based on previous analysis, we divide those $\hat{\eta}_i$'s into 3 groups, $\{\hat{\eta}_1; \hat{\eta}_2,\cdots,\hat{\eta}_{199}; \hat{\eta}_{200} \}$. $\hat{\eta}_1$ has true value 0 with large noise, $\hat{\eta}_2,\cdots,\hat{\eta}_{199}$ have true value 0 with small noise and $\hat{\eta}_{200}$ has non-zero true value with small noise. The figures show $\hat{\eta}_1$, $\hat{\eta}_2$ and $\hat{\eta}_{200}$. We can see that the classification error keeps decreasing as $\tau$ increases for the balanced ensemble (red curves). Part of the reason is that $\hat{\eta}_1$, $\hat{\eta}_2$ and $\hat{\eta}_{200}$ decrease at similar rates. In contrast, for the bi-level regime, as $\tau$ increases, the classification error decreases first, then increases. $\hat{\eta}_2$ decreases with $\tau$, but $\hat{\eta}_1$ increases slowly with $\tau$ for large $\tau$ when $\lambda_1/\binormlbd$ is large. This is consistent with Theorem \ref{thm-bilevel01} in which $A$ is increasing in $\tau$ when $\lambda_1$ is large enough. When $\lambda_1$ is not large enough, as the green curve shows, $\hat{\eta}_1$ decreases at a similar rate as $\hat{\eta}_1$ and all the curves are closer to those of the balanced ensemble.

\begin{figure}[t]
\centering
\begin{minipage}[b]{0.24\linewidth}
  \centering
  \centerline{\includegraphics[width=4.5cm]{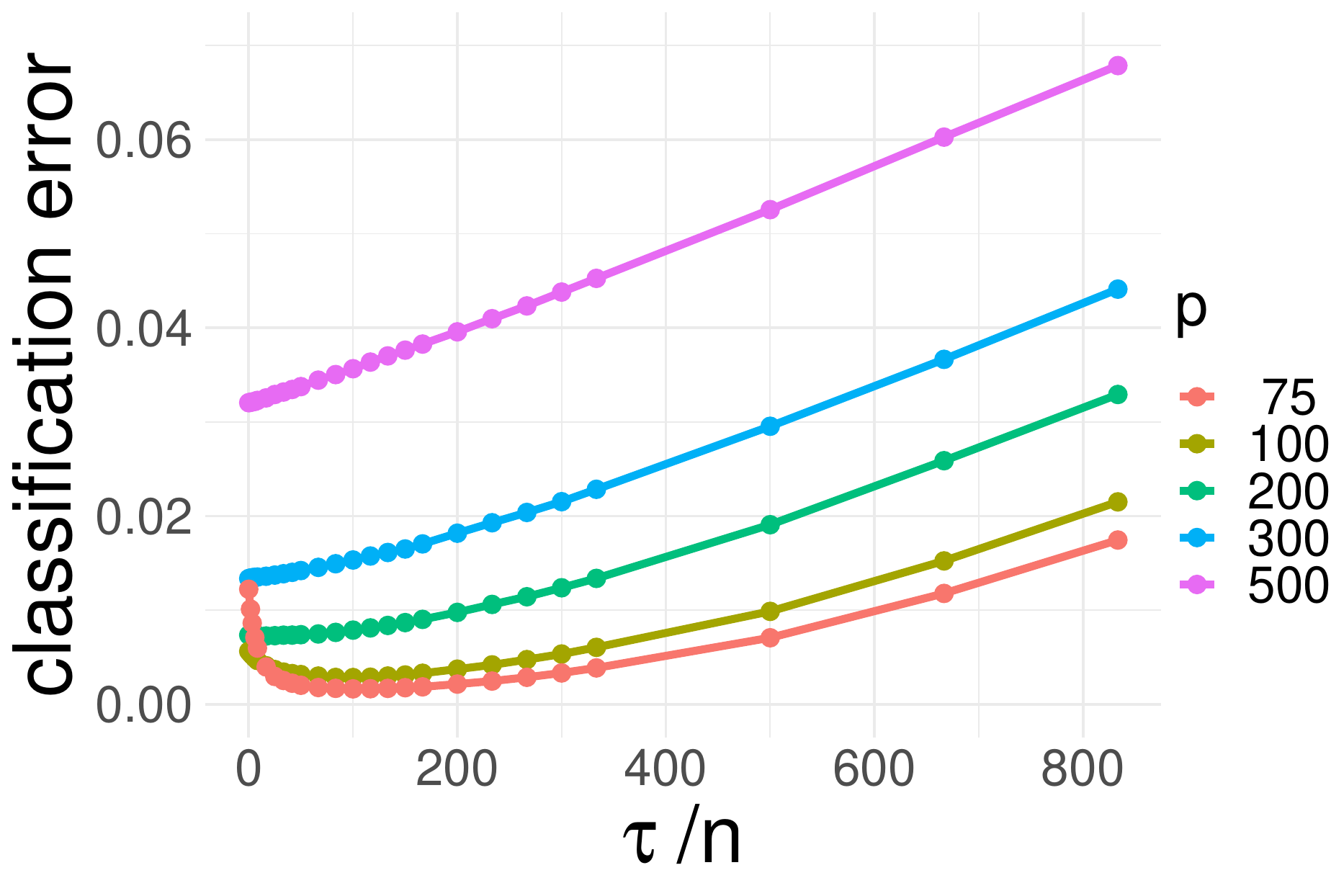}}
 %\vspace{0.05cm}
  \centerline{(a)}\medskip
\end{minipage}
\begin{minipage}[b]{0.75\linewidth}
  \centering
  \centerline{\includegraphics[width=8.35cm]{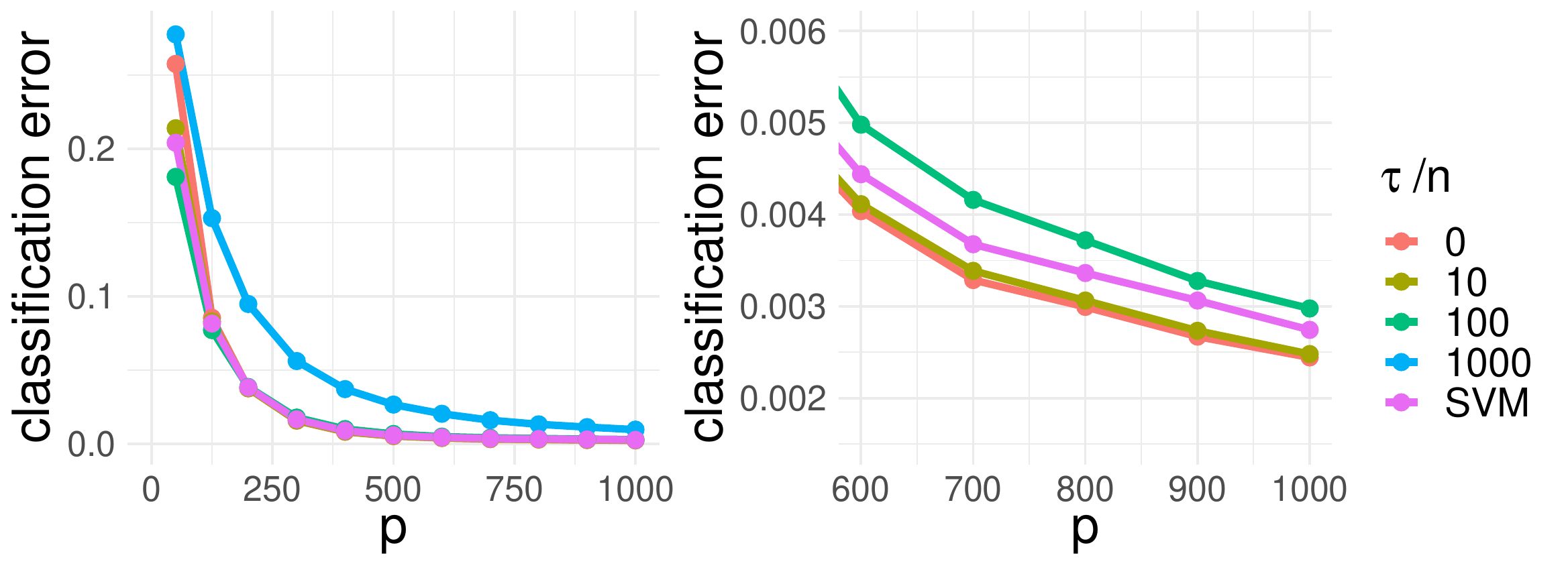}}
 %\vspace{0.05cm}
  \centerline{(b)}\medskip
\end{minipage}
\caption{For all the plots here, we fix $n =30$, $\lambda_2=,...,=\lambda_p =50$ and $\lambda_1/\binormlbd = 10$ (corresponding to the bi-level ensemble). (a) Classification error versus $\tau/n$ for different $p$ and fixed $\eta_p =25$. Observe that the classification error increases monotonically with $\tau$ for large $p$. (b) A regime where $\etals$ performs the best and its classification error approaches $0$ as $p \to \infty$. Specifically, we set here $\eta_p = 0.1\sqrt{50}p^{0.6}$. See text for details.}
\label{fig-bilevel02}
\end{figure}

Finally, we illustrate how the overparameterization ratio $p/n$ affects the role of regularization in Fig.~\ref{fig-bilevel02} (a). Here to guarantee $p/n$ sufficiently large, we fix $n=30$. We plot how the classification error changes with $\tau$ for $p = 75, 100, 200, 300$ and $500$. $\etab$ is one-sparse with $\eta_p = 25$. For eigenvalues of $\Sigmab$, to make $\lambda_1/\binormlbd$ sufficiently large, we set $\lambda_2 =\cdots= \lambda_p = 50$ and $\lambda_1 = 10\binormlbd$. We observe from Fig.~\ref{fig-bilevel02} (a) that when $p$ is large, the classification error increases with $\tau$, thus $\tau=0$ performs the best. The optimal choice of $\tau$ is larger than $0$ when $p$ is not large enough (e.g., $p = 75$ and $100$). In Fig.~\ref{fig-bilevel02} (b), we show a regime \newadd{where} $\etals$ performs better than $\etareg$ and $\etasvm$ when $p$ is large. Again we fix $n =30$ to ensure sufficient overparameterization. Same as before, we set $\lambda_2 =\cdots= \lambda_p = 50$ and $\lambda_1 = 10\binormlbd$. To make the classification error 
approach $0$ as $p \to \infty$, according to Corollary \ref{cor-bilevelboundzero} in Appendix, we set $\eta_p = 0.1\sqrt{50}p^{0.6}$. Fig.~\ref{fig-bilevel02} (b)(Left) shows the classification error over different $p$ for various $\tau$. We also added the curve for $\etasvm$. Fig.~\ref{fig-bilevel02} (b)(Right) zooms in to $p \ge 600$. The classification error for the case with the largest $\tau$ is too large to be shown. We can see that the interpolating estimator $\etals$ performs better than the regularized estimators when $p$ is sufficiently large.

{
%%%%%%%%%%%%%%%%%%%%%%%%%%%%%%%%%%%%%%%%%%%%%%%%
%%%%%%%%%%%%%%%%%%%%%%%%%%%%%%%%%%%%%%%%%%%%%%%%
\section{\new{Noisy GMM: Interpolation and Benign Overfitting}}
%%%%%%%%%%%%%%%%%%%%%%%%%%%%%%%%%%%%%%%%%%%%%%%%
%%%%%%%%%%%%%%%%%%%%%%%%%%%%%%%%%%%%%%%%%%%%%%%%
\label{sec-labelnoise}
We extend our results to a probabilistic label-noise Gaussian mixture model.

\subsection{Model and assumptions}
We formally define the noisy model below; note that this is a special case of the adversarial noise model studied in \citet{chatterji2020finite}.

\begin{definition}[Noisy GMM]
A data pair $(\x,y_c)\in\R^p\times\{\pm1\}$ is generated from the noisy Gaussian mixture model (GMM) with mean vector $\etab$, covariance matrix $\Sigmab$ and corruption probability $\gamma$ as follows. First, the  clean data pair $(\x,y)$ is generated according to \eqref{eq-GM}. Then the label $y_c$ is generated by  flipping the correct label $y$ with probability $\gamma$. We assume that $\gamma$ is independent of everything else (i.e., independent of the label $y$ and the Gaussian noise term $\q$). Also, we assume that $0 \le \gamma \le 1/C$ for a large constant $C$.
\end{definition}

% \color{blue}
% In this section, we analyze the case where label noise exists. The label noise is modeled as follows: 

% Given a small probability $\gamma$ such that $0 \le \gamma \le 1/C$ for a large constant $C$, the potentially wrong label $y_c$ is obtained by flipping the correct label $y$ with probability $\gamma$. Here $\gamma$ is independent of other random component, e.g., $y$ and $\boldq$.

 We define the label vector with clean/corrupted labels as $\y$/$\tildy$. For brevity, we focus here on the isotropic case $\Sigmab = \boldsymbol{I}$ %, and note that extensions to the balanced/bi-level ensembles follow likewise.
 and we derive analogues of Theorems \ref{thm-linkiso}, \ref{thm-eqvariso} and of Corollary \ref{cor-linkandtester}. Throughout this section, we let $\etals$/$\etasvm$ be the LS and SVM solutions obtained by solving minimizations \eqref{eq-minsolution} and \eqref{eq-svmsolution} but with the unobserved clean label vector $\boldy$ substituted by the observed corrupted vector $\tildy$.
 
%  \subsection{Main results}
%  We now present our main results for the isotropic noisy GMM. Specifically, 

 \subsection{Interpolation}
 Our first result establishes the equivalence between SVM and LS solutions for high enough effective overparameterization for noisy GMM data. As we will see, the required overparameterization conditions are now stronger compared to the noiseless case. %However, 
 
%  \citet[Lemma 4.3]{chatterji2020finite} gives that the classification error $\mathcal{R}(\tildeta) \le \gamma + \exp(-c\frac{(\tildeta^T\etab)^2}{\Vert \tildeta \Vert_2^2})$, where $\tildeta$ is the estimate of $\etab$ trained with label noise. Specifically, the LS and SVM estimators are obtained in the same way as \eqref{eq-minsolution} and \eqref{eq-svmsolution}, except for running the program with $\tildy$ instead of $\boldy$.

% Following the structure of the main text, we will first establish sufficient conditions that make the LS and SVM estimators equivalent, then derive an upper bound for the classification error for $\tildetals$ and finally obtain the benign overfitting conditions for $\tildetasvm$. The theorem below provides conditions under which all the samples become support vectors.
\begin{theorem}
\label{thm-linkisonoise}
% For sufficiently large number of training points n and large positive constant $C_i$'s,
Assume $n$ training samples following the noisy GMM with $\Sigmab=\mathbf{I}$. There exist large constants $C_i$'s $>1$ such that, if the following conditions on the number of features $p$ and the mean-vector $\boldsymbol{\eta}$ hold:
\begin{align}
p > C_1n\log n + n -1 \quad\text{and} \quad
p > C_2\max\{n\sqrt{\log(2n)}\normeta, n\normeta^2\}, \label{eq-link02noise}
\end{align}
then, the SVM-solution $\tildetasvm$ satisfies the linear interpolation constraint with probability at least $(1 - \frac{C_3}{n})$.
\end{theorem}
Note the extra term $n\normeta^2$ in the second condition above compared to Theorem \ref{thm-linkiso}. When $\rm{SNR}=\|\etab\|_2^2=\Omega(\log^{1/2}(n))$, this new condition becomes dominant and the overparameterization ratio $p/n$ should exceed $\rm{SNR}$ to guarantee interpolation. In Corollary \ref{cor-eqvariso}, we called the regime $p/n\geq \rm{SNR}$ the low-noise regime. Hence, in the noisy case, we can guarantee equivalence of the SVM and LS solutions only in the low-SNR regime. %\ct{There is currently no explanation here.}

%The reason why this term appears will become more clear after we show the classification error bound below. 

\subsection{Error bounds}
Our next result upper bounds the risk of the LS estimator. The bound holds in a regime where $\etals=\etasvm$, so it also applies to the risk of the SVM solution.

\begin{theorem}
\label{thm-eqvarisonoise}
Assume that conditions in \eqref{eq-link02noise} hold for noisy GMM data with $\Sigmab = \boldsymbol{I}$. Fix $\delta\in(0,1)$ and suppose $n$ is large enough such that $n>c/\delta$ for some $c>1$.
Then, there exist constants $C,b>1$ such that with probability at least $1-\delta$, $\tildetals^T\boldsymbol{\eta} > 0$ provided that
$
    p>b\cdot n$ and $(1-\frac{n}{p})\Vert \boldsymbol{\eta}\Vert_2 > C.
$
Further assume that these two conditions hold for $C,b>1$.  
% Further let $p > bn$ for a large constant $b$ and condition \eqref{eq-posicorr02} is satisfied. 
Then, there exist constants $C_1,C_2>1$ such that with probability at least $1-\delta$:
\begin{equation}
\label{eq-thmisonoise}
    \mathcal{R}(\tildetals) \le \gamma + \exp \Big(-C_2\cdot\Vert \boldsymbol{\eta} \Vert_2 ^4\frac{((1-\frac{n}{p}) - C_1  \frac{1}{\Vert \boldsymbol{\eta} \Vert}_2)^2}{p/n }\Big).
\end{equation}
\end{theorem}
Since the conditions in \eqref{eq-link02noise} hold, we operate here again in the low-SNR regime. The bound has two additive terms. The first term is the noise-level $\gamma$ which we cannot beat due to the corruptions. {The exponential term  is the same as the bound for noiseless GMM in the low-SNR regime presented in Corollary \ref{cor-eqvariso}.}

\begin{remark}[Comparison of risk bounds to \citet{chatterji2020finite}]
\label{rem:chat}
%As we mentioned, the noisy GMM is a special case of the more general adversarial model studied by
For an adversarial noise model and subGaussian features \citet{chatterji2020finite} prove that
\begin{equation}
\label{eq-thmisonoise-chatterji}
    \mathcal{R}(\tildetasvm) \le \gamma + \exp \Big(-C\frac{\Vert \boldsymbol{\eta} \Vert_2 ^4}{p}\Big),
\end{equation}
in the following regime:
\begin{equation}\label{eq:ass-cha}
    p>C\max\left\{ n^2\log(n)\,,\,n\|\etab\|_2^2 \right\},\quad \|\etab\|_2^2\geq C\log(n) \quad\text{and}\quad n\geq C.
\end{equation}
While our model is a special case of theirs, note that Theorem \ref{thm-eqvarisonoise} holds under relaxed assumptions. Specifically, we relax \eqref{eq:ass-cha} to
 \begin{equation}\label{eq:ass-ours}
    p>C\max\left\{ n\log(n)\,,\,n\|\etab\|_2^2 \right\},\quad \|\etab\|_2^2\geq C \quad\text{and}\quad n\geq C.
\end{equation}
Also, assuming the special case \eqref{eq:ass-cha} of \citet{chatterji2020finite} our bound in Theorem \ref{thm-eqvarisonoise} reduces to the one in \eqref{eq-thmisonoise-chatterji}.
\end{remark}

% defined in Corollary \ref{cor-eqvariso} and the corresponding bound can be applied here. The reason why we need $p > Cn\normeta^2$, i.e., the norm of the mean vector should not be very large, is because too much overparameterization will fit the wrong labels and thus lead to a poor generalization. \kw{I am not quite sure if we have math expressions can illustrate this or not, in later proofs, the reason why we have this bound is because we want some terms to be small, but it could be hard to explain these terms intuitively.}

\subsection{Benign overfitting}
Paralleling the exposition in Section \ref{sec-benign}, we use the results above to show that both the SVM and LS solutions approach the Bayes error as overparameterization increases. The requirements for this to happen are now stronger. However, the conclusion is somewhat more surprising in the noisy case: interpolating solutions nearly achieve optimal Bayes error despite perfectly fitting to corrupted labels. Borrowing the terminology introduced by \citet{bartlett2020benign}, our result establishes ``benign overfitting" for noisy GMM data.

% Now we show the conditions for benign overfitting of $\tildetasvm$, i.e., the conditions that make $\mathcal{R}(\tildetals)$ approach $\gamma$ as $p \to \infty$.

\begin{corollary}
\label{cor-linkandtesternoise}
Let the same assumptions as in Theorem \ref{thm-eqvarisonoise} hold and $n$ sufficiently large such that $n > {C}/{\delta}$ for some $C > 1$. Then, for large enough positive constant $C_i$'s $> 1$, $\tildetasvm$ linearly interpolates the data and the classification error $\mathcal{R}(\tildetasvm)$ approaches $\gamma$ as $p/n \to \infty$ with probability at least ($1 - \delta$) provided the following sets of conditions on the number of features $p$ and mean-vector $\boldsymbol{\eta}$ hold:
\begin{align*}
    p > \max\{C_2n\log n +n-1, C_3n\sqrt{\log(2n)}\normeta, n\normeta^2\}, \ \text{and} \ \normeta^4 \ge C_4(\frac{p}{n})^{\alpha}, \ \text{for} \  \alpha >1.
\end{align*}
\end{corollary}

{Note that the benign overfitting condition above is identical to the condition of Corollary \ref{cor-linkandtester} for the Low-SNR regime in the noiseless case.} However, 
%In the noisy case, we cannot guarantee that interpolating solutions asymptotically achieve the Bayes risk in the High-SNR regime. Thus, 
instead of $\normeta = \Theta(p^\beta)$ with $\beta \in (\frac{1}{4}, 1)$ in the noiseless case, the conclusion of Corollary \ref{cor-linkandtesternoise} holds under the stronger condition $\normeta = \Theta(p^\beta)$ for $\beta \in (\frac{1}{4}, \frac{1}{2}]$. We remark that (according also to the discussion in Remark \ref{rem:chat}), our conditions for benign overfitting of noisy GMM coincide with the conditions derived by \citet{chatterji2020finite}.

}      %%% END BLUE  

% When $n$ is fixed and only $p$ and $\normeta$ grow large, then for the classification error to go to $\gamma$ as $p \to \infty$, our results above coincide with \citet[Theorem 3.1]{chatterji2020finite} (both require that $\normeta = \Theta(p^\beta)$ for $\beta \in (\frac{1}{4}, \frac{1}{2}]$). Compared to the results without label noise (Corollary \ref{cor-linkandtester}), which requires $\normeta = \Theta(p^\beta)$ for $\beta \in (\frac{1}{4}, 1)$, the result with label noise needs less overparameterization and as we explained before, the reason is that too much overparameterization will overfit the noise thus generalize poorly.

\section{Proofs outline}
\label{sec-pfoutline}

%The proofs of almost all the theorems are related, hence we discuss them here together. 
The complete proofs are given in the Appendix. Here, we provide an outline. \new{For simplicity, we focus on the noiseless GMM in \eqref{eq-GM}. At a high-level, the proofs for the noisy model remain the same with some more care needed to account for the mismatch between the clean and the corrupted labels (see Appendix \ref{sec-labelnoise-proofs} for details).}
%Before that, we remark on the main challenge of analysis. Even though  $\etals$ is given in closed form in terms of the data $(\boldsymbol{X}, \boldsymbol{y})$, appropriately bounding terms such as $\boldsymbol{y}^T (\boldsymbol{X}\boldsymbol{X}^T)^{-1} \boldsymbol{y}$ is a non-trivial task, because $\boldsymbol{X}$ ``includes" $\boldsymbol{y}$ (recall that $\boldsymbol{X} = \boldsymbol{y}\boldsymbol{\eta}^T + \boldsymbol{Q}$).
%At a high-level, our idea to resolve the challenge is  In order to address the connection between $\boldsymbol{X}$ and $\boldsymbol{y}$, we need to express $\boldsymbol{X}\boldsymbol{X}^T$ as the sum of several components, then apply the matrix inverse identity. Note that the rank of $\boldsymbol{X}\boldsymbol{X}^T$ is $n$ almost surely \cite{hsu2020proliferation}, hence $(\boldsymbol{X}\boldsymbol{X}^T)^{-1}$ exists almost surely.

\subsection{Reductions to quadratic forms} 
\label{subsec-pfoutline01}
We first show that the proofs of all theorems reduce to establishing lower/upper bounds on quadratic forms of the Gram matrix $(\xxplustau)^{-1}$. 

\noindent\textbf{Link between SVM solution and LS solution.} We start with Theorems \ref{thm-linkgen} and \ref{thm-linkiso}. As in  \citet[Theorem 1]{muthukumar2020classification}, it suffices to derive conditions under which the following complementary slackness condition of \eqref{eq-svmsolution} is satisfied with high probability:
\begin{equation}
\label{eq-pfdual}
    y_i\boldsymbol{e}_i^T (\boldsymbol{X}\boldsymbol{X}^{T})^{-1} \boldsymbol{y} > 0, \ \text{for all} \ i\in[n].
\end{equation}
Note that the LHS of \eqref{eq-pfdual} is a quadratic form involving $(\boldsymbol{X}\boldsymbol{X}^{T})^{-1}$. 

\noindent\textbf{Classification error.} When deriving upper bounds on the classification error, it suffices from Lemma \ref{lem-miserror} that we lower bound the ratio
\begin{align}
\label{eq-ratiolowerbound02}
    \frac{(\etareg^T\boldsymbol{\eta})^2}{\etareg^T \boldsymbol{\Sigma}\etareg} = \frac{(\boldsymbol{y}^T(\boldsymbol{X}\boldsymbol{X}^T + \tau\boldsymbol{I})^{-1}\boldsymbol{X} \boldsymbol{\eta})^2}{ \boldsymbol{y}^T(\boldsymbol{X}\boldsymbol{X}^T + \tau\boldsymbol{I})^{-1}\boldsymbol{X}\boldsymbol{\Sigma}\boldsymbol{X}^T (\boldsymbol{X}\boldsymbol{X}^T + \tau\boldsymbol{I})^{-1}\boldsymbol{y}}.
\end{align}
Specifically, when $\tau = 0$ and $\Sigmab = \boldsymbol{I}$, we have
\begin{align}
\label{eq-ratiolowerboundiso}
    \frac{(\etals^T\boldsymbol{\eta})^2}{\etals^T \boldsymbol{\Sigma}\etals} = \frac{(\boldsymbol{y}^T(\boldsymbol{X}\boldsymbol{X}^T)^{-1}\boldsymbol{X} \boldsymbol{\eta})^2}{ \boldsymbol{y}^T(\boldsymbol{X}\boldsymbol{X}^T)^{-1}\boldsymbol{y}}.
\end{align}
Note that both the numerator and denominator above include terms such as $\boldy^T(\xxplustau)^{-1}$ and $\boldy^T(\boldsymbol{X}\boldsymbol{X}^T)^{-1}\boldy$.
%\noindent\textbf{Classification error for isotropic ensemble.} For the isotropic ensemble, we need to lower bound the ratio
%\begin{align}
%\label{eq-ratiolowerboundiso}
%    \frac{(\etals^T\boldsymbol{\eta})^2}{\etals^T \boldsymbol{\Sigma}\etals} = \frac{(\boldsymbol{y}^T(\boldsymbol{X}\boldsymbol{X}^T)^{-1}\boldsymbol{X} \boldsymbol{\eta})^2}{ \boldsymbol{y}^T(\boldsymbol{X}\boldsymbol{X}^T)^{-1}\boldsymbol{y}}
%\end{align}
%The numerator and denominator above includes the term $\boldy^T(\boldsymbol{X}\boldsymbol{X}^T)^{-1}$.
%\noindent\textbf{Classification error for bi-level ensemble.} For the model described in Assumption \ref{ass-onesparse}, we have
%\begin{align}
%    \frac{(\etareg^T\boldsymbol{\eta})^2}{\etareg^T \boldsymbol{\Sigma}\etareg} = \frac{(\eta_k\hat{\eta}_k)^2}{\sum_{i=1}^p\lambda_i\hat{\eta}_i^2}.
%\end{align}
%It is shown in the Appendix that
%\begin{align}
%    \hat{\eta}_i &= \sqrt{\lambda_i}\boldsymbol{z}_i^T(\xxplustau)^{-1}\boldy, \ \ \ \text{for} \ \ i \ne k,\\
%    \hat{\eta}_k &= \eta_k\boldy^T(\xxplustau)^{-1}\boldy+\sqrt{\lambda_k}\boldsymbol{z}_k^T(\xxplustau)^{-1}\boldy,
%\end{align}
%where $\boldsymbol{z}_i$ is the $i$-th column of a $n \times p$ matrix with IID standard Gaussian elements. Again, $\hat{\eta}_i$'s can be expressed in terms of quadratic forms involving the Gram matrix $(\xxplustau)^{-1}$.
Our key technical contribution is bounding those for GMM data. % involving $(\xxplustau)^{-1}$, such as $\boldsymbol{y}^T(\xxplustau)^{-1}\boldsymbol{y}$. 

\noindent\textbf{Challenge.} 
Bounding quadratic forms of $(\xxplustau)^{-1}$ is challenging for GMM data, since $\boldsymbol{X}\boldsymbol{X}^T = (\boldsymbol{y}\boldsymbol{\eta}^T + \boldsymbol{Q})(\boldsymbol{y}\boldsymbol{\eta}^T + \boldsymbol{Q})^T$, i.e. the Gram matrix ``includes" both $\boldsymbol{y}$ and $\etab$. Specifically, this is different to \citet{muthukumar2020classification,bartlett2020benign,tsigler2020benign}, since in their setting  $\boldsymbol{X}\boldsymbol{X}^T=\boldsymbol{Q}\boldsymbol{Q}^T$ and their results on quadratic forms of inverse Wishart matrices do not directly apply here.% or $\boldsymbol{X}\boldsymbol{X}^T=\boldsymbol{Q}\boldsymbol{Q}^T + \tau \boldsymbol{I}$.

% Recall from \eqref{eq-GM} that $\boldsymbol{X}\boldsymbol{X}^T = (\boldsymbol{y}\boldsymbol{\eta}^T + \boldsymbol{Q})(\boldsymbol{y}\boldsymbol{\eta}^T + \boldsymbol{Q})^T$. At this point, the analysis differs from  from \cite{muthukumar2020classification}, since in their setting  $\boldsymbol{X}\boldsymbol{X}^T=\boldsymbol{Q}\boldsymbol{Q}^T$. The additional term $\boldsymbol{y}\boldsymbol{\eta}^T$ makes the analysis more involved and, as we will see, it gives rise to  the second condition $p > C_2 n \Vert \boldsymbol{\eta} \Vert_2$ in \eqref{eq-}. Note from \eqref{eq-pfdual} that we need  a lower bound for $ y_i\boldsymbol{e}_i^T (\boldsymbol{X}\boldsymbol{X}^{T})^{-1} \boldsymbol{y}$, which is a quadratic form involving $(\boldsymbol{X}\boldsymbol{X}^{T})^{-1}$. 

\noindent\textbf{Our approach.}~For concreteness, consider the problem of bounding the quadratic form $T_1:=\boldsymbol{y}^T(\xxplustau)^{-1}\boldsymbol{y}$. \new{A possible approach is to start from bounds on the eigenvalues of $\boldsymbol{X}\boldsymbol{X}^T + \tau\boldsymbol{I}$ and  then obtain bounds for the eigenvalues of its inverse. Specifically, this turned out to be appropriate in the setting of  \citet{bartlett2020benign,muthukumar2020classification}.
%Recall $\boldsymbol{X}\boldsymbol{X}^T = (\boldsymbol{y}\boldsymbol{\eta}^T + \boldsymbol{Q})(\boldsymbol{y}\boldsymbol{\eta}^T + \boldsymbol{Q})^T$. 
%In this spirit, \citet{bartlett2020benign,muthukumar2020classification}, who studied related problems, provide bounds for the eigenvalues of $\boldsymbol{Q}\boldsymbol{Q}^T$, and \citet{tsigler2020benign} provides bounds for the eigenvalues of $\boldsymbol{Q}\boldsymbol{Q}^T + \tau \boldsymbol{I}$, which they then use to bound quadratic forms involving $(\boldsymbol{Q}\boldsymbol{Q}^T)^{-1}$ and $(\boldsymbol{Q}\boldsymbol{Q}^T + \tau \boldsymbol{I})^{-1}$. 
The situation is different here: the same eigenvalue approach fails to capture the dependence of $\X$ on $\y$ when bounding  $T_1$ and results in suboptimal bounds.
%and since $\boldsymbol{X}$ involves $\boldsymbol{y}$, attempting to bound $\boldsymbol{y}^T(\boldsymbol{X}\boldsymbol{X}^T + \tau \boldsymbol{I})^{-1}\boldsymbol{y}$ using the eigenvalue approach as suggested above is suboptimal since it ignores that dependence. 
Instead of decoupling $\y$ and the inverse Gram matrix that appear in $T_1$, we consider both terms simultaneously. To make this possible we begin with the following decomposition of the Gram matrix:
\begin{align*}
    \boldsymbol{X}\boldsymbol{X}^T + \tau\boldsymbol{I}
    %&=  \boldsymbol{Z}\boldsymbol{\Lambda}\boldsymbol{Z} +\normeta^2\boldsymbol{y}\boldsymbol{y}^T + \boldsymbol{Q}\boldsymbol{\eta}\boldsymbol{y}^T + \Big(\boldsymbol{Q}\boldsymbol{\eta}\boldsymbol{y}^T\Big)^{T}\\
    &=(\boldsymbol{Q}\boldsymbol{Q}^T + \tau \boldsymbol{I}) + \begin{bmatrix}\normeta\boldsymbol{y}& \boldsymbol{Q}\boldsymbol{\eta} & \boldsymbol{y}\end{bmatrix} \begin{bmatrix}
    \normeta\boldsymbol{y}^T\\
    \boldsymbol{y}^T\\
    (\boldsymbol{Q}\boldsymbol{\eta})^T
    \end{bmatrix},
\end{align*} 
which already isolates the (translated) Wishart matrix $(\boldsymbol{Q}\boldsymbol{Q}^T + \tau \boldsymbol{I})$  from the terms $\etab$ and $\y$. Once decomposed in this form, our  observation is that with an appropriate application of the matrix inversion lemma we can now express quadratic forms of interest (such as $T_1$) in terms of five more primitive quadratic forms. This idea is materialized in the following key lemma.}
\begin{lemma}
\label{pfot-lem-xinverse}
Let $\boldutau := \boldsymbol{Q}\boldsymbol{Q}^T + \tau \boldsymbol{I}$ (thus, $\boldu_{0} = \boldsymbol{Q}\boldsymbol{Q}^T$) and $\boldsymbol{d} := \boldsymbol{Q}\boldsymbol{\eta}$. Further define the following five primitive quadratic forms
\begin{align}
\label{pfot-quadratic}
s &:=\boldsymbol{y}^T \boldutau^{-1}\boldsymbol{y}, \ t :=\boldsymbol{d}^T \boldutau^{-1}\boldsymbol{d}, \ h :=\boldsymbol{y}^T \boldutau^{-1}\boldsymbol{d}, \  g_i :=\boldsymbol{y}^T\boldu_{0}^{-1}\boldsymbol{e}_i, \ f_i :=\boldsymbol{d}^T\boldu_{0}^{-1}\boldsymbol{e}_i, \ i\in[n],
\end{align}
and denote $D := s(\normeta^2 - t) + (h+1)^2$. With this notation, the following identity is true:
\begin{align}
\label{pfot-eq-xinverse}
    \boldsymbol{y}^T(\boldsymbol{X}\boldsymbol{X}^T+\tau\boldsymbol{I})^{-1} = \boldsymbol{y}^T\boldutau^{-1} - \frac{1}{D}\Big[\normeta s, h^2+h-st, s\Big]\begin{bmatrix}
    \normeta\boldsymbol{y}^T\\ 
    \boldsymbol{y}^T\\
    \boldsymbol{d}^T
    \end{bmatrix} \boldutau^{-1}.
\end{align}
\end{lemma}

%  together work on the inverse directly, which results in a decomposition 

% . Specifically our idea is to 
% Then, we apply the Woodbury matrix inverse identity to the matrix form above and we will see that the problem reduces to bounding quadratic forms involving the inverse Gram matrix $(\boldsymbol{Q}\boldsymbol{Q}^T + \tau \boldsymbol{I})^{-1}$. Let $\boldutau := \boldsymbol{Q}\boldsymbol{Q}^T + \tau \boldsymbol{I}$ and $\boldsymbol{d} := \boldsymbol{Q}\boldsymbol{\eta}$, thus $\boldu_{0} = \boldsymbol{Q}\boldsymbol{Q}^T$. We need to define the following quadratic forms involving $\boldutau^{-1}$:
% \begin{align}
% \label{pfot-quadratic}
% s &:=\boldsymbol{y}^T \boldutau^{-1}\boldsymbol{y}, \ t :=\boldsymbol{d}^T \boldutau^{-1}\boldsymbol{d}, \ h :=\boldsymbol{y}^T \boldutau^{-1}\boldsymbol{d}, \  g_i :=\boldsymbol{y}^T\boldu_{0}^{-1}\boldsymbol{e}_i, \ f_i :=\boldsymbol{d}^T\boldu_{0}^{-1}\boldsymbol{e}_i, \ i\in[n].
% \end{align}
% The following lemma expresses $\boldsymbol{y}^T(\boldsymbol{X}\boldsymbol{X}^T + \tau\boldsymbol{I})^{-1}$ in terms of those quadratic forms.
% \begin{lemma}
% \label{pfot-lem-xinverse}
% Define $D := s(\normeta^2 - t) + (h+1)^2$, then
% \begin{align}
% \label{pfot-eq-xinverse}
%     \boldsymbol{y}^T(\boldsymbol{X}\boldsymbol{X}^T+\tau\boldsymbol{I})^{-1} = \boldsymbol{y}^T\boldutau^{-1} - \frac{1}{D}\Big[\normeta s, h^2+h-st, s\Big]\begin{bmatrix}
%     \normeta\boldsymbol{y}^T\\ 
%     \boldsymbol{y}^T\\
%     \boldsymbol{d}^T
%     \end{bmatrix} \boldutau^{-1}.
% \end{align}
% \end{lemma}
\new{The five quadratic forms defined in \eqref{pfot-quadratic} involve now the inverse of the Wishart matrix $\Q\Q^T$ rather than of the original Gram matrix $\X\X^T$; this is why we call them ``primitive". Despite that feature, bounding these terms still does \emph{not} follow by mere application of results appearing in previous works \citep{bartlett2020benign,tsigler2020benign,muthukumar2020classification}. Moreover, observe in identity \eqref{pfot-eq-xinverse} that the five primitive forms appear with mixed signs each and both in the numerator/denominator. Thus, it is critical to obtain both lower and upper bounds for them. We derive these in the two lemmas below, which together with Lemma \ref{pfot-lem-xinverse} form key technical contributions of our work.}

% Now we proceed by bounding the key terms in \eqref{pfot-quadratic} using bounds on the eigenvalues of the inverse Gram matrix $\boldutau^{-1}$. \kw{Note that to bound $t$ and $h$, we need to bound $\Vert \boldd \Vert_2^2$ first, where we use Bernstein's inequality \citep{vershynin2018high} to bound the sum of sub-exponential variables (see SM \ref{secpf-boundforth}). }For $s, t, h$ and $\Vert \boldd \Vert_2$, we have:

%this To bound the quadratic forms involving the inverse Gram matrix $\boldutau^{-1}$, we need the bounds on the eigenvalues of $\boldutau^{-1}$ and the following lemma adapted from \cite[Lemma 5 (3)]{bartlett2020benign} shows these results:
%\begin{lemma}
%\label{pfot-lm-eigU}
%Assume the balanced $\Sigmab$ assumption \eqref{def-balance} is satisfied. Suppose that $\delta < 1$ with $\log(1/\delta) < n/c$ for some $c > 1$. There is a constant $C > 1$ such that with probability at least $1-\delta$, the largest and smallest eigenvalues of $\boldutau$ satisfy:
%\begin{equation}
%\label{pfot-eq-pfeigu}
%    \frac{1}{C}(\tau+\sum_{i=1}^p\lambda_i)\le\tau+\frac{1}{C}\sum_{i=1}^p\lambda_i \le \lambda_n(\boldutau) \le \lambda_1(\boldutau) \le \tau + C\sum_{i=1}^p\lambda_i \le C(\tau + \sum_{i=1}^p\lambda_i).
%\end{equation}
%\end{lemma}
%Now we are ready to derive the upper/lower bounds for quadratic forms involving $\boldutau^{-1}$.
\begin{lemma}[Balanced]
\label{pfot-lem-ineqforu}
Recall that $\sigma^2 = \suminner$. Assume that $\Sigmab$ follows the balanced ensemble defined in Definition \ref{def-balancedef}. Fix $\delta\in(0,1)$ and suppose $n$ is large enough such that $n>c\log(1/\delta)$ for some $c>1$. Then, there exists constants $C_1, C_2, C_3, C_6>1$, $C_5 > C_4 > 0$ such that with probability at least $1-\delta$, the following results hold:
\begin{align*}
    &\frac{n}{C_1 (\tau+\normlbd)} \le s \le C_1\frac{n}{(\tau+\normlbd)},~~
     C_4\frac{n\sigma^2}{(\tau+\normlbd)} \le t \le  C_5\frac{n\sigma^2}{(\tau+\normlbd)} ,\\ 
     - &C_2 \frac{n\sigma}{(\tau+\normlbd)} \le h  \le C_2 \frac{n\sigma}{(\tau+\normlbd)},~~
    \Vert \boldsymbol{d} \Vert_2^2  \le C_3n\sigma^2,~~
    \Vert \boldy^T\boldutau^{-1}\Vert_2 \le C_6\frac{\sqrt{n}}{(\tau+\normlbd)}.
   % \Vert \boldd^T\boldutau^{-1}\Vert_2 \le C_7\frac{\sqrt{n}\sigma}{(\tau+\normlbd)}.
%    - C_3 \frac{\sqrt{n}}{p} \Vert \boldsymbol{\eta} \Vert_2  \le & f_i  \le   C_3 \frac{\sqrt{n}}{p} \Vert \boldsymbol{\eta} \Vert_2, \ \text{for} \ i \in [n].
\end{align*}
% \new{We need a lower bound for $y_ig_i$. By the proof of \citet[Theorem 1]{muthukumar2020classification}, if \eqref{eq-linkgen01} is satisfied, then with probability at least $(1-\frac{C}{n})$,
% \begin{align*}
%     %\label{eq-pfyimain}
%      y_i g_i \ge \frac{1}{2\normlbd},~~\text{for all}~~i \in [n].
% \end{align*}
% We can have a sharper condition than \eqref{eq-linkgen01} for the isotropic case (see Lemma \ref{lm-linkiso03} for more detail.)}
\end{lemma}

%Next, \kw{we expand $\boldsymbol{e}_i^T(\boldcapx\boldcapx^T)^{-1}\boldy$ in \eqref{eq-pfdual} and as we will see in \eqref{eq-pfgammamain}, }we need a positive lower bound on $y_ig_i$ for $i \in [n]$. Similar to \citet[Proof of Theorem 1]{muthukumar2020classification}, we first write $\boldu_{0}^{-1} = \frac{1}{\normlbd}\boldsymbol{I} - \boldsymbol{E}^{'}$ and then show $y_i\boldsymbol{y}^T\boldsymbol{E}^{'}\boldsymbol{e}_i$ is sufficiently small for all $i \in [n]$. The complete derivation is more intricate and is deferred to SM \ref{pf-linktwo}. 

We state our finding on $ f_i, i \in [n]$ separately since it requires extra technical work to yield a bound that is uniform over $[n]$ and dimension independent. {See Appendix \ref{pfsec-pfboundforf} for details.}
%Bounding $f_i$ for $i \in [n]$ also needs some extra work because we need a ``uniform bound" over $[n]$ without dimensionality dependence. To do that, we consider $\max_{i \in [n]}|f_i|$. Then we again write $\boldu_{0}^{-1} = \frac{1}{\normlbd}\boldsymbol{I} - \boldsymbol{E}^{'}$ and bound $\Vert \boldd \Vert_{\infty}$. The result below summarizes our bound:

\begin{lemma}
\label{pfot-lm-boundforf}
Assume the condition in \eqref{eq-linkgen01} is satisfied, Fix $\delta\in(0,1)$ and suppose large enough $n>c/\delta, c>1$. There exists constant $C >1$ such that with probability at least $1-\delta$,
\begin{align}
    \label{pfot-eq-boundmaxf}
    \max_{i \in [n]}|f_i| \le \frac{C\sqrt{\log(2n)}\sigma}{\normlbd}.
\end{align}
\end{lemma}

\subsection{Proof sketch of Theorems \ref{thm-linkgen} and \ref{thm-linkiso}}\label{sec-proofoutline02}
With the technical lemmas above, we are now ready to sketch the proof of Theorem \ref{thm-linkgen}. For simplicity here,  consider the unregularized estimator ($\tau = 0$). As mentioned previously, it suffices to derive conditions under which \eqref{eq-pfdual} holds with high probability. Thanks to our Lemma \ref{pfot-lem-xinverse}, we derive the following decomposition in terms of the primitive terms defined in \eqref{pfot-quadratic} (with $\tau = 0$ therein):
\begin{align}
\label{eq-pfgammamain}
     \boldsymbol{y}^T(\boldsymbol{X}\boldsymbol{X}^{T})^{-1} \boldsymbol{e}_i = \frac{g_i+hg_i-sf_i}{s(\normeta^2-t)+(h+1)^2}.
\end{align}
%$s, h, t, g_i$ and $f_i$ are as defined in \eqref{pfot-quadratic} with $\tau = 0$ therein. 
The denominator above is positive with high probability. Thus, we only need to ensure that $y_i(g_i+hg_i-sf_i)>0$. For this, we use Lemmas \ref{pfot-lem-ineqforu} and \ref{pfot-lm-boundforf} {(see also \eqref{eq-pfyi} for a lower bound on $y_ig_i$)}. Detailed proof is in Appendix \ref{pf-seclinkgen}.

{The proof of Theorem \ref{thm-linkiso} is similar, except that the bounds on quadratic forms of the Wishart matrix are used when $\Sigmab = \boldsymbol{I}$, thus providing a sharper result. Specifically, when lower bounding $y_ig_i$, less overparameterization is needed, i.e., the first condition in \eqref{eq-link02} is sharper than \eqref{eq-linkgen02}.}

\subsection{Proof sketch of Theorems \ref{thm-eqvar01} and \ref{thm-eqvariso}} As per Section 
%We first consider Theorem \ref{thm-eqvar01}. From section
\ref{subsec-pfoutline01}, we will lower bound the ratio in \eqref{eq-ratiolowerbound02}.
%\begin{equation}
%\label{pf-ratio01}
%\frac{(\boldsymbol{y}^T(\boldsymbol{X}\boldsymbol{X}^T + \tau\boldsymbol{I})^{-1}\boldsymbol{X} \boldsymbol{\eta})^2}{ \boldsymbol{y}^T(\boldsymbol{X}\boldsymbol{X}^T + \tau\boldsymbol{I})^{-1}\boldsymbol{X}\boldsymbol{\Sigma}\boldsymbol{X}^T (\boldsymbol{X}\boldsymbol{X}^T + \tau\boldsymbol{I})^{-1}\boldsymbol{y}}.
%\end{equation}
%We will upper bound the denominator and lower bound the numerator. 
First, work with the denominator. Observe that
$
   \boldsymbol{X}\boldsymbol{\Sigma}\boldsymbol{X}^T =  (\boldy\betab^T+\boldsymbol{Z}\boldsymbol{\Lambda}^{\frac{1}{2}})\boldsymbol{\Lambda}(\boldy\betab^T+\boldsymbol{Z}\boldsymbol{\Lambda}^{\frac{1}{2}})^T.
$
Further let $\boldsymbol{A} := (\xxplustau)^{-1}\boldy\boldy^T(\xxplustau)^{-1}$ and  $\boldsymbol{z}_i$ denote the $i$-th column of  $\boldsymbol{Z}$. Then, we show the following by applying the cyclic property of trace and the inequality $\boldsymbol{v}^T\boldsymbol{M}\boldsymbol{u} \le \frac{1}{2}(\boldsymbol{v}^T\boldsymbol{M}\boldsymbol{v}+\boldsymbol{u}^T\boldsymbol{M}\boldsymbol{u})$, true for any PSD matrix $\boldsymbol{M}$: 
\begin{align*}
    &\text{Tr}\Big(\boldsymbol{y}^T(\boldsymbol{X}\boldsymbol{X}^T + \tau\boldsymbol{I})^{-1}\boldsymbol{X}\boldsymbol{\Sigma}\boldsymbol{X}^T (\boldsymbol{X}\boldsymbol{X}^T + \tau\boldsymbol{I})^{-1}\boldsymbol{y}\Big) =\text{Tr}\Big((\boldy\betab^T+\boldsymbol{Z}\boldsymbol{\Lambda}^{\frac{1}{2}})\boldsymbol{\Lambda}(\boldy\betab^T+\boldsymbol{Z}\boldsymbol{\Lambda}^{\frac{1}{2}})^T\boldsymbol{A}\Big)\\
    & \le 2\Big(\sum_{i=1}^p\lambda_i^2\Vert \boldsymbol{A}\Vert_2\Vert \boldsymbol{z}_i\Vert_2^2 + \sigma^2(\boldy^T(\boldsymbol{X}\boldsymbol{X}^T + \tau\boldsymbol{I})^{-1}\boldy)^2\Big),
\end{align*}

%Thus, we need to upper bound $\sum_{i=1}^p\lambda_i^2\Vert \boldsymbol{A}\Vert_2\Vert \boldsymbol{z}_i\Vert_2^2$ and $\sigma^2(\boldy^T(\boldsymbol{X}\boldsymbol{X}^T + \tau\boldsymbol{I})^{-1}\boldy)^2$. 

Now, to upper bound $\sum_{i=1}^p\lambda_i^2\Vert \boldsymbol{A}\Vert_2\Vert \boldsymbol{z}_i\Vert_2^2$, note $\Vert \boldsymbol{z}_i \Vert_2^2$'s are independent sub-exponentials; thus, for fixed $B>0$, we can bound $\sum_{i=1}^p\lambda_i^2B\Vert \boldsymbol{z}_i\Vert_2^2$ using the Bernstein's inequality.
%\citep[Theorem 2.8.2]{vershynin2018high}. %with the weights given by $B\lambda_i^2$ in blocks of size $n$ \cite[Lemma 7 and Corollary 1]{bartlett2020benign}. By Lemma \ref{pf-lemmasubexp}, with probability at least $1-2e^{-x}$,
%\begin{align*}
%    \sum_{i=1}^p\lambda_i^2B\Vert \boldsymbol{z}_i\Vert_2^2 & \le Bn\sumlbdsq + Bc\max\Big(\lambda_1^2x, \sqrt{xn\sum_{i}^p\lambda_i^4}\Big)\\
%    & \le Bn\sumlbdsq + Bc\max\Big(x\sumlbdsq, \sqrt{xn}{\sum_{i}^p\lambda_i^2}\Big)\\
%    &\le CnB\sum_{i=1}^p\lambda_i^2,
%\end{align*}
%for $x < n/c_0$. 
Specifically, we choose $B$ as an upper bound on $\Vert \boldsymbol{A} \Vert_2 = \Vert \boldy^T(\xxplustau)^{-1} \Vert_2^2$, which we obtain thanks to Lemma \ref{pfot-lem-ineqforu} after the following decomposition as per Lemma \ref{pfot-lem-xinverse}:
%\begin{align}
%\label{pfot-eq-testerror01}
     $\boldsymbol{y}^T(\boldsymbol{X}\boldsymbol{X}^T+\tau\boldsymbol{I})^{-1} =\big((1+h)\boldy^T\boldutau^{-1} - s\boldd^T\boldutau^{-1}\big)/D.$
%\end{align}
%Thus $\Vert \boldsymbol{A} \Vert_2$ can be upper bounded by Lemma \ref{pfot-lem-ineqforu}. %with probability at least $1-\delta$,
%\begin{align*}
%    \Vert \boldy^T(\xxplustau)^{-1} \Vert_2 &\le \frac{1}{D}\bigg(\Big(1+|h|\Big)\Vert \boldy \Vert_2\Vert \boldutau^{-1} \Vert_2 + s\Vert \boldd \Vert_2\Vert \boldutau^{-1} \Vert_2\bigg)\\
%    &\le \frac{1}{D}\bigg(\Big(1+\frac{C_1n\sigma}{(\tau+\normlbd)}\Big)\frac{C_2\sqrt{n}}{(\tau+\normlbd)} + \frac{C_3n\sqrt{n}\sigma}{(\tau+\normlbd)^2}\bigg).
%\end{align*}
%
Similarly, we can upper bound $\sigma^2(\boldy^T(\boldsymbol{X}\boldsymbol{X}^T + \tau\boldsymbol{I})^{-1}\boldy)^2$ since again by Lemma \ref{pfot-lem-xinverse} 
%\begin{align}
%\label{pfot-eq-testerror02}
$\boldsymbol{y}^T(\xxplustau)^{-1}\boldsymbol{y}={s}/{D}.$ %This gives the desired bound on the denominator.
%\end{align}
%Therefore,
%\begin{align*}
%    \sigma^2\boldy^T\boldsymbol{A}\boldy \le \frac{C}{D^2}\frac{n^2\sigma^2}{(\tau+\normlbd)^2}.
%\end{align*}
%Hence the denominator of \eqref{pf-ratio01} is upper bounded by
%\begin{align}
%    \label{pf-balancedensum}
%    \frac{1}{D^2}\frac{n^2}{(\tau+\normlbd)^2}\Big(C_1\max\{1,\frac{n^2\sigma^2}{(\tau+\normlbd)^2}\}\sumlbdsq + C_2\sigma^2\Big)
%\end{align}
%To upper bound the denominator of \eqref{eq-ratiolowerbound02}, we need the upper bounds of \eqref{pfot-eq-testerror01} and \eqref{pfot-eq-testerror02}, which can be obtained by using Lemma \ref{pfot-lem-ineqforu}. 

Next, focus on the numerator in \eqref{eq-ratiolowerbound02}. Thanks to Lemma \ref{pfot-lem-xinverse}, we have the decomposition
\begin{align}
\label{pfot-eq-testerror03}
\boldsymbol{y}^T(\xxplustau)^{-1}\boldsymbol{X} \boldsymbol{\eta} &=\frac{s(\normeta^2 -t)+h^2+h}{D},
\end{align}
and the desired bound is obtained by a careful application of Lemma \ref{pfot-lem-ineqforu} that bounds the primitive quadratic appearing above. {See Appendix \ref{pfsec-pftesterrortwo} for details and proof steps for Theorems \ref{thm-eqvar01} and \ref{thm-eqvariso}}.

%We need the lower bound of \eqref{pfot-eq-testerror03} and this can also be obtained from Lemma \ref{pfot-lem-ineqforu}.

% When deriving Theorem \ref{thm-eqvariso}, we need to lower bound \eqref{eq-ratiolowerboundiso}, and Lemma \ref{pfot-lem-xinverse} gives,
% \begin{equation*}
% %    \label{eq-pfnumde}
%     \frac{(\boldsymbol{y}^T(\boldsymbol{X}\boldsymbol{X}^T)^{-1}\boldsymbol{X} \boldsymbol{\eta})^2}{ \boldsymbol{y}^T(\boldsymbol{X}\boldsymbol{X}^T)^{-1}\boldsymbol{y}} = \frac{\Big(s(\normeta^2 -t) +h^2 + h\Big)^2}{s \Big(s(\normeta^2 -t) + (h+1)^2 \Big)}.
% \end{equation*}
% Again, we can obtain the desired bounds by appropriately using Lemma \ref{pfot-lem-ineqforu}.

%Combining the above gives
%The numerator needs to be lower bounded and Lemma \ref{lem-ineqforu} gives with probability at least $1-\delta$,
%\begin{align}
%    \label{eq-pfcorr01}
%    s(\normeta^2 -t) +h^2 + h &\ge s(\normeta^2 -t) + h \\
%    \label{eq-pfcorr02}
%    &\ge \frac{n}{C(\tau + \normlbd)}\Big(\normeta^2 - \frac{C_1n\sigma^2}{\tau + \normlbd} - C_2\sigma\Big).
%\end{align}
%Combining \eqref{pf-balancedensum} and \eqref{eq-pfcorr02} gives with probability at least $1-\delta$, \eqref{pf-ratio01} is lower bounded by
%\begin{align}
%\label{eq-pfthm01}
%   \frac{\Big(\normeta^2 - \frac{C_1n\sigma^2}{\tau + \normlbd} - C_2\sigma\Big)^2}{C_3\max\{1,\frac{n^2\sigma^2}{(\tau+\normlbd)^2}\}\sumlbdsq + C_4\sigma^2}.
%\end{align}
%This completes the proof of the theorem.

\subsection{Proof sketch of Theorem \ref{thm-bilevel01}}
We need to lower bound the ratio
%\begin{align}
%\label{pfot-ratiobilevel}
  $ \frac{(\etareg^T\boldsymbol{\eta})^2}{\etareg^T \boldsymbol{\Sigma}\etareg} = \frac{(\eta_k\hat{\eta}_k)^2}{\sum_{i=1}^p\lambda_i\hat{\eta}_i^2}.$
%\end{align}
To do this, we divide $\hat{\eta}_i$'s into 3 groups ( $\hat{\eta}_1$, $\hat{\eta}_k$ and the rest) and upper bound the following:
\begin{align*}
%\label{pf-uppthree}
    \frac{\lambda_1\hat{\eta}_1^2}{(\eta_k\hat{\eta}_k)^2}, \ \ \frac{\sum_{i\ne1,k}\lambda_i\hat{\eta}_i^2}{(\eta_k\hat{\eta}_k)^2} \ \ \text{and} \ \ \frac{\lambda_k\hat{\eta}_k^2}{(\eta_k\hat{\eta}_k)^2},
\end{align*}
%then reverse the sum of the upper bounds of the three ratios above to obtain a lower bound of \eqref{pfot-ratiobilevel}. 
where note from $\hat{\eta}_i = \boldsymbol{e}_i^T\hat{\etab}$ that 
    $\hat{\eta}_i = \sqrt{\lambda_i}\boldsymbol{z}_i^T(\xxplustau)^{-1}\boldy, \ \text{for} \ i \ne k,$ and  $\hat{\eta}_k = (\eta_k\boldy^T+\sqrt{\lambda_k}\boldsymbol{z}_k^T)(\xxplustau)^{-1}\boldy. 
    $
As before, thanks to Lemma \ref{pfot-lem-xinverse} this  reduces to upper/lower bounding quadratic forms involving $\boldutau^{-1} = (\boldsymbol{Q}\boldsymbol{Q}^T + \tau \boldsymbol{I})^{-1}$. However, because here $\lambda_1$ is much larger than other eigenvalues of $\Sigmab$, instead of directly bounding the eigenvalues of $\boldutau$, we leverage the leave-one-out trick introduced in \citet{bartlett2020benign} and first separate $\lambda_1$ from the other eigenvalues. Specifically, by Woodbury's identity, $\boldutau^{-1}$ is expressed as
\begin{align*}
    \boldutau^{-1} &= (\tau\boldsymbol{I}+\sum_{i=2}^p\lambda_i\boldz_i\boldz_i^T + \lambda_1\boldz_1\boldz_1^T)^{-1}
    = \boldutaunoone^{-1} - \frac{\lambda_1\boldutaunoone^{-1}\boldz_1\boldz_1^T\boldutaunoone^{-1}}{1+\lambda_1\boldz_1^T\boldutaunoone^{-1}\boldz_1},%\label{pf-invunoone}
\end{align*}
where $\boldutaunoone = \tau\boldsymbol{I} + \sum_{i=2}^p\lambda_i\boldz_i\boldz_i^T$. Now, we first bound the eigenvalues of $\boldutaunoone$, and then use these results to bound the eigenvalues of $\boldutau$ and $\boldutau^{-1}$. See Appendix \ref{pfsec-pfutaunoone} for details.

%After some algebra steps we find that
%\begin{align}
% \label{eq-pfden}
%   \boldsymbol{y}^T(\boldsymbol{X}\boldsymbol{X}^T+\tau\boldsymbol{I})^{-1}\boldsymbol{y}  
%   = s\big/\left({s(\normeta^2 -t) + (h+1)^2}\right)
%\end{align}
%where $\boldutau := \boldsymbol{Q}\boldsymbol{Q}^T+\tau\boldsymbol{I}$, $\boldsymbol{d} := \boldsymbol{Q}\boldsymbol{\eta}$, $s:=\boldsymbol{y}^T \boldutau^{-1}\boldsymbol{y}$, $t:=\boldsymbol{d}^T \boldutau^{-1}\boldsymbol{d}$ and $h:=\boldsymbol{y}^T \boldutau^{-1}\boldsymbol{d}$. The remaining of the proof involves lower and upper bounding the terms $s,t$ and $h$. Note that this is a simpler task as it involves bounding eigenvalues of the inverse of $\boldsymbol{Q}\boldsymbol{Q}^T + \tau\boldsymbol{I}$. We accomplish this leveraging the techniques of \cite{bartlett2020benign,tsigler2020benign}. For the balanced case, the relative small conditional number of $\boldsymbol{Q}\boldsymbol{Q}^T$ enables us to consider all the eigenvalues together. For bi-level ensemble, however, due to the spiky structure of eigenvalues of $\Sigmab$, considering all eigenvalues together can lead to suboptimal bounds, therefore, we use the leave-one-out trick introduced in \cite{bartlett2020benign} and separate $\lambda_1$ with others when deriving results.

\section{Discussion}\label{sec-discussion}
\new{Here, we include further details on how our results fit in the related literature. }

\subsection{Comparison to classical margin-based bounds}
\new{We start by arguing that classical bounds on the generalization of SVM are unimformative in the highly overparameterized settings of GMM data that we focus on. We do this by quantitatively comparing our results with classical margin-based bounds applied to GMM data. }

\new{First, consider the following well-known bound.
% \ct{How our results are new compared to classical max-margin bounds? They don't predict error goes to zero in highly-overparameterized regime.}
%  \ct{Merge and talk a little bit again about what we mean by benign overfitting.}
%  \ct{Acknowledge that Vidya's paper does this comparison numerically for their signed model.}
\begin{proposition}\label{propo:classic}\emph{\citep[Theorem 26.13]{shalev2014understanding}}. 
Consider a distribution $\mathcal{D}$ over $\mathcal{X} \times \{\pm 1\}$ such that there exists some vector $\etab^{*}$ with $\mathbf{P}_{(\boldx, y)\sim \mathcal{D}}(y\cdot{\boldetastar}^T\boldx \ge 1) = 1$ and such that $\Vert \boldx \Vert_2 \le R$ with probability 1. Let $\etasvm$ be the SVM solution. Then with probability at least $1-\delta$, we have that
    \begin{align}
    \label{eq-svmineq}
        \mathcal{R}(\etasvm) \le \frac{2R\Vert \boldetastar \Vert_2}{\sqrt{n}} + (1+R\Vert \boldetastar \Vert_2)\sqrt{\frac{2\log (2/\delta)}{n}}.
    \end{align}
\end{proposition}
We apply Proposition \ref{propo:classic} to the setting studied in Corollary \ref{cor-linkandtester}. Specifically, we will apply the bound for $\boldetastar\leftarrow\etab$. But, first we need to show that this choice satisfies the conditions of the proposition.}
% We now evaluate the bound of \eqref{eq-svmineq} above in the same regime. In particular, we apply the bound for $\boldetastar\leftarrow\etab$. First, we show that bounds assumptions are indeed verified for this choice. 
    % Here, if we assume the true mean vector $\etab$ can be used as $\boldwstar$, we need $\mathbf{P}(y\cdot\etab^T\boldx \ge 1) =1$ (or $\mathbf{P}(y\cdot\etab^T\boldx \le 1) =0$). 
{To this end, by definition of $\boldsymbol{x}$, we have $y\cdot\etab^T\boldsymbol{x} = \normeta^2 + \etab^T(y\boldsymbol{q})$ with $\q\sim\mathcal{N}(0,\mathbf{I}_p)$. Therefore,
    \begin{align*}
        \mathbf{P}(y\cdot\etab^T\boldx \le 1) &= \mathbf{P}(\etab^T(y\boldq) \le 1 - \normeta^2) = \mathbf{P}(\etab^T(y\boldq) \ge  \normeta^2 - 1) \\
        &\le \exp\left(-\frac{(\normeta^2 -1)^2}{2\normeta^2}\right)\le \exp\left(-\frac{\normeta^2}{2}+1\right)\\
        &\le \exp\left(-C(p/n)^{2\alpha}\right)\stackrel{p/n\rightarrow\infty}{\longrightarrow} 0.
    \end{align*}
The inequalities in the second line used Hoeffding's tail bound. In the third line, we used the conditions of Corollary \ref{cor-linkandtester} that 
    % Specifically let $p/n\rightarrow\infty$, %$p>C_1n\sqrt{log(n)}\|\etab\|_2$, and, 
    $\|\etab\|_2\geq C_2(p/n)^{\alpha}$ for some $\alpha>1/4$. Now, we compute the upper bound $R$. Bernstein's inequality gives with probability at least $1 - 2 e^{-p/c}$, $$ \normeta^2 + (1-(1/C))p \le \Vert \boldx \Vert_2^2 \le \normeta^2+ (1+(1/C))p.$$ Thus, in our setting with probability 1,
    $\Vert \boldx \Vert_2 \leq \normeta + C\sqrt{p}=:R$. Plugging this in \eqref{eq-svmineq} we see that $$R\|\boldetastar\|_2/\sqrt{n}=\Theta\left(\|\etab\|_2^2/\sqrt{n} + \sqrt{p/n}\|\etab\|_2\right).$$ This bound becomes vacuous in the setting of Corollary \ref{cor-linkandtester}. Indeed, by using $\|\etab\|_2\geq C_2(p/n)^{\alpha}$, we find that $\sqrt{p/n}\|\etab\|_2\rightarrow\infty$ as $p/n\rightarrow\infty$.}
    %Thus, Proposition \ref{propo:classic} is not able to explain why SVM nearly achieves Bayes optimal (aka zero) error in the highly overparameterized regime of Corollary \ref{cor-linkandtester}.
    
\new{One might wonder if the conclusion would be different have we instead used a margin-based bound. We show that such bounds are also \emph{not} able to explain why SVM nearly achieves Bayes optimal (aka zero) error in the highly overparameterized regime of Corollary \ref{cor-linkandtester}.
 \begin{proposition}\label{propo:classic_margin}\emph{\citep[Theorem 26.14]{shalev2014understanding}}. 
Assume the conditions of Proposition \ref{propo:classic}. Then, with probability at least $1-\delta$, we have that
    \begin{align}
    \label{eq:classic_margin}
        \mathcal{R}(\etasvm) \le \frac{4R\Vert \etasvm \Vert_2}{\sqrt{n}} + \sqrt{\frac{\log\left(4\log_2(\Vert \etasvm \Vert_2) / \delta\right)}{n}} .
    \end{align}
\end{proposition}   }
\new{In order to analytically evaluate the bound above, we need a means to control the inverse margin $\|\etasvm\|_2$. While it is not a-priori clear how to do this, our analysis establishes an upper bound on $\|\etasvm\|_2$ in the sufficiently overparameterized regime.  Specifically, we do this as part of the proof of Theorem \ref{thm-eqvar01} in the process of upper bounding the correlation of the LS solution in Section \ref{pfsec-erroriso} (see Equation \ref{pf-01}). But in the setting of Corollary \ref{cor-linkandtester} $\|\etasvm\|_2=\|\etals\|_2$. Thus, \eqref{eqpf-errorisodeno} and \eqref{eq-pfcorr03} show that 
% $
% \|\etasvm\|_2\leq C\|\etab\|_2\frac{n}{p} + C \sqrt{\frac{n}{p}}
% $.
$
\|\etasvm\|_2^2\leq \frac{C}{(1-n/p)\|\etab\|_2^2}.
$
Recalling from above that $R=\|\etab\|_2+C\sqrt{p}$ and putting things together proves that
% \begin{align}\label{eq:margin1}
% \frac{R\|\etasvm\|_2}{\sqrt{n}} = O\left( \|\eta\|_2^2\frac{\sqrt{n}}{p}  + \|\eta\|_2\sqrt{\frac{n}{p}} +   1 \right).
% \end{align}    
\begin{align}\label{eq:margin1}
\frac{R\|\etasvm\|_2}{\sqrt{n}} = O\left( \frac{1}{\sqrt{n}} + \sqrt{\frac{p}{n\|\etab\|_2^2}} \right).
\end{align}    
In the High-SNR regime of Corollary \ref{cor-linkandtester} recall that $p<n\|\etab\|_2^2/C$, thus the value in \eqref{eq:margin1} is $O(1+1/\sqrt{n})$. We see that (at least in the High-SNR regime) the bound we obtained by combining Proposition \ref{propo:classic_margin} with our upper bound of $\|\etasvm\|_2$ is indeed improved compared to that of Proposition \ref{propo:classic}. However, it still fails to predict the fact that the error goes to zero with increasing overparameterization (as predicted by Proposition \ref{propo:classic_margin}). The bound is similarly inconclusive about the Low-SNR regime.}

% Consider for example the Low-SNR regime in Corollary \ref{5.2}. Then, since $\|\etab\|_2^2\leq p/n$, the value in \eqref{eq:margin1} is $O(1+1/\sqrt{n})$. We see that the bound we obtained by combining Proposition \ref{propo:classic_margin} with our upper bound of $\|\etasvm\|_2$ is indeed improved compared to that of Proposition \ref{propo:classic}. However, it still fails to predict the fact that the error goes to zero with $p\rightarrow\infty$ (as predicted by Proposition \ref{propo:classic_margin}). 

\new{We end this section by noting that the fact that margin-based bounds are loose in the overparameterized regime has been previously also discussed in \cite{montanari2019generalization,deng2019model} and \cite{muthukumar2020classification}. Specifically, \cite{montanari2019generalization,deng2019model} showed that Proposition \ref{propo:classic_margin} fails to predict the exact double-descent behavior of the risk in linear models even if the inverse margin $\|\etasvm\|_2$ in \eqref{eq:classic_margin}  is evaluated using sharp asymptotic formulas. Here, we have used our non-asymptotic bound for $\|\etasvm\|_2$ and showed that a margin-based argument is insufficient to yield the conclusions on Section \ref{sec-benign}. Finally, see also the discussion in \cite[Sec. 6]{muthukumar2020classification} where the authors demonstrate the deficiency of margin-based explanations in classification of signed data via numerical simulations. Here, we have arrived at the same conclusion, this time for GMM data, via an analytic study.}

% \newadd{\citet{muthukumar2020classification} demonstrates that the margin-based generalization bounds are uninformative in 
% overparameterized settings via experiments. }We verify the uninformativeness of the classical margin-based bound by comparing our results with \citet[Theorem 26.13]{shalev2014understanding}: 

%     Our goal is to make the probability above approach $0$. \kw{I am not sure if we want $p \to \infty$ or $n \to \infty$ here.} For $\Vert \boldx \Vert_2 \le R$, Bernstein's inequality gives with probability at least $1 - 2 e^{-p/c}$, $ \normeta^2 + (1-(1/C))p \le \Vert \boldx \Vert_2^2 \le \normeta^2+ (1+(1/C))p$. I think to make the probability in \eqref{eq-svmineq} to be zero, we need $n \to \infty$ since $\sqrt{n}$ is on the denominator.

\subsection{Comparison to previous works}
\label{sec-compare}
\new{We have already discussed how our results are motivated and how they differ from previous works in the Introduction. Here, we focus on the three most closely related papers \citep{bartlett2020benign,muthukumar2020classification,chatterji2020finite} and provide a more detailed discussion. % with reference to our results presented in the sections above.
}

\subsubsection{\citet{bartlett2020benign}}
\label{sec-comparebartlett}
\new{As mentioned in the Introduction \citet{bartlett2020benign} is amongst the first to analytically study generalization principles under overparameterization. Our work is inspired by them, but otherwise differs in four important aspects as outlined next.}
\\
\new{\indent (i) First, unlike linear regression, we study a linear classification model in which labels $y$ are binary and covariates are of the form $\x=y\etab+\q$. As discussed in Section \ref{sec-model} this implies that $y=\x^T\bar\etab + z$ with $\bar\etab:=\etab/\|\etab\|_2^2$ and $z=\q^T\bar\etab$. While this latter formulation resembles the linear regression model, where noise is additive, note here that the additive ``noise" term $z$ is highly signal dependent. The analysis of \citet{bartlett2020benign} makes heavy use of the assumption that noise is signal independent, hence their techniques \emph{cannot} be directly applied to the GMM (see why in point (iii) below).}
\\
{\indent (ii) Second, our model is also different in that the feature vectors, although still Gaussian, are now signal dependent. Again, this does not allow a direct application of the technical results in \citep{bartlett2020benign} in our setting. Specifically, \citet{bartlett2020benign} show that in their setting  bounding generalization can be mapped to a question about controlling the rate of decay of eigenvalues of inverse Wishart matrices. Instead, as explained in Section \ref{subsec-pfoutline01}, in our setting we first express the generalization metric of interest as a non-trivial function of a number of simpler quadratic forms. While these quadratic forms involve inverse Wishart matrices, their statistics are not solely governed by the eigenstructure of the latter, but they also involve the mean vector $\etab$.}
\\
\new{\indent (iii) Third, beyond the model itself what differs fundamentally in classification is the measure of generalization performance. Instead of the squared prediction risk studied by \citet{bartlett2020benign}, relevant for us is the expected error as measured by the 0/1 loss. {For Gaussian covariates, the former essentially reduces to the mean-squared error and the authors show that it suffices controlling a quantity $\boldsymbol{\epsilon}^T\mathbf{C}\boldsymbol{\epsilon}$,  where $\mathbf{C}=(\boldcapx\boldcapx^T)^{-1}\boldcapx\Sigmab\boldcapx^T(\boldcapx\boldcapx^T)^{-1}$ and $\boldsymbol{\epsilon}$ is the additive noise in the linear regression model \citep[Lemma 7]{bartlett2020benign}. To do this, they exploit the assumption that $\boldsymbol{\epsilon}$ is independent of $\boldcapx$ and sub-Gaussian, which reduces the problem to upper bounding  $\text{Tr}(\mathbf{C})$ \citep[Lemma 8]{bartlett2020benign}. Their subsequent analysis is tailored to this term. Instead, Lemma \ref{lem-miserror} shows that controlling the 0/1 risk requires bounding the estimator's correlation. For the latter, we show that one needs to \emph{both} upper bound $\boldy^T\mathbf{C}\boldy$ and lower bound $\boldsymbol{y}^T(\boldsymbol{X}\boldsymbol{X}^T)^{-1}\boldsymbol{X} \boldsymbol{\eta}$ (see \eqref{eq-ratiolowerbound02}). Our goal is now more complicated compared to the situation faced in linear regression because: (a) In the first term $\y$ is not random (unlike $\boldsymbol{\epsilon}$). (b) The second quadratic form involves a matrix other than $\mathbf{C}$ and both vectors $\mathbf{y}$ and $\etab$. (c) The feature matrix $\X$ is a non-centered Gaussian matrix whose mean involves both the response $\mathbf{y}$ and the mean vector $\etab$.  } }
\\
%
%Again, this fact imposes new challenges compared to \citep{bartlett2020benign}. \ct{Can we use some more precise but still sort of high level argument here as to why bounding 0/1 loss or even correlation is harder than mse? To bound correlation you need to control ratio of $\hat\eta^T\eta$ over $\|\hat\eta\|_2$. For mse, suffices to control their difference $\|\hat\eta\|_2^2-2\hat\eta^T\eta$. Is this what they do? Why, if so, is the former harder?}
% \kw{To see the challenge, \citet[Lemma 7]{bartlett2020benign} indicates that the heart of their proof is to control $\text{Tr}((\boldcapx\boldcapx^T)^{-1}\boldcapx\Sigmab\boldcapx^T(\boldcapx\boldcapx^T)^{-1})$. 
% In our work, we need to upper bound $\boldy^T(\boldcapx\boldcapx^T)^{-1}\boldcapx\Sigmab\boldcapx^T(\boldcapx\boldcapx^T)^{-1}\boldy$ and lower bound $\boldsymbol{y}^T(\boldsymbol{X}\boldsymbol{X}^T)^{-1}\boldsymbol{X} \boldsymbol{\eta}$, which needs extra work due to its more complicated structure.}
%\ct{Very nice. But reads as if something is still missing. Why is this different or not implied by bounding the trace? say the first one? }
\new{\indent(iv) The fourth difference is that in our setting, we are interested in the generalization performance of the SVM solution rather than the minimum-norm interpolator. The challenge is that the former is \emph{not} given in closed form in terms of the label vector $\y$ and the feature matrix $\X$. The key innovation to circumvent this challenge is attributed to \citet{muthukumar2020classification} who realized that under sufficient overparameterization SVM becomes equivalent to LS. We remark though that identifying the appropriate conditions for this to happen for GMM data is key contribution of our work (see Section \ref{sec-comparevidya}).
}
\newadd{Following the above discussion emphasizing differences to the setting of \citep{bartlett2020benign} it should not be surprising the our error bounds in Section \ref{sec-testerror} are of different nature to those in \citep{bartlett2020benign}. The first key difference is that our bounds involve not only the eigenstructure of the covariance matrix, but also the mean vector of the GMM. Second, as a natural follow up, our conditions in Section \ref{sec-benign} for which the classifier's error approaches the Bayes error are different to those in \citep{bartlett2020benign}. Despite the differences, it might be interesting to the reader noting that the two ensembles introduced in Definitions \ref{def-balance} and \ref{def-bilevel} can be expressed in terms of the notions of ``effective ranks" defined by \citet{bartlett2020benign}, i.e. $r_k := {(\sum_{i>k}^p \lambda_i)}/{\lambda_{k+1}}$.% and $R_k = (\sum_{i>k}^p\lambda_i)^2/(\sum_{i>k}^p\lambda_i^2)$.
To see the relationship, let $\tilde{r}_k := {(\sum_{i>k+1}^p \lambda_i)}/{\lambda_{k+1}}=r_k -1$. With this notation, in the balanced ensemble, $\tilde{r}_0 \ge bn$, which directly implies $r_0 \ge bn$. For large enough $n$, the reverse direction of implication is also true. In the bi-level ensemble, the first condition $\tilde{r}_0 \le bn$ implies again $r_0 \le b^{'}n$ for large enough $n$. Similarly, the second condition $\tilde{r}_1 \ge b_1 n$ implies $r_1 \ge b_1n$. %Despite the similarities, we repeat here that (as shown by our results in Sections \ref{sec-link})  generalization behavior of linear regression (governed entirely by the eigenvalue structure of the feature covariance matrix)  differs from that of classification with GMM data, in which both the data covariance and the mean vector $\etab$ interact.
}%  which is in turn governed the interplay between the data covariance and the mean $\etab$ is more intricate and plays important role for the GMM; see Sections \ref{sec-link} and \ref{sec-benign}.}

\subsubsection{\citet{muthukumar2020classification}}
\label{sec-comparevidya}
\new{The paper by \cite{muthukumar2020classification} is the most closely related to this work in terms of the approach that we follow. We complement the discussion in the Introduction regarding the different setting between the two works with a more  detailed exposition of our key technical differences. For concreteness, we focus on the proof of equivalence between SVM and LS in Theorems \ref{thm-linkgen} and \ref{thm-linkiso}, since the same differences apply to the error analysis in Theorems \ref{thm-eqvar01}, \ref{thm-eqvariso} and \ref{thm-bilevel01}.}

\new{There are two main steps in proving Theorems \ref{thm-linkgen} and \ref{thm-linkiso}. The first step involves a deterministic sufficient condition guaranteeing that the constraints of the SVM  optimization in \eqref{eq-svmsolution} are active.  %\citet{muthukumar2020classification} show that  specifically 
%\begin{align}\label{eq:det}
%\boldy\odot (\boldsymbol{X}\boldsymbol{X}^T)^{-1}\boldsymbol{y}>0
%\end{align}
%(see Equation (7.1) in the paper). 
The second step involves a probabilistic analysis of this deterministic condition using the generative statistical model at hand. The first part of our proof is same as in \citet{muthukumar2020classification} and \citet{hsu2020proliferation}. Specifically, we use their deterministic condition \eqref{eq-pfdual}. On the other hand, the second part of our analysis is  technically challenging. 
%To see this, note that in the model considered by \citet{muthukumar2020classification} and \citet{hsu2020proliferation} the covariates are zero mean Gaussians. Henceforth, proving that \eqref{eq:det} holds is possible with \ct{correct} minor effort by leveraging either known properties of the inverse Wishart distribution (for isotropic data) or \kw{Lemma 1 in \citet{muthukumar2020classification}}. 
The reason is that unlike previous work where the covariates are zero mean Gaussians, in our case, $\boldsymbol{X}=\boldsymbol{Q}+\boldsymbol{y}\boldsymbol{\eta}^T$ for a zero-mean Gaussian matrix $\boldsymbol{Q}$. Note that the deterministic condition \eqref{eq-pfdual} to be checked involves the inverse Gram matrix. The key relevant technical argument in \citet{muthukumar2020classification} (i.e., Lemma 1 therein) controls how far the inverse Wishart matrix $(\mathbf{Q}\mathbf{Q}^T)^{-1}$ is from $\left(\sum_{i\in[p]}|\lambda_i|\right)\mathbf{I}_d$. This results is clearly not sufficient in our case as $(\X\X^T)^{-1}$ involves more terms. We repeat our strategy at circumventing this challenge as also sketched in Section \ref{subsec-pfoutline01}. We start by expanding the terms in $(\X\X^T)$ and recognizing that after appropriate application of the matrix inversion lemma together with some algebra we can express the LHS of \eqref{eq-pfdual} as a function of five quadratic forms of either of two random matrices,  $(\boldsymbol{Q}\boldsymbol{Q}^T)^{-1}$ or $\boldsymbol{Q}^T(\boldsymbol{Q}\boldsymbol{Q}^T)^{-1}$. It should be noted that this function involves the five quadratic forms in a convoluted way making it necessary to provide both upper and lower bounds for those forms (see Equation \eqref{eq-pfgammamain}). Besides lower bounding one of the first two terms that involves $(\boldsymbol{Q}\boldsymbol{Q}^T)^{-1}$ using Lemma 1 in \citet{muthukumar2020classification}, \emph{none} of the remaining quadratic forms appear in the analysis of \citet{muthukumar2020classification}. Lemmas \ref{lem-ineqforu} and \ref{lm-boundforf}, where we obtain lower/upper bounds for them, form a main technical contribution of our work (see Appendix \ref{sec-pflemma} for details). Finally, the delicate piece of putting together those bounds to guarantee a positive quantity overall is also new compared to previous works (see Appendix \ref{pf-linktwo}).}

\newadd{As we have highlighted in the previous sections, differences to \citep{muthukumar2020classification} are not only technical. Most importantly, the differences extend to the conclusions regarding the conditions playing a key role for interpolation of the SVM solution and for the classification error of SVM to approach the Bayes error. See discussions in Sections \ref{sec-link} and \ref{sec-testerror}. %Regarding the former, our Theorems \ref{thm-linkgen} and \ref{thm-linkiso} reveal a new interpolation conditions $\normlbd > Cn\sqrt{\log(2n)}\sigma$ (for anisotropic case) and $p > Cn\sqrt{\log(2n)}\normeta$ (for isotropic case) that are unique to the GMM. We have also provided empirical evidence suggesting that these new conditions might be tight (see Figure \ref{fig-linksvm}). 
%The same remark regarding the key role played by the mean vector applies also to our results in the classification analysis of Sections \ref{sec-testerror} and \ref{sec-benign}. 
As a side technical note here, we have here relaxed the one-sparse assumption \citep[Assumption 1]{muthukumar2020classification} on the parameter vector $\etab$ in the balanced ensemble. Finally, unlike \citet{muthukumar2020classification}, our bounds further apply to regularized LS and are extended to a model with label corruptions. % and  (see Section \ref{sec-reg}) are new compared to 
}
% \newadd{Also, regarding the classification analysis, compared with \citet{muthukumar2020classification}, we relax the one-sparse assumption of the parameter vector $\etab$ in the balanced ensemble. Besides, the regularized case and the discussion on the role of regularization (see Section \ref{sec-reg}) are new compared to \citet{muthukumar2020classification}. %In terms of error analysis, compared to Vidya, we relax the 1-sparse assumption of the parameter vecor $\eta$ when in blah ensemble.
% }

%\kw{The sentences below are from section 3, would be put it here or just leave it in section 3, because it explains our results.}\ct{This paragraph remains in Section 3 this is fine.}
%\newadd{For the anisotropic case, the first condition in Theorem \ref{thm-linkgen} is the same as \citet[Theorem 1]{muthukumar2020classification} and is related to the the effective ranks $R_0$ and $r_0$ defined in Section \ref{sec-comparebartlett}. It requires that the spectral structure in the covariance matrix $\Sigmab$ has sufficiently slowly decaying eigenvalues (corresponding to sufficiently large $R_0$), and that it is not too ``spiky" (corresponding to sufficiently large $r_0$). \citet[Remark 4]{muthukumar2020classification} provides a detailed discussion on how the effective ranks relate to different spectrum regimes.}

\newadd{As a last note, we discuss the nice follow-up to \citet{muthukumar2020classification} by  \citet{hsu2020proliferation}, which involves two key contributions. The first concerns conditions for interpolation. The first step in their analysis (aka the deterministic condition \eqref{eq-pfdual} discussed above) is the same as in \citet{muthukumar2020classification}, but \eqref{eq-pfdual} is eventually expressed in a different equivalent form that allows tightening the probabilistic analysis that follows in the case of anisotropic convariance. Their second novelty involves  relaxing the requirement for Gaussianity of the features to subGaussianity and Haar distribution. These improvements still only apply to the discriminative model, thus they are not directly applicable here.
}
%In their analysis Hsu et al. (2020) also adapt the first step above (the deterministic condition) from Vidya. Their innovation is in tightening the bounds obtained in the second part of the proof and relaxing the requirement for Gaussianity to subgaussianity and Haar. Yet, their analysis still assumes a discriminative model as in Vidya. Thus, their technical work cannot be adapted in our case (see discussion above).

%The focus of Hsu et al. is improving on step one of the proof above (the deterministic condition) making it tight.As we have noted in Lines BLAH, the analysis of that stricter condition for the GMM model seems to require new ideas on top of our techniques and this forms a potentially interesting reserach direction.} 

\subsubsection{\citet{chatterji2020finite}}\label{sec-comparechatterji}
\newadd{We now compare our work to \citet{chatterji2020finite}, who also derive non-asymptotic error bounds on the classification error of GMM data.} 

\new{First, there are certain differences in the problem setting. On the one hand, \citet{chatterji2020finite} relaxes the assumption on Gaussianity by studying the case where $\q$ in \eqref{eq-GM} has subGaussian entries \footnote{This is interesting as for example it includes a Boolean noisy version of the rare-weak model by \citet{jin2009impossibility},  for which our results do not directly apply.}. On the other hand, while we require that $\q$ is Gaussian, our results capture explicitly the role of the data covariance matrix and its interplay with the mean vector via the key parameter $\sigma^2=\etab^T\Sigmab\etab$. As we have seen in Sections \ref{sec-testerror}, \ref{sec-benign} and \ref{sec-reg}, the error behavior can differ substantially for different covariance structures (e.g., balanced vs bi-level ensembles). This phenomenon is \emph{not} revealed by  \citet[Thm.~3.1]{chatterji2020finite} \footnote{We note that the key role played by data covariance in double-decent and benign overfitting has been also emphasized in several related works, e.g., \cite{hastie2019surprises,bartlett2020benign,muthukumar2020classification,montanari2019generalization,chang2020provable}}.
Another distinguishing feature of the results in \citet{chatterji2020finite} is that they apply to a noisy model that allows for (bounded number of) adversarial label corruptions. Our main focus is the noiseless GMM, but we also extended our results to a special case of their model in Section \ref{sec-labelnoise}. 
}

\new{In terms of analysis, our techniques are very different. As mentioned we follow the high-level recipe of \citet{muthukumar2020classification} (also adapted by \citet{hsu2020proliferation}), that is first showing equivalence of SVM to LS and then deriving error bounds for the latter. Instead, \citet{chatterji2020finite} analyze the SVM solution by viewing it as the limit of gradient-descent updates on logistic loss minimization with sufficiently small step-size \citep{soudry2018implicit}.  Specifically, they produce a recursive argument that at each iteration lower bounds the the expected margin of the current gradient-descent iterate on a clean point with respect to the margin of the previous iterate \citep[Lem.~4.4]{chatterji2020finite}. We believe that both techniques are of interest. Via the connection to logistic loss minimization, their approach also yields insights on the degree to which one example (possibly a noisy one) can affect the quality of the learnt classifier \citep[Lem.~4.8]{chatterji2020finite}. It also allows the study of subGaussian features (rather than Gaussian) rather naturally. On the other hand, the approach followed here leads to Theorems \ref{thm-linkgen} and \ref{thm-linkiso} on equivalence of SVM to LS under sufficient effective overparameterization, which is a result of its own interest. Besides, as mentioned, our technique allows us to capture the effect of data covariance. 
}

\new{We already discussed in Sections \ref{sec-benign} and \ref{sec-labelnoise} how our findings compare to those in \citet{chatterji2020finite}. {In summary, for the noiseless case, we show that interpolating solutions asymptotically achieve the Bayes error under relaxed assumptions compared to the noisy model (see Remark \ref{rem:cha-noiseless}. For the noisy model, our benign-overfitting conditions are identical, but our risk bounds hold under relaxed assumptions (see Remark \ref{rem:chat}).}
%we improve their conditions for the noiseless case and recover their results in the probabilistic noise model. 
Finally, in addition to the risk bounds for SVM derived by \citet{chatterji2020finite}, we also derive conditions for which SVM solution interpolates the data and we investigate regularized LS.
}

%\new{In their work, they studied the isotropic subGaussian case and analyze the gradient descent steps directly by leveraging the line of work on implicit bias \citep{soudry2018implicit}. Specifically, they obtain the estimate of the mean vector by running gradient descent with the exponential loss and by \citet{soudry2018implicit}, the iterates will converge the the SVM solution. They derive bounds for $\hat{\etab}_{(t)}^T\etab$, where $(t)$ means the $t$-th iterate and unroll the bounds via induction and take limit to reach the final result.
%In contrast, we study the SVM solution by studying LS following the equivalence between SVM and LS. In Section \ref{sec-benign}, we show that we actually relax their condition for benign overfitting $\normeta = \Theta(p^{\beta})$, $\beta \in (1/4, 1/2]$ to $\beta \in (1/4, 1)$.}
%\ct{Different approach interesting. By following the approach here we also identify conditions for which SVM=LS which is interesting on its own right.}

\subsection{Contemporaneous and follow-up work}

\new{While the current version of our paper was undergoing review and after an earlier version of our paper \citep{wang2020benign}, we became aware of contemporaneous independent work by \citet{cao2021risk}. Compared to our setting, \citet{cao2021risk} only requires sub-Gaussian features. Similar to us their results capture the key role of the spectrum of the data covariance. Their proofs for the correlated case build on ideas developed in our earlier version \citep{wang2020benign} for the isotropic case. Compared to them, we also derived bounds for regularized LS in our paper. A more detailed technical comparison between the two paper is as follows. First, \citet{cao2021risk} obtains a sharper first condition $\Vert \lbdb \Vert_1 \ge \max\{n\sqrt{n}\Vert \lbdb \Vert_\infty, n\Vert \lbdb \Vert_2\}$ for equivalence of SVM to LS in Theorem \ref{thm-linkgen}, by invoking stronger concentration arguments. Their second condition is the same as Theorem \ref{thm-linkgen}. {For this, we further present insightful simulation results suggesting its tightness (see Figure \ref{fig-linksvm}).} Regarding the classification error, \citet{cao2021risk} provides both upper and the lower bound for $\mathcal{R}(\etasvm)$. However, note that their results only apply to the balanced ensemble. For the anisotropic balanced setting, compared to Theorem \ref{thm-linkgen}, \citet[Theorem 3.1]{cao2021risk} proved that $\mathcal{R}(\etals)\leq\exp\left(\frac{-C\Vert \etab \Vert_2^4}{\Vert \lbdb \Vert_\infty + (\normlbd^2/n) + \sigma^2)}\right)$. Under the same assumptions in Theorem \ref{thm-linkgen}, \citet[Theorem 3.1]{cao2021risk}, the numerator of our corresponding bound in \eqref{eq-testls} can be simplified to the same as the result in \citet{cao2021risk}. However, the denominators are slightly different, where instead of $\Vert \lbdb \Vert_2^2/n$  in \citet{cao2021risk}, we obtain $\Vert \lbdb \Vert_2^2$ and an additional $\Vert \lbdb \Vert_\infty$ term. For the isotropic setting, after some simplification, the bound on $  \mathcal{R}(\hat{\boldsymbol{\eta}}_{\rm LS})$ in \citet[Corollary 3.3]{cao2021risk} is the same as Theorem \ref{thm-eqvariso}. Therefore, the benign overfitting condition ($\normeta = \Theta(p^\beta), \beta\in(1/4,1)$) is matching for finite $n$ in the isotropic setting. As mentioned, we also investigate regularized LS in this paper. Additionally, in Section \ref{sec-labelnoise} we extend our results to a probabilistic label-noise model and derive conditions for benign overfitting that are not studied in \citet{cao2021risk}. }

\new{More recently, \citet{ardeshir2021support} derived lower bounds for the conditions required to make SVM and LS solutions equivalent for discriminative models. For unconditional Gaussian covariates they show a sharp phase-transition characterizing the equivalence phenomenon. It is interesting to extend their analyses focusing on lower bounds to GMM data as studied in our paper.
Finally, it is worth mentioning exciting related work \citet{zou2021benign,varre2021last} that  explores benign overfitting of \emph{stochastic} gradient descent (SGD) (instead, note in Section \ref{sec-model} that our motivation for studying SVM or the minimum-norm interpolator comes from imiplicit bias of GD rather than SGD).}

\section{Future work}\label{sec-future}
We established a connection between the SVM and the LS solutions in the overparameterized regime for GMMs. We then proved  a non-asymptotic bound for the classification error, \new{which we used to study generalization of SVM in the highly overparameterized regime.} We also discussed the role of regularization and illustrated a regime that interpolation estimators perform better than the regularized estimators. \new{We then show that our analysis and results (both on equivalence of SVM to LS and on risk bounds) extend naturally to a probabilistic label-noise model. For this model, we derive conditions for benign overfitting.}
%An immediate future work is extending our analysis to the anisotropic case $\boldsymbol{\Sigma}\neq\boldsymbol{I}$. 
We are interested in extending our analysis to adversarial corruptions and misspecified models. Another possible direction is deriving lower bounds to investigate whether our conditions in Theorems \ref{thm-linkgen} and \ref{thm-linkiso} are tight (as suggested by Figure \ref{fig-linksvm}). Possible extensions to more complex nonlinear settings are naturally very important. Finally, we are particularly interested in extensions to multiclass settings for which the GMM studied here serves as a natural model.
%Another interesting future direction is understanding how those ideas and the proof techniques could apply to nonlinear models such as neural networks.

\section*{Acknowledgments}
This work is partially supported by the NSF under Grant Number CCF-2009030 and by a research grant from KAUST. The authors would like to thank Dr. Vidya Muthukumar from the Georgia Institute of Technology for very helpful discussions and the anonymous reviewers for helpful suggestions that helped improve the presentation of our results.

%\clearpage
%\vfill\pagebreak

%\section{REFERENCES}
%\label{sec:refs}

%List and number all bibliographical references at the end of the
%paper. The references can be numbered in alphabetic order or in
%order of appearance in the document. When referring to them in
%the text, type the corresponding reference number in square
%brackets as shown at the end of this sentence \cite{C2}. An
%additional final page (the fifth page, in most cases) is
%allowed, but must contain only references to the prior
%literature. \cite{belkin2018overfitting}


%\onecolumn
%\bibliographystyle{IEEEbib}
% \bibliographystyle{siamplain}
%\bibliography{strings,refs}
\bibliography{ref}

%\clearpage

\appendix
%\section{Appendix}

\section{Key Technical Lemmas}
\label{sec-keylemma}
{For the reader's convenience, we repeat here some definitions and lemmas that were previously stated in Section \ref{sec-proofoutline02}}. Define $\boldutau := \boldsymbol{Q}\boldsymbol{Q}^T + \tau \boldsymbol{I}$ and $\boldsymbol{d} := \boldsymbol{Q}\boldsymbol{\eta}$; thus $\boldu_{0} = \boldsymbol{Q}\boldsymbol{Q}^T$. The lemma below expresses $\boldsymbol{y}^T(\boldsymbol{X}\boldsymbol{X}^T + \tau\boldsymbol{I})^{-1}$ in terms of the following quadratic forms:
\begin{align*}
s &=\boldsymbol{y}^T \boldutau^{-1}\boldsymbol{y},\\ 
t &=\boldsymbol{d}^T \boldutau^{-1}\boldsymbol{d},\\ h &=\boldsymbol{y}^T \boldutau^{-1}\boldsymbol{d},\\ g_i &=\boldsymbol{y}^T \boldu_{0}^{-1}\boldsymbol{e}_i,~~ i \in [n],\\  
f_i &=\boldsymbol{d}^T\boldu_{0}^{-1}\boldsymbol{e}_i, ~~ i \in [n].
\end{align*}
\begin{lemma}
\label{lem-xinverse}
Define $D := s(\normeta^2 - t) + (h+1)^2$, then
\begin{align}
\label{eq-xinverse}
    \boldsymbol{y}^T(\boldsymbol{X}\boldsymbol{X}^T+\tau\boldsymbol{I})^{-1} = \boldsymbol{y}^T\boldutau^{-1} - \frac{1}{D}\Big[\normeta s, h^2+h-st, s\Big]\begin{bmatrix}
    \normeta\boldsymbol{y}^T\\ 
    \boldsymbol{y}^T\\
    \boldsymbol{d}^T
    \end{bmatrix} \boldutau^{-1}.
\end{align}
\end{lemma}
The lemma below derives upper/lower bounds for those quadratic forms involving the inverse Gram matrix $\boldutau^{-1}$.
\begin{lemma}[Balanced]
\label{lem-ineqforu}
Recall that $\sigma^2 = \suminner$. Assume the $\Sigmab$ follows the balanced ensemble defined in Definition \ref{def-balancedef}. Fix $\delta\in(0,1)$ and suppose $n$ is large enough such that $n>c\log(1/\delta)$ for some $c>1$. Then, there exists constants $C_1, C_2, C_3, C_6, C_7 >1$, $C_5 > C_4 > 0$ such that with probability at least $1-\delta$, the following results hold:
\begin{align*}
    \frac{n}{C_1 (\tau+\normlbd)} \le & s \le C_1\frac{n}{(\tau+\normlbd)},\\
     C_4\frac{n\sigma^2}{(\tau+\normlbd)}  \le &t \le  C_5\frac{n\sigma^2}{(\tau+\normlbd)} ,\\
    - C_2 \frac{n\sigma}{(\tau+\normlbd)} \le & h  \le   C_2 \frac{n\sigma}{(\tau+\normlbd)},\\
    \Vert \boldsymbol{d} \Vert_2^2 & \le C_3n\sigma^2,\\
    \Vert \boldy^T\boldutau^{-1}\Vert_2 &\le C_6\frac{\sqrt{n}}{(\tau+\normlbd)},\\
    \Vert \boldd^T\boldutau^{-1}\Vert_2 &\le C_7\frac{\sqrt{n}\sigma}{(\tau+\normlbd)}.
%    - C_3 \frac{\sqrt{n}}{p} \Vert \boldsymbol{\eta} \Vert_2  \le & f_i  \le   C_3 \frac{\sqrt{n}}{p} \Vert \boldsymbol{\eta} \Vert_2, \ \text{for} \ i \in [n].
\end{align*}
\end{lemma}
To bound the term $f_i$, we need some additional work, which leads to the following result.
\begin{lemma}
\label{lm-boundforf}
Assume that the condition in \eqref{eq-linkgen01} is satisfied, Fix $\delta\in(0,1)$ and suppose $n$ is large enough such that $n>c/\delta$ for some $c>1$. Then, there exists a constant $C >1$ such that with probability at least $1-\delta$,
\begin{align}
    \label{eq-boundmaxf}
    \max_{i \in [n]}|f_i| \le \frac{C\sqrt{\log(2n)}\sigma}{\normlbd}.
\end{align}
\end{lemma}
The proofs of Lemmas \ref{lem-xinverse}, \ref{lem-ineqforu} and \ref{lm-boundforf} are given in Section \ref{sec-pflemma}. We will also need the following lemmas adapted from \citet[Proof of Theorem 1]{muthukumar2020classification}.
\begin{lemma}
\label{lm-linkiso02}
Let $\boldsymbol{E} = \boldsymbol{Q}\boldsymbol{Q}^T - \normlbd\cdot\boldsymbol{I}$ and $\boldsymbol{E}^{'} = \frac{1}{\normlbd}\cdot(\boldsymbol{Q}\boldsymbol{Q}^T)^{-1}\boldsymbol{E}$. Assume that the condition in \eqref{eq-linkgen01} is satisfied, then there exists a constant $C > 1$ such that with probability at least $(1 - \frac{C}{n})$,
\begin{align}
    \label{eq-linkerror}
    \Vert \boldsymbol{E}^{'} \Vert_2 \le \frac{1}{2\sqrt{n}\normlbd}.
\end{align}
\end{lemma}
\begin{lemma}
\label{lm-linkiso03}
Let $d^{'}(n) := (p-n+1)$. With probability at least $(1-\frac{2}{n^2})$,
\begin{equation*}
    y_ig_i = y_i(\boldsymbol{e}_i^T\boldsymbol{U}_0^{-1}\boldsymbol{y}) \ge \frac{1}{4\sqrt{n}} \frac{2\sqrt{n}d^{'}(n) - 2n \sqrt{4\log(n)d^{'}(n)} - 4n\log(n)}{(d^{'}(n)+\sqrt{4\log(n)d^{'}(n)})(d^{'}(n)-\sqrt{4\log(n)d^{'}(n)})}, \ \text{for} \ i \in [n].
\end{equation*}
\end{lemma}

\section{Proof of Theorem \ref{thm-linkgen} and Theorem \ref{thm-linkiso}}
\label{pf-linktwo}
\subsection{Proof of Theorem \ref{thm-linkgen}}
\label{pf-seclinkgen}
Now we are ready to prove Theorem \ref{thm-linkgen}. In this section, we only consider the unregularized estimator, i.e., $\tau = 0$. Define $\gamma^{*} := (\boldsymbol{X}\boldsymbol{X}^{T})^{-1} \boldsymbol{y}$. Using duality (see \citet[Appendix C.1]{muthukumar2020classification}), all the constraints in \eqref{eq-svmsolution} hold with equality provided that
\begin{equation}
\label{eq-pfdual02}
    y_i\gamma^*_i > 0, \ \text{for all} \ i \in [n].
\end{equation}
Hence it suffices to derive conditions under which \eqref{eq-pfdual02} holds with high probability. Note that
%\begin{align*}
     $\gamma^{*}_i = \boldsymbol{y}^T(\boldsymbol{X}\boldsymbol{X}^{T})^{-1} \boldsymbol{e}_i, \ \text{for all} \ i \in [n]$.
%\end{align*}
Using \eqref{eq-xinverse} and some algebra steps, it can be checked that:
\begin{align}
%\label{eq-pfgamma}
     \boldsymbol{y}^T(\boldsymbol{X}\boldsymbol{X}^{T})^{-1} \boldsymbol{e}_i & = \boldsymbol{y}^T\boldu_0^{-1}\boldsymbol{e}_i - \frac{1}{D}\begin{bmatrix}\normeta s & h^2+h-st & s\end{bmatrix}\begin{bmatrix}
    \normeta\boldsymbol{y}^T\boldu_0^{-1}\boldsymbol{e}_i  \\ 
    \boldsymbol{y}^T\boldu_0^{-1}\boldsymbol{e}_i\\
    \boldsymbol{d}^T\boldu_0^{-1}\boldsymbol{e}_i
    \end{bmatrix} \notag\\
     & = g_i - \frac{1}{D}\begin{bmatrix}\normeta s & h^2+h-st & s\end{bmatrix}\begin{bmatrix}
    \normeta g_i\\ 
    g_i\\
    f_i
    \end{bmatrix} \notag\\
    & = \frac{g_i(s(\normeta^2 - t) + (h+1)^2) - \normeta^2sg_i - (h^2+h-st)g_i - sf_i}{D} \notag\\
    \label{eq-pfgamma02}
     &= \frac{g_i+hg_i-sf_i}{s(\normeta^2-t)+(h+1)^2}.
\end{align}
Here, $s, h, t, g_i$ and $f_i$ are as defined in Section \ref{sec-keylemma} with $\tau = 0$. The denominator of \eqref{eq-pfgamma02} is non-negative, thus to make $\gamma_i > 0$, we only need to study the numerator:
\begin{align*}
    y_i(g_i+hg_i-sf_i) = (1+\boldsymbol{y}^T\boldsymbol{U}_0^{-1}\boldsymbol{d})y_i(\boldsymbol{e}_i^T\boldsymbol{U}_0^{-1}\boldsymbol{y}) - y_i(\boldsymbol{e}_i^T\boldsymbol{U}_0^{-1}\boldsymbol{d})\boldsymbol{y}^T\boldsymbol{U}_0^{-1}\boldsymbol{y}.
\end{align*}
First, consider the term $y_i(\boldsymbol{e}_i^T\boldsymbol{U}_0^{-1}\boldsymbol{y})$. By the proof of \citet[Theorem 1]{muthukumar2020classification}, if \eqref{eq-linkgen01} is satisfied, then with probability at least $(1-\frac{C}{n})$,
\begin{align}
    \label{eq-pfyi}
     y_i g_i \ge \frac{1}{2\normlbd}.
\end{align}
We know that $\boldu_0^{-1} = \frac{1}{\normlbd}\boldsymbol{I} - \boldsymbol{E}^{'}$. Thus for $\boldy^T\boldu_0^{-1}\boldsymbol{d}$, by Lemma \ref{lem-ineqforu} and \ref{lm-linkiso02}, with probability at least $(1-\frac{C}{n})$,
\begin{align*}
    \boldy^T\boldu_0^{-1}\boldsymbol{d} = \boldy^T(\frac{1}{\normlbd}\boldsymbol{I} - \boldsymbol{E}^{'})\boldsymbol{d} \ge -\frac{C_1n\sigma}{\normlbd} - \frac{C_2\sqrt{n}\sigma}{\normlbd} \ge -\frac{C_3n\sigma}{\normlbd},
\end{align*}
where the first inequality above follows from the fact $\boldsymbol{v}^T\boldsymbol{M}\boldsymbol{u} \ge -\Vert \boldsymbol{v} \Vert_2\Vert \boldsymbol{u} \Vert_2\Vert \boldsymbol{M} \Vert_2$.
Lemma \ref{lm-boundforf} gives for every $i \in [n]$, with the same high probability, 
\begin{align*}
    y_i\boldsymbol{e}_i^T\boldu_0^{-1}\boldsymbol{d} = y_if_i \ge - \max_{i \in [n]}|f_i| \ge -\frac{C_4\sqrt{\log(2n)}\sigma}{\normlbd}.
\end{align*}
Similarly, the fact $\boldsymbol{v}^T\boldsymbol{M}\boldsymbol{u} \le \Vert \boldsymbol{v} \Vert_2\Vert \boldsymbol{u} \Vert_2\Vert \boldsymbol{M} \Vert_2$ gives with probability at least $1-\delta$,
\begin{align*}
    \boldy^T\boldu_0^{-1}\boldy = \boldy^T(\frac{1}{\normlbd}\boldsymbol{I} - \boldsymbol{E}^{'})\boldy \le \boldy^T\frac{1}{\normlbd}\boldsymbol{y} \le \frac{C_5n}{\normlbd}.
\end{align*}
Combining the results above gives
\begin{align*}
    y_i(g_i+hg_i-sf_i) &\ge \Big(\frac{\normlbd - C_1n\sigma}{\normlbd}\Big)\frac{1}{2\normlbd} - \frac{C_2n\sqrt{\log(2n)}\sigma}{\normlbd^2}\\
    & \ge \frac{\normlbd - C_1n\sigma - 2C_2n\sqrt{\log(2n)}\sigma}{2\normlbd^2}.
\end{align*}
To make the expression above positive, it suffices to have $\normlbd \ge C n\sqrt{\log(2n)}\sigma$. This completes the proof.

\subsection{Proof of Theorem \ref{thm-linkiso}}
\label{pfsec-linkiso}
According to section \ref{pf-seclinkgen}, we need to study:
\begin{align*}
    y_i(g_i+hg_i-sf_i) = (1+\boldsymbol{y}^T\boldsymbol{U}_0^{-1}\boldsymbol{d})y_i(\boldsymbol{e}_i^T\boldsymbol{U}_0^{-1}\boldsymbol{y}) - y_i(\boldsymbol{e}_i^T\boldsymbol{U}_0^{-1}\boldsymbol{d})\boldsymbol{y}^T\boldsymbol{U}_0^{-1}\boldsymbol{y}.
\end{align*}
First, consider the term $y_i(\boldsymbol{e}_i^T\boldsymbol{U}_0^{-1}\boldsymbol{y})$. By Lemma \ref{lm-linkiso03}, if $d^{'}(n) = p-n+1 > 9 n\log(n)$, then, $4n\log(n) < \frac{4}{9}d^{'}(n)$ gives
\begin{align}
    %\label{eq-pfyi}
     y_i g_i = y_i(\boldsymbol{e}_i^T\boldsymbol{U}_0^{-1}\boldsymbol{y}) & > \frac{1}{4\sqrt{n}}\frac{2\sqrt{n}d^{'}(n) - \frac{4}{3}\sqrt{n}d^{'}(n) - \frac{4}{9}d^{'}(n)}{(d^{'}(n)+\sqrt{4/(9n)}d^{'}(n))d^{'}(n)} \notag\\ 
     & > \frac{1}{4\sqrt{n}}\frac{2\sqrt{n}d^{'}(n) - \frac{4}{3}\sqrt{n}d^{'}(n) - \frac{4}{9}\sqrt{n}d^{'}(n)}{2d^{'}(n)^2} \notag\\
     & > \frac{1}{4\sqrt{n}}\frac{(2-\frac{4}{3} - \frac{4}{9})\sqrt{n}}{2p} \notag\\
     \label{eq-pfyi02}
     &>\frac{1}{36p}.
\end{align}
Second, by Lemmas \ref{lem-ineqforu}, \ref{lm-boundforf} and \eqref{eq-pfyi02}, we find that for large enough constants $C_i$'s $>1$, with probability at least $1-\frac{C_1}{n^2}$,
\begin{align*}
    y_i(g_i+hg_i-sf_i) &= (1+h)y_i g_i - y_isf_i \\
    & \ge (1-|h|)\frac{1}{36p} - \max_{i\in [n]}|f_i|s \\
    & \ge (1-\frac{C_2n}{p}\normeta)\frac{1}{36p} - \frac{C_3\sqrt{\log(2n)}}{p}\normeta\frac{n}{p}\\
    & \ge \frac{p - C_2n\normeta - 36C_3\sqrt{2\log(2n)}n\normeta}{36p^{2}}\\
    & \ge \frac{p - 36C_4\sqrt{2\log(2n)}n\normeta}{36p^{2}}.
\end{align*}
To make the expression above positive, it suffices to have $p > C_5n\sqrt{\log(2n)}\normeta$, for large enough $C_5 > 1$. The result above holds for every $\gamma^*_i$, $i \in [n]$ with probability $1-\frac{C_1}{n^2}$ each (by Lemma \ref{lm-linkiso02}). Applying union bound over all $n$ training data points, we conclude that $y_i\gamma^*_i > 0$ for all $i$ with probability at least $1-\frac{C_1}{n}$. This completes the proof.

\section{Proof of Theorem \ref{thm-eqvar01} and Theorem \ref{thm-eqvariso}}\label{pfsec-pftesterrortwo}

\subsection{Proof of Theorem \ref{thm-eqvar01}}

From Section \ref{sec-pfoutline}, we need to lower bound the ratio
\begin{equation}
\label{pf-ratio01}
\frac{(\boldsymbol{y}^T(\boldsymbol{X}\boldsymbol{X}^T + \tau\boldsymbol{I})^{-1}\boldsymbol{X} \boldsymbol{\eta})^2}{ \boldsymbol{y}^T(\boldsymbol{X}\boldsymbol{X}^T + \tau\boldsymbol{I})^{-1}\boldsymbol{X}\boldsymbol{\Sigma}\boldsymbol{X}^T (\boldsymbol{X}\boldsymbol{X}^T + \tau\boldsymbol{I})^{-1}\boldsymbol{y}}.
\end{equation}
We will upper bound the denominator and lower bound the numerator. We first look at the denominator. We know $\boldsymbol{X}\boldsymbol{\Sigma}\boldsymbol{X}^T =  (\boldy\betab^T+\boldsymbol{Z}\boldsymbol{\Lambda}^{\frac{1}{2}})\boldsymbol{\Lambda}(\boldy\betab^T+\boldsymbol{Z}\boldsymbol{\Lambda}^{\frac{1}{2}})^T$. Define $\boldsymbol{A} := (\xxplustau)^{-1}\boldy\boldy^T(\xxplustau)^{-1}$. Then by the cyclic property of trace, the denominator of \eqref{pf-ratio01} can be expressed as
\begin{align*}
    &\text{Tr}\Big(\boldsymbol{y}^T(\boldsymbol{X}\boldsymbol{X}^T + \tau\boldsymbol{I})^{-1}\boldsymbol{X}\boldsymbol{\Sigma}\boldsymbol{X}^T (\boldsymbol{X}\boldsymbol{X}^T + \tau\boldsymbol{I})^{-1}\boldsymbol{y}\Big)\\
    &=\text{Tr}\Big((\boldy\betab^T+\boldsymbol{Z}\boldsymbol{\Lambda}^{\frac{1}{2}})\boldsymbol{\Lambda}(\boldy\betab^T+\boldsymbol{Z}\boldsymbol{\Lambda}^{\frac{1}{2}})^T\boldsymbol{A}\Big)\\
    &=\sum_{i=1}^p \lambda_i^2\boldsymbol{z}_i^T\boldsymbol{A}\boldsymbol{z}_i + \sum_{i=1}^p \lambda_i\beta_i^2(\boldy^T\boldsymbol{A}\boldy) + 2\sum_{i=1}^p \lambda_i^{1.5}\beta_i\boldsymbol{z}_i^T\boldsymbol{A}\boldy \\
    &\le 2\Big(\sum_{i=1}^p \lambda_i^2\boldsymbol{z}_i^T\boldsymbol{A}\boldsymbol{z}_i + \sigma^2\boldy^T\boldsymbol{A}\boldy \Big)\\
    &\le 2\Big(\sum_{i=1}^p\lambda_i^2\Vert \boldsymbol{A}\Vert_2\Vert \boldsymbol{z}_i\Vert_2^2 + \sigma^2(\boldy^T(\boldsymbol{X}\boldsymbol{X}^T + \tau\boldsymbol{I})^{-1}\boldy)^2\Big),
\end{align*}
where the first inequality follows from the inequality $\boldsymbol{v}^T\boldsymbol{M}\boldsymbol{u} \le \frac{1}{2}(\boldsymbol{v}^T\boldsymbol{M}\boldsymbol{v}+\boldsymbol{u}^T\boldsymbol{M}\boldsymbol{u})$ for positive semidefinite matrix $\boldsymbol{M}$ and $\boldsymbol{z}_i$ is the $i$-th column of matrix $\boldsymbol{Z}$. Thus, we need to upper bound $\sum_{i=1}^p\lambda_i^2\Vert \boldsymbol{A}\Vert_2\Vert \boldsymbol{z}_i\Vert_2^2$ and $\sigma^2(\boldy^T(\boldsymbol{X}\boldsymbol{X}^T + \tau\boldsymbol{I})^{-1}\boldy)^2$. For $\sum_{i=1}^p\lambda_i^2\Vert \boldsymbol{A}\Vert_2\Vert \boldsymbol{z}_i\Vert_2^2$, note that $\Vert \boldsymbol{z}_i \Vert_2^2$'s are independent sub-exponential random variables \citep[Chapter 2]{vershynin2018high}, thus for a fixed number $B>0$, $\sum_{i=1}^p\lambda_i^2B\Vert \boldsymbol{z}_i\Vert_2^2$ is the weighted sum of sub-exponential random variables, with the weights given by $B\lambda_i^2$ in blocks of size $n$ \citep[Lemma 7 and Corollary 1]{bartlett2020benign}. By Lemma \ref{pf-lemmasubexp}, with probability at least $1-2e^{-x}$,
\begin{align*}
    \sum_{i=1}^p\lambda_i^2B\Vert \boldsymbol{z}_i\Vert_2^2 & \le Bn\sumlbdsq + Bc\max\Big(\lambda_1^2x, \sqrt{xn\sum_{i}^p\lambda_i^4}\Big)\\
    & \le Bn\sumlbdsq + Bc\max\Big(x\sumlbdsq, \sqrt{xn}{\sum_{i}^p\lambda_i^2}\Big)\\
    &\le CnB\sum_{i=1}^p\lambda_i^2,
\end{align*}
for $x < n/c_0$. The number $B$ above should be replaced by the upper bound of $\Vert \boldsymbol{A} \Vert_2$. Recall $\boldsymbol{A} := (\xxplustau)^{-1}\boldy\boldy^T(\xxplustau)^{-1}$, thus $\Vert \boldsymbol{A} \Vert_2 = \Vert \boldy^T(\xxplustau)^{-1} \Vert_2^2$. Further recalling $D := s(\normeta^2 - t) + (h+1)^2$, by Lemma \ref{lem-xinverse},
\begin{align*}
     \boldsymbol{y}^T(\boldsymbol{X}\boldsymbol{X}^T+\tau\boldsymbol{I})^{-1} &= \boldsymbol{y}^T\boldutau^{-1} - \frac{1}{D}\begin{bmatrix}\normeta s& h^2+h-st& s\end{bmatrix}\begin{bmatrix}
    \normeta\boldsymbol{y}^T\\ 
    \boldsymbol{y}^T\\
    \boldsymbol{d}^T
    \end{bmatrix} \boldutau^{-1}\\
    &=\frac{1}{D}\Big((1+h)\boldy^T\boldutau^{-1} - s\boldd^T\boldutau^{-1}\Big).
\end{align*}
Therefore, by Lemma \ref{lem-ineqforu}, with probability at least $1-\delta$,
\begin{align*}
    \Vert \boldy^T(\xxplustau)^{-1} \Vert_2 &\le \frac{1}{D}\bigg(\Big(1+|h|\Big)\Vert \boldy \Vert_2\Vert \boldutau^{-1} \Vert_2 + s\Vert \boldd \Vert_2\Vert \boldutau^{-1} \Vert_2\bigg)\\
    &\le \frac{1}{D}\bigg(\Big(1+\frac{C_1n\sigma}{(\tau+\normlbd)}\Big)\frac{C_2\sqrt{n}}{(\tau+\normlbd)} + \frac{C_3n\sqrt{n}\sigma}{(\tau+\normlbd)^2}\bigg).
\end{align*}
The above result can be further simplified. If $1 \le \frac{n\sigma}{\tau + \normlbd}$, then,
\begin{align*}
    \Vert \boldy^T(\xxplustau)^{-1} \Vert_2 &\le \frac{1}{D}\frac{C_4n\sqrt{n}\sigma}{(\tau+\normlbd)^2}.
\end{align*}
If $1 > \frac{n\sigma}{\tau + \normlbd}$, then,
\begin{align*}
    \Vert \boldy^T(\xxplustau)^{-1} \Vert_2 &\le \frac{1}{D}\frac{C_5\sqrt{n}}{(\tau+\normlbd)}.
\end{align*}
Combining the above gives, with probability at least $1-\delta$,
\begin{align}
    \label{pf-balanceden01}
    \sum_{i=1}^p\lambda_i^2\Vert\boldsymbol{A}\Vert_2\Vert\boldsymbol{z}_i\Vert_2^2 \le \frac{C}{D^2}\frac{n^2}{(\tau+\normlbd)^2}\max\{1,\frac{n^2\sigma^2}{(\tau+\normlbd)^2}\}\normlbdtwo^2.
\end{align}
Now we look at $\sigma^2(\boldy^T(\boldsymbol{X}\boldsymbol{X}^T + \tau\boldsymbol{I})^{-1}\boldy)^2$. We need to upper bound $\boldy^T(\xxplustau)^{-1}\boldy$. Using Lemma \ref{lem-ineqforu} gives with probability at least $1-\delta$,
%Here we will lower bound $\boldsymbol{y}^T(\boldsymbol{X}\boldsymbol{X}^T)^{-1}\boldsymbol{X} \boldsymbol{\eta}$ and upper bound $\boldsymbol{y}^T(\boldsymbol{X}\boldsymbol{X}^T)^{-1}\boldsymbol{y}$. By Lemma \ref{lem-miserror}, we know that the bound is not useful if $\boldsymbol{y}^T(\boldsymbol{X}\boldsymbol{X}^T)^{-1}\boldsymbol{X} \boldsymbol{\eta} < 0$, hence we need the conditions that ensure $\boldsymbol{y}^T(\boldsymbol{X}\boldsymbol{X}^T)^{-1}\boldsymbol{X} \boldsymbol{\eta} \ge 0$ with high probability. Using \eqref{eq-xinverse} and some algebra steps, it can be checked that:
\begin{align*}
   \boldsymbol{y}^T(\xxplustau)^{-1}\boldsymbol{y}&=s - \frac{\normeta^2 s^2 + sh^2 + 2sh - s^2t}{s(\normeta^2 -t) + (h+1)^2}\\
   &= \frac{s}{s(\normeta^2 -t) + (h+1)^2}\\
   &= \frac{s}{D} \le \frac{1}{D}\frac{Cn}{(\tau+\normlbd)}.
\end{align*}
Therefore, $\sigma^2\boldy^T\boldsymbol{A}\boldy \le \frac{C}{D^2}\frac{n^2\sigma^2}{(\tau+\normlbd)^2}$. Hence the denominator of \eqref{pf-ratio01} is upper bounded by
\begin{align}
    \label{pf-balancedensum}
    \frac{1}{D^2}\frac{n^2}{(\tau+\normlbd)^2}\Big(C_1\max\{1,\frac{n^2\sigma^2}{(\tau+\normlbd)^2}\}\normlbdtwo + C_2\sigma^2\Big).
\end{align}
Now we look at the numerator of \eqref{pf-ratio01}. By Lemma \ref{lem-xinverse},
\begin{align*}
\boldsymbol{y}^T(\xxplustau)^{-1}\boldsymbol{X} \boldsymbol{\eta} &= \normeta^2\boldsymbol{y}^T(\xxplustau)^{-1}\boldsymbol{y} + \boldsymbol{y}^T(\xxplustau)^{-1}\boldsymbol{Q}\boldsymbol{\eta} \\
&=\frac{s(\normeta^2 -t)+h^2+h}{D}.
\end{align*}
%Combining the above gives
The numerator needs to be lower bounded and Lemma \ref{lem-ineqforu} gives with probability at least $1-\delta$,
\begin{align}
    %\label{eq-pfcorr01}
    \label{eq-pfcorr02}
    s(\normeta^2 -t) +h^2 + h &\ge s(\normeta^2 -t) + h \ge \frac{n}{C(\tau + \normlbd)}\Big(\normeta^2 - \frac{C_1n\sigma^2}{\tau + \normlbd} - C_2\sigma\Big).
\end{align}
Combining \eqref{pf-balancedensum} and \eqref{eq-pfcorr02} gives with probability at least $1-\delta$, \eqref{pf-ratio01} is lower bounded by
\begin{align}
\label{eq-pfthm01}
   \frac{\Big(\normeta^2 - \frac{C_1n\sigma^2}{\tau + \normlbd} - C_2\sigma\Big)^2}{C_3\max\{1,\frac{n^2\sigma^2}{(\tau+\normlbd)^2}\}\normlbdtwo^2 + C_4\sigma^2}.
\end{align}
This completes the proof of the theorem.

\subsection{Proof of Theorem \ref{thm-eqvariso}}
\label{pfsec-erroriso}
We need to lower bound the ratio
\begin{equation}
\label{pf-01}
\frac{(\boldsymbol{y}^T(\boldsymbol{X}\boldsymbol{X}^T)^{-1}\boldsymbol{X} \boldsymbol{\eta})^2}{ \boldsymbol{y}^T(\boldsymbol{X}\boldsymbol{X}^T)^{-1}\boldsymbol{y}}.
\end{equation}
Here we will lower bound $\boldsymbol{y}^T(\boldsymbol{X}\boldsymbol{X}^T)^{-1}\boldsymbol{X} \boldsymbol{\eta}$ and upper bound $\boldsymbol{y}^T(\boldsymbol{X}\boldsymbol{X}^T)^{-1}\boldsymbol{y}$. By Lemma \ref{lem-miserror}, we know that the bound is not useful if $\boldsymbol{y}^T(\boldsymbol{X}\boldsymbol{X}^T)^{-1}\boldsymbol{X} \boldsymbol{\eta} < 0$, hence we need the conditions that ensure $\boldsymbol{y}^T(\boldsymbol{X}\boldsymbol{X}^T)^{-1}\boldsymbol{X} \boldsymbol{\eta} \ge 0$ with high probability. Using \eqref{eq-xinverse} and some algebra steps, it can be checked that:
\begin{align}
\label{eqpf-errorisodeno}
   \boldsymbol{y}^T(\boldsymbol{X}\boldsymbol{X}^T)^{-1}\boldsymbol{y} &= s - \frac{\normeta^2 s^2 + sh^2 + 2sh - s^2t}{s(\normeta^2 -t) + (h+1)^2} = \frac{s}{s(\normeta^2 -t) + (h+1)^2}.
\end{align}
Similarly,
\begin{align*}
\boldsymbol{y}^T(\boldsymbol{X}\boldsymbol{X}^T)^{-1}\boldsymbol{X} \boldsymbol{\eta} = \normeta^2\boldsymbol{y}^T(\boldsymbol{X}\boldsymbol{X}^T)^{-1}\boldsymbol{y} + \boldsymbol{y}^T(\boldsymbol{X}\boldsymbol{X}^T)^{-1}\boldsymbol{Q}\boldsymbol{\eta} =\frac{s\normeta^2-st+h^2+h}{s(\normeta^2 -t) + (h+1)^2}.
\end{align*}
Combining the above gives
\begin{equation}
    \label{eq-pfnumde}
    \frac{(\boldsymbol{y}^T(\boldsymbol{X}\boldsymbol{X}^T)^{-1}\boldsymbol{X} \boldsymbol{\eta})^2}{ \boldsymbol{y}^T(\boldsymbol{X}\boldsymbol{X}^T)^{-1}\boldsymbol{y}} = \frac{\Big(s(\normeta^2 -t) +h^2 + h\Big)^2}{s \Big(s(\normeta^2 -t) + (h+1)^2 \Big)}.
\end{equation}
%We need the lower bounds and upper bounds for $s$, $t$ and $h$. Before moving forward, we define $\alpha$. Recall the eigendecomposition of $\boldsymbol{\Sigma} = \boldsymbol{V}\boldsymbol{\Lambda} \boldsymbol{V}^T$ and the weight vector $\boldsymbol{\eta}= \boldsymbol{V}\boldsymbol{\beta}$. $\alpha$ is defined as $\frac{\boldsymbol{\beta}^T \boldsymbol{V}^{'}{\boldsymbol{V}^{'}}^T\boldsymbol{\beta}}{\Vert \boldsymbol{\eta} \Vert_2 ^2}$, where $\boldsymbol{V}^{'} \in \mathbb{R}_{p \times n}$ is an orthogonal matrix whose columns are the right-singular vectors of a $\boldsymbol{Q}$ defined in (\ref{pf-datax02}). Under the isotropic ensemble, $\alpha$ is close to $\frac{n}{p}$ with high probability \cite{vershynin2018high, wainwright2019high}. Our numerical result shows that even if $\boldsymbol{\Sigma}$ is not identity, with $n$ unchanged, $\alpha$ still decreases when $p$ increases. Therefore, sufficient overparameterization can make $\alpha$ sufficiently small.
The numerator needs to be lower bounded and Lemma \ref{lem-ineqforu} gives with probability at least $1-\delta$,
\begin{align}
    %\label{eq-pfcorr01}
    s(\normeta^2 -t) +h^2 + h &\ge s(\normeta^2 -t) + h \notag\\
    %\label{eq-pfcorr02}
    &\ge \frac{n}{C_1 p}(1-\frac{n}{p})\Vert \boldsymbol{\eta} \Vert_2 ^2 - C_2 \frac{n}{p} \Vert \boldsymbol{\eta} \Vert_2 \notag\\
    \label{eq-pfcorr03}
    &\ge \frac{n}{C_1 p}\Big((1-\frac{n}{p})\Vert \boldsymbol{\eta} \Vert_2 ^2 - C_3 \Vert \boldsymbol{\eta} \Vert_2 \Big).
\end{align}
Similarly, the denominator is upper bounded by:
\begin{align*}
     s \Big(s(\normeta^2 -t) + (h+1)^2 \Big) &\le  s \Big(s\normeta^2 + (1+|h|)^2 \Big)\\
     &\le C_1\frac{n}{p} \Big( C_1\frac{n}{p} \normeta^2 + (1+C_2 \frac{n}{p}\Vert \boldsymbol{\eta} \Vert_2)^2\Big)\\
     &\le C_1\frac{n}{p} \Big( C_3\frac{n}{p} \normeta^2 + C_4\Big)\\
     &\le C_5\frac{n^2}{p^2} \Big(\normeta^2 + \frac{p}{n}\Big),
\end{align*}
where we also use the fact $(a+b)^2 \le 2(a^2 + b^2)$ and $n<p$.
Combining the above results gives with probability at least $1-\delta$,
\begin{align}
\label{eq-pfthm01b}
    \frac{(\boldsymbol{y}^T(\boldsymbol{X}\boldsymbol{X}^T)^{-1}\boldsymbol{X} \boldsymbol{\eta})^2}{ \boldsymbol{y}^T(\boldsymbol{X}\boldsymbol{X}^T)^{-1}\boldsymbol{y}} &\ge \Vert \boldsymbol{\eta} \Vert_2^2 \,\frac{\big((1-\frac{n}{p})\Vert \boldsymbol{\eta} \Vert_2 - C_3 \big)^2}{C_6 ( \frac{p}{n} + \normeta^2)}.
\end{align}
To ensure the classification error is smaller than $0.5$, we need $p>b\cdot n$ and $(1-\frac{n}{p})\normeta > C_3$ for $b >1$ to make $\etals^T\etab >0$ with high probability. This completes the proof of the theorem.

\section{Proof of Theorem \ref{thm-bilevel01} and benign overfitting for the bi-level ensemble }
\label{sec-pfthmbilevel}
\subsection{Proof of Theorem \ref{thm-bilevel01}}
We first introduce some new notations. Following Assumption \ref{ass-onesparse}, we assume that the covariance matrix $\Sigmab$ is diagonal and the mean vector $\etab$ is one-sparse ($\eta_k \ne 0$ and $k \ne 1$). Hence the data matrix $\boldsymbol{X}$ can be written as
\begin{align*}
    \boldsymbol{X} = \boldy\etab^T + \boldsymbol{Q} =\boldy\etab^T + \boldsymbol{Z}\boldsymbol{\Lambda}^{\frac{1}{2}}.
\end{align*}
Let $\boldsymbol{z}_i$ be the $i$-th column of the matrix $\boldsymbol{Z}$ above whose elements are IID standard Gaussian. Recall $\boldutau := \boldsymbol{Q}\boldsymbol{Q}^T + \tau \boldsymbol{I}$, define
\begin{align*}
    s &:=\boldsymbol{y}^T \boldutau^{-1}\boldsymbol{y},\\ 
t_k &:=\boldz_k^T \boldutau^{-1}\boldz_k,\\ 
f_1 &:=\boldsymbol{y}^T \boldutau^{-1}\boldsymbol{z}_1,\\ 
f_k &:=\boldsymbol{y}^T \boldutau^{-1}\boldsymbol{z}_k,\\  
g_1 &:=\boldsymbol{z}_1^T\boldutau^{-1}\boldsymbol{z}_k.
\end{align*}
\begin{lemma}[Bi-level]
\label{lem-ineqforubilevel}
Assume that $\Sigmab$ follows the bi-level ensemble defined in Definition \ref{def-balance}. Fix $\delta\in(0,1)$ and suppose $n$ is large enough such that $n>c\log(1/\delta)$ for some $c>1$. Then, there exists constants $C_i$'s $>1$ such that with probability at least $1-\delta$, the following results hold:
\begin{align*}
    \frac{n}{C_1 (\tau+\binormlbd)} \le & s \le \frac{C_2n}{(\tau+\binormlbd)},\\
    - \frac{C_3n}{(\tau+\binormlbd)} \le & f_k \le \frac{C_3n}{(\tau+\binormlbd)} ,\\
    \frac{n}{C_4(\tau+\binormlbd)} \le & t_k \le \frac{C_5n}{(\tau+\binormlbd)} ,\\
    - \frac{C_6n}{(\tau+\binormlbd + n\lambda_1)} \le & f_1 \le \frac{C_6n}{(\tau+\binormlbd+n\lambda_1)} ,\\
    - \frac{C_7n}{(\tau+\binormlbd + n\lambda_1)} \le & g_1 \le \frac{C_7n}{(\tau+\binormlbd+n\lambda_1)},\\
    \Vert \boldy^T\boldutau^{-1} \Vert_2 &\le \frac{C_8\sqrt{n}}{(\tau+\binormlbd)}, \\
    \Vert \boldz_k^T\boldutau^{-1} \Vert_2 &\le \frac{C_9\sqrt{n}}{(\tau+\binormlbd)}.
%    - C_3 \frac{\sqrt{n}}{p} \Vert \boldsymbol{\eta} \Vert_2  \le & f_i  \le   C_3 \frac{\sqrt{n}}{p} \Vert \boldsymbol{\eta} \Vert_2, \ \text{for} \ i \in [n].
\end{align*}
\end{lemma}
Now we are ready to prove Theorem \ref{thm-bilevel01}. We know from the proof outline that we need to lower bound
\begin{align}
\label{pf-ratiobilevel}
    \frac{(\etareg^T\boldsymbol{\eta})^2}{\etareg^T \boldsymbol{\Sigma}\etareg} = \frac{(\eta_k\hat{\eta}_k)^2}{\sum_{i=1}^p\lambda_i\hat{\eta}_i^2}.
\end{align}
We divide $\hat{\eta}_i$'s into 3 groups: $\hat{\eta}_1$, $\hat{\eta}_k$ and the rest. Rather than lower bounding \eqref{pf-ratiobilevel}, we will upper bound its reciprocal. Specifically, we will upper bound
\begin{align}
\label{pf-uppthree}
    \frac{\lambda_1\hat{\eta}_1^2}{(\eta_k\hat{\eta}_k)^2}, \ \ \frac{\sum_{i\ne1,k}\lambda_i\hat{\eta}_i^2}{(\eta_k\hat{\eta}_k)^2} \ \ \text{and} \ \ \frac{\lambda_k\hat{\eta}_k^2}{(\eta_k\hat{\eta}_k)^2},
\end{align}
then reverse the sum of the upper bounds of the three ratios above to obtain the lower bound of \eqref{pf-ratiobilevel}.

Following the fact that $\hat{\eta}_i = \boldsymbol{e}_i^T\hat{\etab}$, we have
\begin{align}
    \hat{\eta}_i &= \sqrt{\lambda_i}\boldsymbol{z}_i^T(\xxplustau)^{-1}\boldy, \ \ \ \text{for} \ \ i \ne k,\label{pf-etaibilevel}\\
    \hat{\eta}_k &= \eta_k\boldy^T(\xxplustau)^{-1}\boldy+\sqrt{\lambda_k}\boldsymbol{z}_k^T(\xxplustau)^{-1}\boldy.
\end{align}
To upper bound the 3 terms in \eqref{pf-uppthree}, we need to lower bound $(\eta_k\hat{\eta}_k)^2$.
Recall that under Assumption \ref{ass-onesparse}, $\sigma^2$ is $\lambda_k\eta_k^2$. By Lemma \ref{lem-xinverse} and using our newly defined notations, we have
\begin{align*}
    \eta_k\hat{\eta}_k &=  \frac{\eta_k^2s}{D}+\sqrt{\lambda_k}\eta_k\Big(f_k-\frac{\normeta^2sf_k + ((\sqrt{\lambda_k}\eta_kf_k)^2+\sqrt{\lambda_k}\eta_kf_k-s(\lambda_k\eta_k^2t_k))f_k + \sqrt{\lambda_k}\eta_kt_ks}{D}\Big)\\
    &=\frac{1}{D}\Big(\eta_k^2s(1-\lambda_kt_k)+\sigma f_k + \sigma^2f_k^2\Big),
\end{align*}
where $D$ becomes $(\sigma f_k+1)^2 + s(\eta_k^2 - \sigma^2t_k)$. Lemma \ref{lem-ineqforubilevel} gives with probability at least $1-\delta$,
\begin{align}
    \label{pf-bilevelnume}
    \eta_k\hat{\eta}_k \ge \frac{1}{D}\frac{n}{C_1(\tau + \binormlbd)}\Big(\eta_k^2(1-\frac{C_2n\lambda_k}{\tau + \binormlbd}) - \sigma \Big).
\end{align}
Now we upper bound $\sum_{i\ne1,k}\lambda_i\hat{\eta}_i^2$. From the proof of Theorem \ref{thm-eqvar01}, we know
\begin{align*}
    \sum_{i\ne1,k}\lambda_i\hat{\eta}_i^2 &= \sum_{i\ne1,k}\lambda_i^2\boldz_i^T\boldsymbol{A}\boldz_i \le \sum_{i\ne1,k}\lambda_i^2 \Vert \boldz_i \Vert_2^2\Vert \boldsymbol{A} \Vert_2,
\end{align*}
where $\boldsymbol{A} = (\xxplustau)^{-1}\boldy\boldy^T(\xxplustau)^{-1}$. Then we need to upper bound $\Vert \boldy^T(\xxplustau)^{-1} \Vert_2$. Following Lemmas \ref{lem-xinverse} and \ref{lem-ineqforubilevel}, we have
\begin{align*}
    \Vert \boldy^T(\xxplustau)^{-1} \Vert_2 &= \Vert \frac{1}{D}\Big(\boldy^T\boldutau^{-1}(1+\sigma f_k) - \boldz_k^T\boldutau^{-1}(\sigma s)\Big)\Vert_2 \\
    &\le\frac{1}{D}\Big((1+\sigma f_k)\Vert \boldy^T\boldutau^{-1} \Vert_2 + (\sigma s)\Vert \boldz_k^T\boldutau^{-1} \Vert_2\Big)\\
    &\le \frac{1}{D}\frac{C_1\sqrt{n}}{(\tau+\binormlbd)}\Big(1+\frac{C_2n\sigma}{(\tau+\binormlbd)}\Big).
\end{align*}
Note for a fixed number $B>0$, $\sum_{i\ne1,k}\lambda_i^2 \Vert \boldz_i \Vert_2^2B$ is the weighted sum of sub-exponential random variables. By \citet[Lemma 7]{bartlett2020benign}, with probability at least $1-2e^{-x}$,
\begin{align*}
    \sum_{i=1}^p\lambda_i^2B\Vert \boldsymbol{z}_i\Vert_2^2 \le Cn\sum_{i=1}^p\lambda_i^2B,
\end{align*}
for $x < n/c_0$. Combining \eqref{pf-bilevelnume} and bounds above with $B$ replaced by the upper bound of $\boldsymbol{A}$ gives with probability at least $1-\delta$,
\begin{align}
    \label{pf-uppbdbilevel01}
    \frac{\sum_{i\ne1,k}\lambda_i\hat{\eta}_i^2}{(\eta_k\hat{\eta}_k)^2} \le \frac{C_1\sum_{i\ne1,k}\lambda_i^2}{\Big(\eta_k^2(1-\frac{C_2n\lambda_k}{\tau + \binormlbd}) - \sigma\Big)^2}\Big(1+\frac{C_3n\sigma}{(\tau+\binormlbd)}\Big)^2.
\end{align}
Next we upper bound $\lambda_1\hat{\eta}_1^2$. \eqref{pf-etaibilevel} gives
\begin{align*}
    \hat{\eta}_1 &= \sqrt{\lambda_1}\boldsymbol{z}_1^T(\xxplustau)^{-1}\boldy\\
    &=\frac{\sqrt{\lambda_1}}{D}(f_1 + \sigma f_1f_k - \sigma g_1s)\\
    &\le \frac{\sqrt{\lambda_1}}{D}\frac{C_6n}{(\tau+\binormlbd+n\lambda_1)}\Big(1+\frac{C_3n\sigma}{(\tau+\binormlbd)}\Big).
\end{align*}
Combining the result above and \eqref{pf-bilevelnume} gives with probability at least $1-\delta$,
\begin{align}
    \label{pf-uppbdbilevel02}
    \frac{\lambda_1\hat{\eta}_1^2}{(\eta_k\hat{\eta}_k)^2} \le \frac{C_1\lambda_1^2}{\Big(\eta_k^2(1-\frac{C_2n\lambda_k}{\tau + \binormlbd}) - \sigma\Big)^2}\frac{(\tau + \binormlbd)^2}{(\tau + \binormlbd +n\lambda_1)^2}\Big(1+\frac{C_3n\sigma}{(\tau+\binormlbd)}\Big)^2.
\end{align}
In addition, $\frac{\lambda_k\hat{\eta}_k^2}{(\eta_k\hat{\eta}_k)^2}$ in \eqref{pf-uppthree} is $\frac{\lambda_k}{\eta_k^2}$. Then the sum of \eqref{pf-uppbdbilevel01}, \eqref{pf-uppbdbilevel02} and $\frac{\lambda_k}{\eta_k^2}$ is
\begin{align*}
%\label{pf-invuppbound}
    \frac{A + B +\lambda_k\Big(\eta_k(1-\frac{C_2n\lambda_k}{\tau + \binormlbd}) - \sqrt{\lambda_k}\Big)^2}{\Big(\eta_k^2(1-\frac{C_2n\lambda_k}{\tau + \binormlbd}) - \sigma\Big)^2} \notag \le & \frac{A + B +\lambda_k\Big(\eta_k^2 + {\lambda_k}\Big)}{\Big(\eta_k^2(1-\frac{C_2n\lambda_k}{\tau + \binormlbd}) - \sigma\Big)^2},
\end{align*}
where we use $(a-b)^2 \le a^2+b^2$ for $a,b>0$ and with
\begin{align*}
     &A = C_3\lambda_1^2 \bigg(\frac{\tau + \binormlbd)}{\tau + n \lambda_1 +\binormlbd}\bigg)^2\Big(1+\frac{C_4n\sigma}{\tau + \binormlbd}\Big)^2,\\
    &B = C_5\Big(\sum_{i \ne 1,k}\lambda_i^2\Big)\Big(1+\frac{C_4n\sigma}{\tau + \binormlbd}\Big)^2,
\end{align*}
for large constants $C_i$'s. The inverse of the upper bound of $\frac{\lambda_k}{\eta_k^2}$ gives the lower bound of \eqref{pf-ratiobilevel}. To ensure $\etareg^T\etab >0$ with high probability, we need $\eta_k^2 > \frac{c_1n\sigma^2}{\tau + \binormlbd} + c_2{\sigma}$ for $c_1, c_2 >1$.

\subsection{Benign overfitting for the bi-level ensemble}
For the bi-level ensemble, the first condition in Theorem \ref{thm-linkgen} is not satisfied, hence we can no longer analyze $\etasvm$ by studying $\etals$. We show a regime that suffices to make the classification error of $\etals$ vanish as $p$ increase to $+\infty$. Consider the setting:
\begin{align}
    \label{eq-bilevelbenign}
    \lambda_2 = \cdots = \lambda_p = \lambda \ \ \text{and} \ \ \lambda_1 = \alpha p\lambda, \ \ \text{for} \ \alpha>1. 
\end{align}
For large enough $p$, the setting above can ensure the bi-level ensemble condition in \eqref{def-bilevel} is satisfied. 
\begin{corollary}
\label{cor-bilevelboundzero}
Assume that the data generating process follows Assumption \ref{ass-onesparse} and \eqref{eq-bilevelbenign}. Fix $\delta\in(0,1)$ and suppose $n$ is finite but large enough such that $n>c\log(1/\delta)$ for some $c>1$. Then for large enough $C >1$, with probability at least $1-\delta$, $\mathcal{R}(\etals)$, the expected 0-1 loss of the least squares estimator $\etals$, approaches $0$ as $p \to \infty$, provided that $\eta_k > C\sqrt{\lambda}p^{r}$, for $r > \frac{1}{2}$.
\end{corollary}
\begin{proof}
First, the bound on the unregularized estimator $\etals$ can be obtained by setting $\tau = 0$ in \eqref{eq-bilevel02}. Thus, with probability at least $1-\delta$, $\mathcal{R}(\etals)$ is upper bounded by
\begin{align}
\label{eq-corbilevel}
    &\exp \bigg(\frac{-\Big(\eta_k^2(1 - \frac{C_1n\lambda_k}{\binormlbd}) - C_2\sigma\Big)^2}{A + B + C_6(\lambda_k^2+\sigma^2)} \bigg), \ \ \\
    \text{with} \ &A = C_3\lambda_1^2 \bigg(\frac{\binormlbd+C_4n\sigma}{\binormlbd+n\lambda_1}\bigg)^2, 
    B = C_5\Big(\sum_{i \ne 1,k}\lambda_i^2\Big)\Big(1+\frac{C_4n\sigma}{\binormlbd}\Big)^2.\notag
\end{align}
Note from \eqref{eq-bilevelbenign} that $\binormlbd = (p - 1)\lambda$. We first look at the denominator of the exponent of \eqref{eq-corbilevel}. Assuming $n\eta_k \le p\sqrt{\lambda}$, we have
\begin{align*}
    A &= C_1\lambda_1^2 \bigg(\frac{\binormlbd+C_2n\sigma)}{n \lambda_1 +\binormlbd}\bigg)^2 \\
    &= C_1 \alpha^2p^2\lambda^2\Big(\frac{(p-1)\lambda + C_2n\sigma}{n\alpha p \lambda +(p-1)\lambda}\Big)^2\\
    &\le C_3\alpha^2p^2\lambda^2\Big(\frac{p\lambda}{n\alpha p \lambda +(p-1)\lambda}\Big)^2\\
    &\le C_4\alpha^2p^2\lambda^2\Big(\frac{p\lambda}{n\alpha p \lambda}\Big)^2, \ \ \text{by} \ n\alpha p \gg 1\\
    &\le C_4\frac{p^2\lambda^2}{n^2}.
\end{align*}
Moreover,
\begin{align*}
    B &= C_5\Big(\sum_{i \ne 1,k}\lambda_i^2\Big)\Big(1+\frac{C_6n\sigma}{ \binormlbd}\Big)^2\\
    & = C_5(p-2)\lambda^2\Big(1+\frac{C_6n\sigma}{\binormlbd}\Big)^2\\
    &\le C_7(p-2)\Big(\lambda^2 + \frac{n^2\sigma^2}{(p-1)^2}\Big)\\
    &\le C_8(p-2)\lambda^2,
\end{align*}
where the last inequality comes from $n\eta_k \le p\sqrt{\lambda}$. Combining the results above, we have the denominator of the exponent of \eqref{eq-corbilevel} is upper bounded by
\begin{align}
\label{pf-bilevelbenign01}
    C_4\frac{p^2\lambda^2}{n^2} + C_8(p-2)\lambda^2 + C_9(\lambda^2 + \sigma^2) \le C_{10}(\frac{p^2}{n^2}\lambda^2 + p\lambda^2).
\end{align}
Now we look at the numerator.
\begin{align}
    \Big(\eta_k^2(1 - \frac{C_0n\lambda_k}{\binormlbd}) - C_2\sigma\Big)^2 &\ge \eta_k^4(1-C\frac{n}{p})^2 - C_9\sqrt{\lambda}\eta_k^3 \notag\\
    &\ge \eta_k^4 - 2C\eta_k^4\frac{n}{p} - C_9\sqrt{\lambda}\eta_k^3. \label{pf-bilevelbenign02}
\end{align}
Let $\eta_k = C\sqrt{\lambda}p^{r}$, for some $C >1$. Combining \eqref{pf-bilevelbenign01} and \eqref{pf-bilevelbenign02}, if $p>n^2$, in \eqref{pf-bilevelbenign01}, $\frac{p^2}{n^2} > p$, then the negative exponent of \eqref{eq-corbilevel} is lower bounded by
\begin{align}
\label{pf-bilevelbenign03}
    \frac{\eta_k^4n^2}{\lambda^2p^2} - C_{10}\frac{\eta_k^4n^3}{\lambda^2p^3} - C_{11}\frac{\eta_k^3n^2\sqrt{\lambda}}{\lambda^2p^2}.
\end{align}
Then
\begin{align*}
    \eqref{pf-bilevelbenign03} &\ge n^2p^{4r-2} - C_{10}n^3p^{4r-3} - C_{11}n^2p^{3r-2}\\
    &\ge n^2p^{4r-2} - C_{10}np^{4r-2} - C_{11}n^2p^{3r-2}, 
\end{align*}
where we use $\frac{p^2}{n^2} > p$ in the last inequality. Thus to make the bound above approach $+\infty$ as $p \to \infty$, it suffices to have $r > \frac{1}{2}$.

If $p\le n^2$, then the negative exponent of \eqref{eq-corbilevel} is lower bounded by
\begin{align}
%\label{pf-bilevelbenign04}
    \frac{\eta_k^4}{\lambda^2p} - C_{10}\frac{\eta_k^4n}{\lambda^2p^2} - C_{11}\frac{\eta_k^3\sqrt{\lambda}}{\lambda^2p} \ge  p^{4r-1} - C_{10}\frac{n}{p}p^{4r-1} - C_{11}p^{3r-1}.\notag
\end{align}
It suffices to have $r>\frac{1}{4}$ to make the bound above approach $+\infty$ as $p \to \infty$. Combing previous results, it suffices to have $\eta_k = C\sqrt{\lambda}p^{r}$, for $r > \frac{1}{2}$ and some $C >1$. Recall that we assume $n\eta_k \le p\sqrt{\lambda}$, and actually $n\eta_k > p\sqrt{\lambda}$ is stronger than the condition $\eta_k > C\sqrt{\lambda}p^{r}$, for $r > \frac{1}{2}$ and some $C >1$ for finite $n$, hence the later condition is sufficient to make the classification error approach $0$ as $p \to \infty$.
\end{proof}

\section{Proof of Corollaries \ref{cor-benign01} and \ref{cor-linkandtester}}
\subsection{Proof of Corollary \ref{cor-benign01}}
We need to find the conditions that make the negative exponent of \eqref{eq-testls} vanish as $p$ increases when conditions in Theorem \ref{thm-linkgen} hold. Recall Theorem \ref{thm-linkgen} requires
\begin{align}
\label{pf-linkgen02}
\normlbd &> \max\{\lambda_{*}, C_1n\sqrt{\log(2n)} \sigma\},
\end{align}
for some $C_1, C_2 >1$ and $\lambda_{*} = 72\Big(\Vert \boldsymbol{\lambda}\Vert_2\cdot n\sqrt{\log{n}} + \Vert \boldsymbol{\lambda}\Vert_{\infty}\cdot n\sqrt{n}\log{n} + 1\Big)$. It is not hard to check that \eqref{pf-linkgen02} implies that $\frac{n^2{\sigma^2}}{\normlbd^2} < 1$. Then the negative exponent of \eqref{eq-testls} is lower bounded by
\begin{align}
    \frac{\Big(\normeta^2 - \frac{C_1n\sigma^2}{\normlbd} - C_2\sigma\Big)^2}{C_3\max\{1,\frac{n^2\sigma^2}{\normlbd^2}\}\sumlbdsq + C_4\sigma^2} \ge &\frac{\Big(\normeta^2 - \frac{C_1n\sigma^2}{\normlbd} - C_2\sigma\Big)^2}{C_5\Big(\sumlbdsq + \sigma^2\Big)} \notag\\
    \ge &\frac{\Big(\normeta^2 - C_1\frac{\normlbd}{n\sqrt{\log(2n)}}\frac{n\sigma}{\normlbd} - C_2\sigma\Big)^2}{C_5\Big(\sumlbdsq + \sigma^2\Big)} \notag\\
    \ge &\frac{\Big(\normeta^2 - C_1\frac{\sigma}{\sqrt{\log(2n)}} - C_2\sigma\Big)^2}{C_5\Big(\sumlbdsq + \sigma^2\Big)} \notag\\
    \ge &\frac{\Big(\normeta^2  - C_6\sigma\Big)^2}{C_5\Big(\sumlbdsq + \sigma^2\Big)} \notag\\
    \label{pf-linkbound01}
    \ge &\frac{\normeta^4  - C_7\normeta^2\sigma}{C_5\Big(\sumlbdsq + \sigma^2\Big)}.
\end{align}
Note that $\betab = \begin{bmatrix}\beta &\beta &... &\beta\end{bmatrix}^T$, hence $\sigma = \beta\sqrt{\normlbd}$. Looking at the denominator of \eqref{pf-linkbound01}, when $\sumlbdsq \le \sigma^2$, i.e. $\Vert \lbdb \Vert_2^2 \le \beta^2\normlbd$,
\begin{align}
    \eqref{pf-linkbound01} &\ge \frac{\normeta^4  - C_7\normeta^2\sigma}{C_8\sigma^2} \notag \\
    &\ge \frac{(p\beta^2)^2  - C_7(p\beta^2)\sqrt{\beta^2\normlbd}}{C_8\Big(\beta^2\normlbd\Big)} \notag \\
    \label{pf-linkbound02}
    & \ge \Big(\frac{p\beta}{\sqrt{\normlbd}}\Big)^2 - \frac{C_9p\beta}{\sqrt{\normlbd}}.
\end{align}
To guarantee \eqref{pf-linkbound02} $\to \infty$ as $p \to \infty$, it suffices to have $\normlbd \le C\beta^2p^{\alpha}$, for $\alpha < 2$. Note that the second condition in Theorem \ref{thm-linkgen} becomes
\begin{align*}
    \normlbd > Cn\sqrt{\log(2n)}\beta\sqrt{\normlbd} \iff \normlbd > C^2n^2\log(2n)\beta^2.
\end{align*}
Combing the conditions above, the SVM solution goes to $0$ with $p \to \infty$ provided the assumptions of Theorem \ref{thm-eqvar01} with $\tau=0$ hold and
\begin{align}
\label{pf-linkcondition01}
    \max\{\lambda_{*}, C_1\beta^2n^2\log(2n)\} < \normlbd \le C_2\beta^2p^{\alpha}, \ \ \text{for} \ \alpha<2.
\end{align}

When $\sumlbdsq > \sigma^2$, i.e. $\Vert \lbdb \Vert_2^2 > \beta^2\normlbd$,
\begin{align}
    \eqref{pf-linkbound01} &\ge \frac{\normeta^4  - C_7\normeta^2\sigma}{C_8\Big({\sumlbdsq}\Big)} \notag \\
    \label{pf-linkbound03}
    & \ge \Big(\frac{p\beta^2}{\sqrt{\sumlbdsq}}\Big)^2 - \frac{C_9p\beta^2}{\sqrt{\sumlbdsq}}.
\end{align}
To guarantee \eqref{pf-linkbound03} $\to \infty$ as $p \to \infty$, it suffices to have $\sumlbdsq \le C\beta^4p^{\alpha}$, for $\alpha < 2$, which is equivalent to $\Vert \lbdb \Vert_2 \le C \beta^2p^{\alpha}$, for $\alpha < 1$. Combing the conditions in Theorem \ref{thm-linkgen}, the SVM solution goes to $0$ with $p \to \infty$ provided the assumptions of Theorem \ref{thm-eqvar01} with $\tau=0$ hold and
\begin{align}
\label{pf-linkcondition02}
     \normlbd > \max\{\lambda_{*}, C_1\beta^2n^2\log(2n)\} \ \ \text{and} \ \ \Vert \lbdb \Vert_2 \le C \beta^2p^{\alpha}, \ \ \text{for} \ \ \alpha < 1.
\end{align}
Combining  \eqref{pf-linkcondition01} and \eqref{pf-linkcondition02} completes the proof.

\subsection{Proof of Corollary \ref{cor-linkandtester}}
We start from the high-SNR regime. In fact, we can assume a bit stronger that
\begin{align}
\label{eq-pfhighsnr}
    \normeta^2 > C\frac{p}{n}, \ \ \text{for some large} \ \ C > 1.
\end{align}
Then the exponent in \eqref{eq-isolarge} becomes:
\begin{align}
    \normeta^2(1-\frac{n}{p})^2 + C_1 - 2C_1(1-\frac{n}{p})\normeta \notag&> \normeta^2 - 2\normeta^2\frac{n}{p} - 2C_1\normeta \notag\\
    \label{pf-highsnr02}
    &> C\frac{p}{n}-2\normeta^2\frac{n}{p} - 2C_1\normeta,
\end{align}
where the last inequality comes from \eqref{eq-pfhighsnr}. Following Theorem \ref{thm-linkiso}, we further assume that 
\begin{align}
\label{pf-svmlink}
    p > 10n\log n +n -1 \ \ \text{and} \ \ p>C_2n\sqrt{\log(2n)}\normeta, 
\end{align}
for some constant $C_2 > 1$. Then combining the relationships above gives
\begin{align}
    \eqref{pf-highsnr02} & > C\frac{p}{n} - 2\Big(\frac{p}{C_2n\sqrt{\log(2n)}}\Big)^2\frac{n}{p}- 2C_1\frac{p}{n\sqrt{\log(2n)}} \notag\\
    & = C\frac{p}{n} - \frac{2p}{C_2n{\log(2n)}} - \frac{2C_1p}{n\sqrt{\log(2n)}} \notag\\
    \label{pf-highsnr03}
    & = \frac{p}{n}\Big(C - \frac{2C_1}{\sqrt{\log(2n)}} - \frac{2}{C_2\log(2n)}\Big).
\end{align}
Notice that \eqref{pf-highsnr03} $\to \infty$ as $(p/n) \to \infty$ for sufficiently large $C$ and $n$. 
%For $n$ growing with $p$ growing, it follows from \eqref{pf-svmlink},
%\begin{align*}
%    \eqref{pf-highsnr03} > 10\log n \Big(C - \frac{2C_1}{\sqrt{\log(2n)}} - \frac{2}{C_2\log(2n)}\Big) \to \infty \ \ \text{as} \ \ n \to \infty.
%\end{align*}
Thus we have proved that in the high-SNR regime, the error of SVM solution goes to $0$ with $(p/n) \to \infty$ provided that the assumptions of Theorem \ref{thm-eqvariso} hold and
\begin{align*}
    \frac{1}{C}n\normeta^2 > p > \max\{10n\log n+n-1, C_1n\sqrt{\log(2n)}\normeta\},
\end{align*}
for sufficiently large constants $C, C_1 >1$.

%For the low-SNR regime, assume that
%\begin{align}
%\label{pf-lowsnr01}
%    & p > 10n\log n +n-1,\  p>C_1n\sqrt{\log(2n)}\normeta,\\
%    \label{pf-lowsnr02}
%    \text{and} \ & \normeta^2\le\frac{p}{n},\  \normeta^4 = C_2(\frac{p}{n})^{\alpha}, \ \ \text{for} \ \ \alpha > 1.
%\end{align}
%Then the exponent in \eqref{eq-isosmall} becomes:
%\begin{align}
%    \frac{n}{p}\normeta^4\Big((1-\frac{n}{p}) - C_3\frac{1}{\normeta}\Big)^2 \notag 
%    &> \frac{n}{p}\normeta^4 - 2\frac{n^2}{p^2}\normeta^4 - 2\frac{n}{p}C_3\normeta^3 \notag\\
%    \label{pf-lowsnr03}
%    &= C_2\Big(\frac{p}{n}\Big)^{\alpha -1}- 2C_2\Big(\frac{p}{n}\Big)^{\alpha-2} -2C_3\Big(\frac{p}{n}\Big)^{\alpha^{0.75}-1}.
%\end{align}
%where the last inequality comes from \eqref{pf-lowsnr01} and \eqref{pf-lowsnr02}. \eqref{pf-lowsnr03} 
%The term above will go to $+\infty$ as $p \to \infty$ provided that $\alpha > 1$.

For the low-SNR regime, assume that
\begin{align}
\label{pf-lowsnr01}
    & p > 10n\log n +n-1,\  p>C_1n\sqrt{\log(2n)}\normeta,\\
    \label{pf-lowsnr02}
    \text{and} \ & \normeta^2\le\frac{p}{n},\  \normeta^4 = C_2(\frac{p}{n})^{\alpha}, \ \ \text{for} \ \ \alpha > 1.
\end{align}
Then the exponent in \eqref{eq-isosmall} becomes:
\begin{align}
    \frac{n}{p}\normeta^4\Big((1-\frac{n}{p}) - C_3\frac{1}{\normeta}\Big)^2 \notag 
    &> \frac{n}{p}\normeta^4 - 2\frac{n^2}{p^2}\normeta^4 - 2\frac{n}{p}C_3\normeta^3 \notag\\
    \label{pf-lowsnr03}
    &\ge C_2\Big(\frac{p}{n}\Big)^{\alpha -1}- 2C_2\Big(\frac{p}{n}\Big)^{\alpha -2} -2C_3C_2\Big(\frac{p}{n}\Big)^{0.75\alpha -1},
\end{align}
where the last inequality comes from \eqref{pf-lowsnr01} and \eqref{pf-lowsnr02}. \eqref{pf-lowsnr03} will go to $+\infty$ as $(p/n) \to \infty$ provided that $\alpha  > 1$. Overall in the low-SNR regime, we need the assumptions of Theorem \ref{thm-eqvar01} plus
\begin{align*}
    &p > \max\{10n\log n +n-1, C_1n\sqrt{\log(2n)}\normeta, n\normeta^2\},\\
    \text{and} \ &\normeta^4 \ge C_2(\frac{p}{n})^{\alpha}, \ \ \text{for} \ \alpha \in (1, 2].
\end{align*}

\section{Results for the averaging estimator}
\label{subsec-avgest}
The theorem below shows an upper bound on the classification error for the averaging estimator $\etaavg$. Note the result below is for general $\Sigmab$, i.e. no balanced or bi-level structure is required.
%\begin{lemma}
%\label{lem-posicorravg}
%Fix $\delta\in(0,1)$ and suppose $n$ is large enough such that $n>c\log(1/\delta)$ for some $c>1$. 
%Then, there exist a constant $C_1>1$ such that with probability at least $1-\delta$, $\etaavg^T\boldsymbol{\eta} > 0$ provided that
%\begin{equation}
%\label{eq-posicorravg}
%   \normeta^2 > C_1\sigma.
%\end{equation}
%\end{lemma}

\begin{theorem}
\label{thm-avgest}
Assume that the data are generated with the GMM model. Fix $\delta\in(0,1)$ and suppose $n$ is large enough such that $n>c\log(1/\delta)$ for some $c>1$. Then, there exist a constant $c_1>1$ such that with probability at least $1-\delta$, $\etaavg^T\boldsymbol{\eta} > 0$ provided that $\normeta^2 > c_1\sigma$.
% Further let $p > bn$ for a large constant $b$ and condition \eqref{eq-posicorr02} is satisfied. 
Then, there exists constants $C_i$'s $>1$ such that with probability at least $1-\delta$,
\begin{align}
\label{eq-avgest}
    \mathcal{R}(\etaavg) \le \exp \bigg(\frac{-\Big(\normeta^2 -  C_1\sigma\Big)^2}{C_2\normlbdtwo^2 + C_3\sigma^2} \bigg).
\end{align}
\end{theorem}

The bound above is the same as \eqref{eq-testreginfty}.
\begin{proof}

We need to lower bound $\frac{(\etaavg^T\boldsymbol{\eta})^2}{\etaavg^T \boldsymbol{\Sigma}\etaavg}$. Recall that $\etaavg = \frac{1}{n}\boldsymbol{X}^T\boldy$. For the denominator,
\begin{align*}
   \etaavg^T \boldsymbol{\Sigma}\etaavg \le \frac{2}{n^2}\Big(n^2\etab^T\Sigmab\etab + \boldy^T\boldsymbol{Q}\boldsymbol{\Sigma}\boldsymbol{Q}^T\boldy\Big), 
\end{align*}
where we use the fact $\boldsymbol{v}^T\boldsymbol{u} \le \frac{1}{2}(\boldsymbol{v}^T\boldsymbol{v} + \boldsymbol{u}^T\boldsymbol{u})$. %Note $\etab^T\Sigmab\etab = \suminner$. 
Then we need to upper bound $\boldy^T\boldsymbol{Q}\boldsymbol{\Sigma}\boldsymbol{Q}^T\boldy$. Following what we show in the proof of Theorem \ref{thm-eqvar01}, with probability at least $1-\delta$,
\begin{align*}
    \boldy^T\boldsymbol{Q}\boldsymbol{\Sigma}\boldsymbol{Q}^T\boldy = \text{Tr}\Big(\sum_{i=1}^p \lambda_i^2\boldz_i^T(\boldy\boldy^T)\boldz_i\Big) \le \sum_{i=1}^p \lambda_i^2\Vert \boldy\boldy^T \Vert_2\Vert\boldz_i \Vert_2^2 \le Cn\sum_{i=1}^p \lambda_i^2\Vert\boldz_i \Vert_2^2 \le Cn^2\normlbdtwo^2,
\end{align*}
where the last inequality follows the fact that $\sum_{i=1}^p \lambda_i^2\Vert\boldz_i \Vert_2^2$ is the weighted sum of sub-exponential variables. Next we lower bound the numerator $\etaavg^T\etab$, Lemma \ref{lem-ineqforu} gives with probability at least $1-\delta$,
\begin{align}
    \etaavg^T\etab &= \frac{1}{n}\boldy^T\boldy\normeta^2 + \frac{1}{n}\boldy^T\boldd \notag \ge \normeta^2 - C\sigma.\label{pf-posicoravg}
\end{align}
We need $\normeta^2 - C\sigma >0$ to guarantee $\etaavg^T\etab >0$ with high probability. Combining results above completes the proof.
\end{proof}

\section{Proof of Lemmas}
\label{sec-pflemma}
\subsection{Proof of Lemmas \ref{lem-posicorr}}
For Lemma \ref{lem-posicorr}, the proof of Theorem \ref{thm-eqvar01} gives
\begin{equation*}
    \etareg^T\boldsymbol{\eta} = \frac{s(\normeta^2 -t) +h^2 + h}{D},
\end{equation*}
for $D > 0$. Then we proceed by directly applying \eqref{eq-pfcorr02}.

%For Lemma \ref{lem-posicorriso}, the proof of Theorem \ref{thm-eqvariso} gives
%$\hat{\boldsymbol{\eta}}\cdot\boldsymbol{\eta} = \frac{s(\normeta^2 -t) +h^2 + h}{D}$,
%for $D > 0$. Then \eqref{eq-pfcorr03} completes the proof.

%For Lemma \ref{lem-posicorrbilevel}, from the proof of Theorem \ref{thm-bilevel01}, by \eqref{pf-bilevelnume}, we want $\eta_k - \frac{C_1n\lambda_k}{\tau + \binormlbd} - C_2\sqrt{\lambda_k} > 0$, this completes the proof.

%For Lemma \ref{lem-posicorravg}, in the proof of Theorem \ref{thm-avgest}, \eqref{pf-posicoravg} shows how to make itself positive.

\subsection{Proof of Lemma \ref{lem-xinverse}}
\label{pflemsec-xinverse}
Recall
\begin{equation*}
    \boldsymbol{X}\boldsymbol{X}^T + \tau\boldsymbol{I}=  \boldsymbol{Q}\boldsymbol{Q}^T+\tau\boldsymbol{I} +\normeta^2\boldsymbol{y}\boldsymbol{y}^T + \boldsymbol{Q}\boldsymbol{\eta}\boldsymbol{y}^T + \Big(\boldsymbol{Q}\boldsymbol{\eta}\boldsymbol{y}^T\Big)^{T} =  \boldsymbol{U}_{\tau} + \begin{bmatrix}\normeta\boldsymbol{y} &  \boldsymbol{d} & \boldsymbol{y}\end{bmatrix} \begin{bmatrix}
    \normeta\boldsymbol{y}^T\\ 
    \boldsymbol{y}^T\\
    \boldsymbol{d}^T
    \end{bmatrix}.
\end{equation*}
Thus, by Woodbury identity \citep{horn2012matrix}, $(\boldsymbol{X}\boldsymbol{X}^T)^{-1}$ can be expressed as:
\begin{equation}
\label{eq-pfxinverse02}
   \boldsymbol{U}_{\tau}^{-1} - \boldsymbol{U}_{\tau}^{-1}\begin{bmatrix}\normeta\boldsymbol{y}& \boldsymbol{d}& \boldsymbol{y}\end{bmatrix} \begin{bmatrix} \boldsymbol{I} + \begin{bmatrix}
    \normeta\boldsymbol{y}^T\\ 
    \boldsymbol{y}^T\\
    \boldsymbol{d}^T
    \end{bmatrix} \boldsymbol{U}_{\tau}^{-1} \begin{bmatrix}\normeta\boldsymbol{y} & \boldsymbol{d}& \boldsymbol{y}\end{bmatrix}\end{bmatrix}^{-1}\begin{bmatrix}
    \normeta\boldsymbol{y}^T\\ 
    \boldsymbol{y}^T\\
    \boldsymbol{d}^T
    \end{bmatrix}\boldsymbol{U}_{\tau}^{-1}.
\end{equation}
We first compute the inverse of the $3 \times 3$ matrix $\boldsymbol{A} := \bigg[ \boldsymbol{I} + \begin{bmatrix}
    \normeta\boldsymbol{y}^T\\ 
    \boldsymbol{y}^T\\
    \boldsymbol{d}^T
    \end{bmatrix} \boldsymbol{U}_{\tau}^{-1} \begin{bmatrix}\normeta\boldsymbol{y}& \boldsymbol{d}& \boldsymbol{y}\Big]\end{bmatrix}$. By our definitions of $s, h$ and $t$ in Section \ref{sec-keylemma}:
\begin{align*}
    \boldsymbol{A} = \begin{bmatrix}
    1+\normeta^2 s & \normeta h & \normeta s \\
    \normeta s & 1+h & s \\
    \normeta h & t & 1+h
    \end{bmatrix}.
\end{align*}    
Recalling $\boldsymbol{A}^{-1} = \frac{1}{\text{det}(\boldsymbol{A})}\text{adj}(\boldsymbol{A})$, where $\text{det}(\boldsymbol{A})$ is the determinant of $\boldsymbol{A}$ and $\text{adj}(\boldsymbol{A})$ is the adjoint of $\boldsymbol{A}$, it can be checked that:
\begin{align*}
    \text{det}(\boldsymbol{A}) = D = s(\normeta^2 - t) + (h+1)^2,
\end{align*}
and
\begin{align*}
    \text{adj}(\boldsymbol{A}) = \begin{bmatrix}
    (h+1)^2-st & \normeta(st-h-h^2) & -\normeta s \\
    -\normeta s & h+1+\normeta^2 s & -s \\
    \normeta(st-h-h^2) & \normeta^2h^2-t(1+\normeta^2s) & h+1+\normeta^2 s
    \end{bmatrix}.
\end{align*}
Combining the above gives
\begin{align*}
    \boldsymbol{y}^T(\boldsymbol{X}\boldsymbol{X}^T+ \tau\boldsymbol{I})^{-1} & = \boldsymbol{y}^T\boldutau^{-1} - \begin{bmatrix}\normeta s& h& s\end{bmatrix} \boldsymbol{A}^{-1}\begin{bmatrix}
    \normeta\boldsymbol{y}^T\\ 
    \boldsymbol{y}^T\\
    \boldsymbol{d}^T
    \end{bmatrix} \boldutau^{-1}\\
    & = \boldsymbol{y}^T\boldutau^{-1} - \frac{1}{D}\begin{bmatrix}\normeta s & h^2+h-st& s\end{bmatrix}\begin{bmatrix}
    \normeta\boldsymbol{y}^T\\ 
    \boldsymbol{y}^T\\
    \boldsymbol{d}^T
    \end{bmatrix} \boldutau^{-1}.
\end{align*}
This completes the proof of the lemma.

\subsection{Proof of Lemma \ref{lem-ineqforu} and Lemma \ref{lm-boundforf}}
To prove Lemma \ref{lem-ineqforu}, we need to bound the eigenvalues of $\boldutau$. Recall $\boldu_0 = \boldsymbol{Q}\boldsymbol{Q}^T = \sum_{i=1}^{p}\lambda_i\boldz_i\boldz_i^T$, where $\boldz_i \in \mathbb{R}^n$ are independent vectors with IID standard normal elements. Let $\lambda_k(\boldsymbol{M})$ represent the $k$-th eigenvalue of matrix $\boldsymbol{M}$. We start from \citet[Lemma 5 (3)]{bartlett2020benign}:
\begin{lemma}
\label{lem-bartletteigen}
There are constants $b, c \ge 1$ such that, for any $k \ge 0$, with probability at least $1- 2e^{-n/c}$, if $\frac{\sum_{i>k}\lambda_i}{\lambda_{k+1}} \ge bn$, then
\begin{align*}
    \frac{1}{c}{\sum_{i>k}\lambda_i} \le \lambda_{n}(\sum_{i>k}^{p}\lambda_i\boldz_i\boldz_i^T) \le \lambda_{1}(\sum_{i>k}^{p}\lambda_i\boldz_i\boldz_i^T) \le c{\sum_{i>k}\lambda_i}.
\end{align*}
\end{lemma}
First note that the balanced ensemble requirement $bn\lambda_1 \le \binormlbd$ implies $bn\lambda_1 \le \normlbd$. We can then obtain the bounds for eigenvalues of $\boldu_0$ by letting $k = 0$ in Lemma \ref{lem-bartletteigen}. Then the eigenvalues of $\boldutau$ are bounded as follows.
\begin{lemma}
\label{lm-eigU}
Assume the balanced $\Sigmab$ assumption is satisfied. Suppose that $\delta < 1$ with $\log(1/\delta) < n/c$ for some $c > 1$. There is a constant $C > 1$ such that with probability at least $1-\delta$, the largest and smallest eigenvalues of $\boldutau$ satisfy:
\begin{equation}
\label{eq-pfeigu}
    \frac{1}{C}(\tau+\sum_{i=1}^p\lambda_i)\le\tau+\frac{1}{C}\sum_{i=1}^p\lambda_i \le \lambda_n(\boldutau) \le \lambda_1(\boldutau) \le \tau + C\sum_{i=1}^p\lambda_i \le C(\tau + \sum_{i=1}^p\lambda_i).
\end{equation}
\end{lemma}
Now we are ready to prove Lemma \ref{lem-ineqforu}.

\subsubsection{Bounds for s}
For $s =\boldsymbol{y}^T \boldutau^{-1}\boldsymbol{y}$, from \eqref{eq-pfeigu} and $ \Vert \boldsymbol{y} \Vert_2^2 =n$, the variational characterization of eigenvalues gives:
\begin{align*}
    s =\boldsymbol{y}^T \boldutau^{-1}\boldsymbol{y} \le \Vert \boldsymbol{y} \Vert_2^2 \lambda_1(\boldutau^{-1}) \le n \frac{1}{\lambda_n(\boldutau)} \le C_1\frac{n}{\tau+\normlbd}.
\end{align*}
The lower bound can be derived in a similar way and is omitted for brevity.

\subsubsection{Bounds for t and h}
\label{secpf-boundforth}
We begin by presenting the definitions of sub-Gaussian and sub-exponential norms. For a detailed discussion of sub-Gaussian and sub-exponential variables, we refer the readers to \citet[Chapter 2]{vershynin2018high}.
\begin{definition}
For a sub-Gaussian variable $X$ defined in \citet[2.5]{vershynin2018high}, the sub-Gaussian norm of $X$, denoted by $\Vert X \Vert_{\psi_2}$, is defined as
\begin{align*}
    \Vert X \Vert_{\psi_2} = \inf\{t>0: \mathbb{E}[e^{X^2/t^2}] < 2\}.
\end{align*}
\end{definition}
Then \citet[Example 2.5.8 (a)]{vershynin2018high} states that if $X\sim \Nn(0, \sigma^2)$, then $X$ is sub-Gaussian with $\Vert X \Vert_{\psi_2} < C\sigma$, where $C$ is an absolute constant.
\begin{definition}
For a sub-exponential variable $X$ defined in \citet[2.7]{vershynin2018high}, the sub-exponential norm of $X$, denoted by $\Vert X \Vert_{\psi_1}$, is defined as
\begin{align*}
    \Vert X \Vert_{\psi_1} = \inf\{t>0: \mathbb{E}[e^{|X|/t}] < 2\}.
\end{align*}
\end{definition}
\citet[Lemma 2.7.6]{vershynin2018high} shows that sub-exponential is sub-Gaussian squared.
\begin{lemma}
A random variable $X$ is sub-Gaussian if and only if $X^2$ is sub-exponential. Moreover,
\begin{align*}
    \Vert X^2 \Vert_{\psi_1} = \Vert X \Vert_{\psi_2}^2.
\end{align*}
\end{lemma}
We now look at $\Vert \boldd \Vert_2$. Recall $\boldd = \boldsymbol{Q}\etab = \boldsymbol{Z}\boldsymbol{\Lambda}^{\frac{1}{2}}\betab$. $\Vert \boldd \Vert_2^2 = \sum_{j=1}^n{d_j}^2$, where $d_j = \sum_{i=1}^p\sqrt{\lambda_i}\beta_iz_{ji}$ and $z_{ji}$'s are IID standard Gaussian variable. Hence $d_j$ is Gaussian with mean zero and variance $\sum_{i=1}^p\lambda_i\beta_i^2$ and ${d_j}^2$ is sub-exponential with  $\Vert d_j^2 \Vert_{\psi_1} < c\suminner$ and mean $\suminner$. To bound $\Vert \boldd \Vert_2$, we need the Bernstein's inequality \citep[Theorem 2.8.2]{vershynin2018high}:
\begin{lemma}
Let $\xi_1,...,\xi_n$ be independent, mean zero, sub-exponential random variables with sub-exponential norm $\Vert \xi \Vert_{\psi_1}$, and $a = (a_1,...,a_n) \in \mathbb{R}^n$. Then for every $t \ge 0$, we have
\begin{align*}
    \mathbb{P}\Big(|\sum_{i=1}^n a_i \xi_i| \ge t\Big) \le 2 \exp\Big\{-c\min\Big(\frac{t^2}{\Vert \xi \Vert_{\psi_1}^2\cdot\sum_{i=1}^n a_i^2}, \frac{t}{\Vert \xi \Vert_{\psi_1}\cdot\max_{i\in[n]}|a_i|}\Big)\Big\}.
\end{align*}
\end{lemma}
\begin{corollary}
\label{pf-lemmasubexp}
Suppose $\{a_i\}$ is a non-increasing sequence of non-negative numbers such that $\sum_{i} a_i < \infty$. Then there is a constant $c$ such that for any sequence of independent, zero-mean sub-exponential random variables $\{\xi_i\}$ with sub-exponential norm $\Vert \xi \Vert_{\psi_1}$, and any $x>0$, with probability at least $1-2e^{-x}$,
\begin{align*}
    |\sum_{i=1} a_i \xi_i| \le c\Vert \xi \Vert_{\psi_1}\cdot\max\Big(a_1x, \sqrt{x\sum_{i}a_i^2}\Big).
\end{align*}
\end{corollary}
Let fix the length of the sequence as $n$ and let $a_i = 1$, for $i \in [n]$. Then combing the inequality above with $x = n/c$ and the fact that ${d_j}^2$'s are sub-exponential gives with probability at least $1-2e^{-\frac{n}{c}}$,
\begin{align}
\label{pf-boundofd}
    \Vert \boldd \Vert_2 \le C\sqrt{n\suminner} = C\sqrt{n}\sigma.
\end{align}
Recall $t=\boldsymbol{d}^T\boldutau^{-1}{\boldd}$ and $h=\boldsymbol{y}^T \boldutau^{-1}\boldsymbol{d}$, we can obtain the upper and lower bounds of $t$ by the variational characterization of eigenvalues. The bounds of $h$ can be derived from the fact $-\Vert \boldsymbol{d} \Vert_2\Vert \boldsymbol{y} \Vert_2\Vert \boldsymbol{U}_{\tau}^{-1} \Vert_2 \le h \le \Vert \boldsymbol{d} \Vert_2\Vert \boldsymbol{y} \Vert_2\Vert \boldsymbol{U}_{\tau}^{-1} \Vert_2$. The bounds for $\Vert \boldy^T\boldutau^{-1} \Vert_2$ and $\Vert \boldd^T\boldutau^{-1} \Vert_2$ can be obtained from Cauchy-Schwarz for matrices.

\subsubsection{Proof of Lemma \ref{lm-boundforf}}\label{pfsec-pfboundforf}
Now we prove Lemma \ref{lm-boundforf}. Recall
\begin{align*}
    f_i = \boldsymbol{e}_i^T\boldu_0^{-1}\boldsymbol{d} &= \boldsymbol{e_i}^T(\frac{1}{\normlbd}\boldsymbol{I} - \boldsymbol{E}^{'})\boldsymbol{d},
\end{align*}
thus,
\begin{align*}
    \max_{i \in [n]}|f_i| &\le \frac{1}{\normlbd}\Vert \boldd \Vert_{\infty} + \Vert \boldsymbol{e}_i^T\boldsymbol{E}^{'}\boldd \Vert_\infty\\
    & \le \frac{1}{\normlbd}\Vert \boldd \Vert_{\infty} + \Vert \boldsymbol{e}_i^T\boldsymbol{E}^{'}\boldd \Vert_2,
\end{align*}
where the last equality comes from the fact that the $\ell_2$ norm of a vector won't be smaller than its infinity norm. By Markov's inequality \citep[2.1.1]{wainwright2019high}, for sufficiently large constant $C >1$,
\begin{align*}
\mathbb{P}(\max_{i\in [n]}|d_i| \ge C \mathbb{E}[\max_{i\in [n]}|d_i|])   \le \delta.  
\end{align*}
Thus it suffices to bound $\mathbb{E}[\max_{i\in [n]}|d_i|]$. We know that the elements of $\boldd$ are IID zero-mean Gaussian variables with variance $\suminner$. By \citet[Exercise 2.11]{wainwright2019high},
\begin{align*}
    \mathbb{E}[\max_{i\in [n]}|d_i|] \le \sqrt{\suminner}\sqrt{2\log(2n)} = \sqrt{2\log(2n)}\sigma.
\end{align*}
Thus with probability at least $1-\delta$, 
\begin{align*}
    \Vert \boldd \Vert_\infty \le C\sqrt{2\log(2n)}\sigma.
\end{align*}
To bound $\Vert \boldsymbol{e}_i^T\boldsymbol{E}^{'}\boldd \Vert_2$, using $\boldsymbol{v}^T\boldsymbol{M}\boldsymbol{u} \le \Vert \boldsymbol{v} \Vert_2\Vert \boldsymbol{u} \Vert_2\Vert \boldsymbol{M} \Vert_2$ and the bound on $\Vert \boldd \Vert_2$ in \eqref{pf-boundofd} and the bound on $\Vert \boldsymbol{E}^{'} \Vert_{2}$ in \eqref{eq-linkerror} give, for $n > c/\delta$ and for every $i \in [n]$,
\begin{align*}
    \Vert \boldsymbol{e}_i^T\boldsymbol{E}^{'}\boldd \Vert_2 &\le \Vert \boldsymbol{e}_i \Vert_2\Vert \boldd \Vert_2\Vert \boldsymbol{E}^{'} \Vert_2\\
    &\le \frac{C_1\sigma}{\normlbd}.
\end{align*}
Combining results above completes the proof.

\subsection{Proof of Lemma \ref{lem-ineqforubilevel}}\label{pfsec-pfutaunoone}
To prove Lemma \ref{lem-ineqforubilevel}, the first step is to separate the largest eigenvalue from others. Specifically, by Woodbury identity, $\boldutau^{-1}$ can be expressed as
\begin{align}
    \boldutau^{-1} &= (\tau\boldsymbol{I}+\sum_{i=2}^p\lambda_i\boldz_i\boldz_i^T + \lambda_1\boldz_1\boldz_1^T)^{-1}\\
    &= \boldutaunoone^{-1} - \frac{\lambda_1\boldutaunoone^{-1}\boldz_1\boldz_1^T\boldutaunoone^{-1}}{1+\lambda_1\boldz_1^T\boldutaunoone^{-1}\boldz_1},\label{pf-invunoone}
\end{align}
where $\boldutaunoone = \tau\boldsymbol{I} + \sum_{i=2}^p\lambda_i\boldz_i\boldz_i^T$.
By Lemma \ref{lm-eigU} above, with probability at least $1-\delta$,
\begin{equation}
\label{eq-pfeigunoone}
    \frac{1}{C}(\tau+\sum_{i=2}^p\lambda_i) \le \lambda_n(\boldutaunoone) \le \lambda_1(\boldutaunoone) \le C(\tau + \sum_{i=2}^p\lambda_i).
\end{equation}
Then we need to bound $\Vert \boldz_1 \Vert_2$ and $\Vert \boldz_k \Vert_2$. In Lemma \ref{pf-lemmasubexp}, let $x< \frac{n}{c_0}$ with sufficiently large $c_0$, if $n > C_0\log(1/\delta)$ for some $C_0 >1$, then there exist $C_1, C_2 > 1$ such that with probability at least $1-\delta$,
\begin{align*}
    \frac{1}{C_1}n \le \Vert \boldz_i \Vert_2^2 \le C_2n, \ \ \ i\in[p].
\end{align*}
Now we are ready to derive the bounds in Lemma \ref{lem-ineqforubilevel}. 

For $s=\boldsymbol{y}^T \boldutau^{-1}\boldsymbol{y}$, by \eqref{pf-invunoone} and \eqref{eq-pfeigunoone} and using the variational characterization of eigenvalues and $\boldsymbol{v}^T\boldsymbol{M}\boldsymbol{u} \le \Vert \boldsymbol{v} \Vert_2\Vert \boldsymbol{u} \Vert_2\Vert \boldsymbol{M} \Vert_2$, with probability at least $1-\delta$,
\begin{align*}
    s &= \frac{\boldy^T\boldutaunoone^{-1}\boldy + \lambda_1\boldz_1^T\boldutaunoone^{-1}\boldz_1\boldy^T\boldutaunoone^{-1}\boldy - \lambda_1\boldy^T\boldutaunoone^{-1}\boldz_1\boldz_1^T\boldutaunoone^{-1}\boldy}{1+\lambda_1\boldz_1^T\boldutaunoone^{-1}\boldz_1}\\
    &\le \frac{\frac{C_1n}{(\tau+\binormlbd)}\Big(1+\frac{C_2n\lambda_1}{\tau+\binormlbd}\Big)}{1+\frac{n\lambda_1}{C_3(\tau+\binormlbd)}}\\
    &\le \frac{C_1n}{(\tau + \binormlbd)}\cdot\Big(\frac{\tau+\binormlbd+C_2n\lambda_1}{\tau+\binormlbd}\Big)\cdot\Big(\frac{C_3(\tau+\binormlbd)}{C_3(\tau+\binormlbd) + n\lambda_1}\Big)\\
    &\le \frac{C_1n}{(\tau + \binormlbd)}\cdot\Big(\frac{C_2C_3(\tau+\binormlbd)+C_2n\lambda_1}{\tau+\binormlbd}\Big)\cdot\Big(\frac{C_3(\tau+\binormlbd)}{C_3(\tau+\binormlbd) + n\lambda_1}\Big)\\
    &\le \frac{C_4n}{\tau + \binormlbd}.
\end{align*}
For the lower bound of $s$, we need to show $\boldz_1^T\boldutaunoone^{-1}\boldy$ is sufficiently small compared with $\boldz_1^T\boldutaunoone^{-1}\boldz_1$ and $\boldy^T\boldutaunoone^{-1}\boldy$. We thus need the following Hanson-Wright inequality \citep{rudelson2013hanson}.
\begin{lemma}
\label{pf-hanson}
Let $\boldz$ be a random vector whose elements are IID zero-mean sub-Gaussian random variable with parameter at most 1. Then , there exists universal constant $c>0$ such that for any positive semi-definite matrix $\boldsymbol{M}$ and for every $t>0$, we have
\begin{align*}
   P\Big(|\boldz^T\boldsymbol{M}\boldz - \mathbb{E}[\boldz^T\boldsymbol{M}\boldz ]| > t\Big) \le \exp\bigg\{-c\min\Big\{\frac{t^2}{\Vert \boldsymbol{M}\Vert_F^2},\frac{t}{\Vert \boldsymbol{M}\Vert_2} \Big\}\bigg\}. 
\end{align*}
\end{lemma}
Note $\Vert \boldsymbol{M}\Vert_F^2 \le n\Vert \boldsymbol{M}\Vert_2^2$ and let $t = \frac{1}{C_0}n\Vert \boldsymbol{M}\Vert_2$ for sufficiently large constant $C_0$ to get with probability at least $1-2e^{-\frac{n}{c_1}}$,
\begin{align}
    |\boldz^T\boldsymbol{M}\boldz - \mathbb{E}[\boldz^T\boldsymbol{M}\boldz ]| \le \frac{1}{C_0}n\Vert \boldsymbol{M}\Vert_2.
\end{align}
Then we use the similar trick as \citet[D.3.1]{muthukumar2020classification} and apply the parallelogram law to $\boldz_1^T\boldutaunoone^{-1}\boldy$,
\begin{align*}
    \boldz_1^T\boldutaunoone^{-1}\boldy = \frac{1}{4}((\boldz_1+\boldy)^T\boldutaunoone^{-1}(\boldz_1+\boldy) - (\boldz_1-\boldy)^T\boldutaunoone^{-1}(\boldz_1-\boldy)).
\end{align*}
To use the Hanson-Wright inequality, we need to calculate the conditional expectation
\begin{align*}
    \mathbb{E}[\boldz_1^T\boldutaunoone^{-1}\boldy|\boldutaunoone^{-1}] = \mathbb{E}[\text{Tr}(\boldutaunoone^{-1}\boldy\boldz_1^T)|\boldutaunoone^{-1}] = \text{Tr}(\boldutaunoone^{-1}\mathbb{E}[\boldy\boldz_1^T]),
\end{align*}
where we use the fact that $\boldy$ and $\boldz_1$ are independent of $\boldutaunoone^{-1}$.
It is not hard to check that $\mathbb{E}[\boldy\boldz_1^T] = \boldsymbol{0}$, where $\boldsymbol{0}$ is the matrix with all elements $0$. Now applying Lemma \ref{pf-hanson} to both $(\boldz_1+\boldy)^T\boldutaunoone^{-1}(\boldz_1+\boldy)$ and $(\boldz_1-\boldy)^T\boldutaunoone^{-1}(\boldz_1-\boldy)$
gives with probability at least $1-2e^{-\frac{n}{c_1}}$,
\begin{align*}
    |\boldz_1^T\boldutaunoone^{-1}\boldy| \le \frac{2}{C_0}n\Vert \boldutaunoone^{-1} \Vert_2.
\end{align*}
Now for the numerator of $s$, using the bound of eigenvalues of $\boldutaunoone$ in \eqref{eq-pfeigunoone} and the fact that $C_0$ is sufficiently large gives with probability at least $1-\delta$,
\begin{align*}
    \lambda_1\boldz_1^T\boldutaunoone^{-1}\boldz_1\boldy^T\boldutaunoone^{-1}\boldy - \lambda_1\boldy^T\boldutaunoone^{-1}\boldz_1\boldz_1^T\boldutaunoone^{-1}\boldy \ge \frac{n^2\lambda_1}{C_2(\tau+\binormlbd)^2},
\end{align*}
for some large $C_2$.
Therefore,
\begin{align*}
    s &\ge \frac{\frac{n}{C_1(\tau+\binormlbd)}\Big(1+\frac{n\lambda_1}{C_2(\tau+\binormlbd)}\Big)}{1+\frac{C_3n\lambda_1}{(\tau+\binormlbd)}}\\
    &\ge \frac{n}{C_1(\tau + \binormlbd)}\cdot\Big(\frac{C_2(\tau+\binormlbd)+n\lambda_1}{C_2(\tau+\binormlbd)}\Big)\cdot\Big(\frac{(\tau+\binormlbd)}{(\tau+\binormlbd) + C_3n\lambda_1}\Big)\\
    &\ge \frac{n}{C_1(\tau + \binormlbd)}\cdot\Big(\frac{C_2(\tau+\binormlbd)+n\lambda_1}{C_2(\tau+\binormlbd)}\Big)\cdot\Big(\frac{(\tau+\binormlbd)}{(C_2C_3(\tau+\binormlbd) + C_3n\lambda_1}\Big)\\
    &\ge \frac{n}{C_4(\tau + \binormlbd)}.
\end{align*}
The derivation of bounds for $t_k$ is the same as the procedure above. 

For $f_k$, with probability at least $1-\delta$,
\begin{align*}
    |f_k| &= |\frac{\boldy^T\boldutaunoone^{-1}\boldz_k + \lambda_1\boldy^T\boldutaunoone^{-1}\boldz_k\boldz_1^T\boldutaunoone^{-1}\boldz_1 - \lambda_1\boldy^T\boldutaunoone^{-1}\boldz_1\boldz_1^T\boldutaunoone^{-1}\boldz_k}{1+\lambda_1\boldz_1^T\boldutaunoone^{-1}\boldz_1}|\\
    &\le \frac{\frac{C_1n}{(\tau+\binormlbd)}\Big(1+\frac{C_2n\lambda_1}{\tau+\binormlbd}\Big)}{1+\frac{n\lambda_1}{C_3(\tau+\binormlbd)}}\\
    &\le \frac{C_4n}{\tau + \binormlbd}.
\end{align*}
Similarly we can obtain upper bounds for $\Vert \boldy^T\boldutau^{-1} \Vert_2$ and $\Vert \boldz_k^T\boldutau^{-1} \Vert_2$. 

For $f_1$,
\begin{align*}
    |f_1| &= |\frac{\boldy^T\boldutaunoone^{-1}\boldz_1 + \lambda_1\boldz_1^T\boldutaunoone^{-1}\boldz_1\boldy^T\boldutaunoone^{-1}\boldz_1 - \lambda_1\boldy^T\boldutaunoone^{-1}\boldz_1\boldz_1^T\boldutaunoone^{-1}\boldz_1}{1+\lambda_1\boldz_1^T\boldutaunoone^{-1}\boldz_1}|\\
    &=|\frac{\boldy^T\boldutaunoone^{-1}\boldz_1}{1+\lambda_1\boldz_1^T\boldutaunoone^{-1}\boldz_1}|\\
    &\le \frac{\frac{C_1n}{(\tau+\binormlbd)}}{1+\frac{n\lambda_1}{C_2(\tau+\binormlbd)}}\\
    &\le \frac{C_1n}{(\tau + \binormlbd)}\cdot\Big(\frac{C_2(\tau+\binormlbd)}{C_2(\tau+\binormlbd) + n\lambda_1}\Big)\\
    &\le \frac{C_3n}{\tau + \binormlbd + n\lambda_1}.
\end{align*}
For $g_1$, we have
\begin{align*}
    |g_1| &= |\frac{\boldz_k^T\boldutaunoone^{-1}\boldz_1 + \lambda_1\boldz_1^T\boldutaunoone^{-1}\boldz_1\boldz_k^T\boldutaunoone^{-1}\boldz_1 - \lambda_1\boldz_k^T\boldutaunoone^{-1}\boldz_1\boldz_1^T\boldutaunoone^{-1}\boldz_1}{1+\lambda_1\boldz_1^T\boldutaunoone^{-1}\boldz_1}|\\
    &=|\frac{\boldz_k^T\boldutaunoone^{-1}\boldz_1}{1+\lambda_1\boldz_1^T\boldutaunoone^{-1}\boldz_1}|\\
    &\le \frac{\frac{C_1n}{(\tau+\binormlbd)}}{1+\frac{n\lambda_1}{C_2(\tau+\binormlbd)}}\\
    &\le \frac{C_1n}{(\tau + \binormlbd)}\cdot\Big(\frac{C_2(\tau+\binormlbd)}{C_2(\tau+\binormlbd) + n\lambda_1}\Big)\\
    &\le \frac{C_3n}{\tau + \binormlbd + n\lambda_1}.
\end{align*}
This completes the proof.

\section{Proofs for Section \ref{sec-labelnoise}}\label{sec-labelnoise-proofs}
The proofs follow similar conceptual steps to the noiseless case, but several  technical adjustments are needed. This is because: 
%The reason is that the data matrix $\boldsymbol{X} = \boldy\etab^T + \boldqcap$ where $\boldy$ is the clean label vector. $\tildeta$, however, is obtained by training the noisy label vector $\tildy$. 
on the on hand, the clean label vector $\boldy$ enters the features equation $\boldsymbol{X} = \boldy\etab^T + \boldqcap$; on the other had, the estimator $\hat\etab$ is generated according to the noisy label vector $\tildy$. We start from defining some additional primitive quadratic forms on $\boldu_0 = \boldqcap\boldqcap^T$:
\begin{align}
    \tilds &= \tildy^T\boldu_0^{-1}\boldy,\notag\\
    \tildh &= \tildy^T\boldu_0^{-1}\boldd,\notag\\
    \tildg &= \tildy^T\boldu_0^{-1}\boldsymbol{e}_i, \ \ i \in [n],\notag\\
    \tildtilds &= \tildy^T\boldu_0^{-1}\tildy, \label{pfeq-quardnoise}
\end{align}
%where recall that $\boldu_0 = \boldqcap\boldqcap^T$.
The subscript $c$ here emphasizes that the corrupted noise vector enters these quantities (unlike the corresponding ones in Appendix \ref{sec-keylemma}). The lemma below is our analogue to Lemma \ref{lem-xinverse}.
%, but has a more complex form because $\boldsymbol{X}$ contains $\boldy$ not $\tildy$ appearing in the LHS of \eqref{eq-xinverse}.
\begin{lemma}
\label{lem-xinversenoise}
Recall $D := s(\normeta^2 - t) + (h+1)^2$, then
\begin{align}
\label{eq-dummy}
    &\tildy^T(\boldsymbol{X}\boldsymbol{X}^T)^{-1} = \boldsymbol{y}_c^T\boldu_0^{-1} 
    - \frac{1}{D}\boldsymbol{v}\begin{bmatrix}
    \normeta\boldsymbol{y}^T\\ 
    \boldsymbol{y}^T\\
    \boldsymbol{d}^T
    \end{bmatrix} \boldu_0^{-1}, \notag
\end{align}
where
\begin{align*}
    \boldsymbol{v} = \Big[\normeta \tilds + \normeta(\tilds h - s\tildh), \tildh h+\tildh-\tilds t - \normeta^2(\tilds h - s\tildh), \tilds + \tilds h - s\tildh\Big].
\end{align*}
\end{lemma}
Next, the lemma below gives upper/lower bounds for the newly defined quadratic forms in \eqref{pfeq-quardnoise}.
\begin{lemma}
\label{lem-ineqforunoise}
Assume $\Sigmab = \boldsymbol{I}$ and $p > Cn\log n + n +1$ for a sufficiently large constant $C$. Fix $\delta\in(0,1)$ and suppose $n$ is large enough such that $n>c/\delta$ for some $c>1$. Then, there exist constants $C_i$'s $>1$ such that with probability at least $1-\delta$, the following results hold:
\begin{align*}
    \frac{n}{C_1p} \le &\tilds \le \frac{C_1n}{p}, \\
    -\frac{C_2n\normeta}{p} \le &\tildh \le \frac{C_2n\normeta}{p}, \\
    \frac{n}{C_3p} \le &\tildtilds \le \frac{C_3n}{p}.
\end{align*}
\end{lemma}
Note that the bounds for the quadratic forms above are of the same order as those for the corresponding quadratic forms defined with $\boldy$, e.g., both $s$ and $\tilds$ are at the order of $\Theta(n/p)$. Now we are ready to prove the theorems.

\subsection{Proof of Theorem \ref{thm-linkisonoise}}
Similar to the proofs in Appendix \ref{pf-linktwo}, we again start from the duality argument of \citet{muthukumar2020classification} and so we need to find the conditions ensuring
\begin{align}
     {y}_{c,i}\tildy^T(\boldsymbol{X}\boldsymbol{X}^{T})^{-1} \boldsymbol{e}_i > 0, \ \text{for all} \ i \in [n],
\end{align}
where $y_{c,i}$ is the $i$-th element of $\tildy$.
Lemma \ref{lem-xinversenoise} and some algebra steps give:
\begin{align}
    &\tildy^T(\boldsymbol{X}\boldsymbol{X}^{T})^{-1} \boldsymbol{e}_i \notag = \tildg\\
    &- \frac{1}{D}\Big[\normeta \tilds + \normeta(\tilds h - s\tildh), \tildh h+\tildh-\tilds t - \normeta^2(\tilds h - s\tildh), \tilds + \tilds h - s\tildh\Big]\begin{bmatrix}
    \normeta g_i\\ 
    g_i\\
    f_i
    \end{bmatrix} \notag \\
    &\quad\qquad\qquad\qquad= \frac{A + B}{s(\normeta^2-t)+(h+1)^2},\notag %\label{eq-boundnoise01}
\end{align}
where
\begin{align*}
    &A =  \tildg+ 2\tildg h-\tilds f_i\\
    &B =  \normeta^2(\tildg s - g_i\tilds) +g_i\tilds t - \tildg st + \tildg h^2 - g_i\tildh h - g_i\tildh + sf_i\tildh - \tilds hf_i.
\end{align*}
Let us start with an observation regarding the numerator $A + B$. We have already derived  conditions making $y_{c,i}A > 0$ in Appendix \ref{pfsec-linkiso} (to be precise, Appendix \ref{pfsec-linkiso} considers $y_i(g_i + g_ih - sf_i)$, but the quadratic forms are of the same order, so the same results apply). Specifically, when showing $y_{c,i}A > 0$, we first have 
$$y_{c,i}\tildg > 1/(Cp) > 0$$ 
with high probability (obtained by Lemma \ref{lm-linkiso03}). Then, in Appendix \ref{pfsec-linkiso}, we show that the rest of the terms in $A$, i.e., $|\tildg h|, |\tilds f_i|$, are sufficiently small compared to $1/(Cp)$. Note that when there is no label noise, i.e., $\gamma = 0$, then $\tildg = g_i, \tilds = s$, $\tildh = h$ and $A + B = g_i + g_ih -sf_i$, which becomes the same as what we have in Appendix \ref{pfsec-linkiso}.

Now, in order to derive conditions under which $y_{c,i}(A + B) > 0$, we  first decompose
\begin{align*}
    A+B = \tildg + A_h - A_f + A_s,
\end{align*}
where 
\begin{align*}
    &A_h = 2\tildg h+\tildg h^2 - g_i\tildh h - g_i\tildh\\
    &A_f = \tilds f_i - sf_i\tildh + \tilds hf_i\\
    &A_s = \tildg\normeta^2s - g_i\normeta^2\tilds +g_i\tilds t - \tildg st.
\end{align*}
The idea is to that show that: (a) in $A_h$, the term $\tildg h$ is dominant; (b) in $A_f$, $\tilds f_i$ is the dominant term; (c) $|A_s|$ is sufficiently smaller than $1/(Cp)$. To achieve this, we need 
\begin{align}
\label{pfeq-linkpfnoise}
    p > C_0\max\{n\sqrt{\log(2n)}\normeta, n\normeta^2\}
\end{align}
for a sufficiently large constant $C_0$. The reason is that in $A_h$ and $A_f$, $|h|$ (and $|\tildh|$) is upper bounded by $O(n\normeta/p)$ with high probability. In $A_s$, $s_c\normeta^2 \le O(n\normeta^2/p)$ and $st \le O(n^2\normeta^2/p^2)$ with high probability. Therefore, \eqref{pfeq-linkpfnoise} ensures the terms mentioned above are sufficiently smaller than $1$ as desired.

%To do that, we first consider terms $\tildg_ih, \tildg_ih^2, g_i\tildh h, g_i\tildh$. Since the bounds for $|h|$ and $|\tildh|$ are both at the order of $O(n\normeta/p)$ and we require $p > C\max\{n\sqrt{\log(2n)}\normeta, n\normeta^2\}$ in Theorem \ref{thm-linkisonoise}, the dominant term is $\tildg_ih$ and it is enough to only consider this dominant term. Similarly, it is enough to look at $\tilds f_i$, the dominant term of $\tilds f_i, sf_i\tildh, \tilds hf_i$. We have shown in Appendix \ref{pfsec-linkiso} that with high probability, $\tildy_i(\tildg_i+ 2\tildg_ih-\tilds f_i) > 1/(Cp)$ provided the conditions in Theorem \ref{thm-linkisonoise} (Appendix \ref{pf-linktwo} is for $y_i(g_i + g_ih - sf_i)$, but $s$ and $\tilde{s}$ are of the same order, so the same results apply). We next need to show $\normeta^2(\tildg_is - g_i\tilds) +g_i\tilds t - \tildg_is t \le 1/C_1p$ for large constant $C_1$. This is true when assuming , since Lemmas \ref{lem-ineqforu} and \ref{lem-ineqforunoise} give with high probability, $s\normeta^2 \le C_2n\normeta^2/p$ and $\tilds t \le C_3(n/p)(n\normeta^2/p)$.

\subsection{Proofs of Theorem \ref{thm-eqvarisonoise} and Corollary \ref{cor-linkandtesternoise}}
Again, similar to the proofs in Appendix \ref{pfsec-erroriso}, we need to lower bound the ratio
\begin{equation}
\label{pf-errornoise02}
\frac{(\tildy^T(\boldsymbol{X}\boldsymbol{X}^T)^{-1}\boldsymbol{X} \boldsymbol{\eta})^2}{ \tildy^T(\boldsymbol{X}\boldsymbol{X}^T)^{-1}\tildy}.
\end{equation}
Here we will lower bound $\tildy^T(\boldsymbol{X}\boldsymbol{X}^T)^{-1}\boldsymbol{X} \boldsymbol{\eta}$ and upper bound $\tildy^T(\boldsymbol{X}\boldsymbol{X}^T)^{-1}\tildy$. Lemma \ref{lem-xinversenoise} and some algebra steps give:
\begin{align*}
\tildy^T(\boldsymbol{X}\boldsymbol{X}^T)^{-1}\boldsymbol{X} \boldsymbol{\eta} = \normeta^2\tildy^T(\boldsymbol{X}\boldsymbol{X}^T)^{-1}\boldsymbol{y} + \tildy^T(\boldsymbol{X}\boldsymbol{X}^T)^{-1}\boldsymbol{Q}\boldsymbol{\eta} =\frac{\tilds\normeta^2-\tilds t+\tildh h+\tildh}{s(\normeta^2 -t) + (h+1)^2}.
\end{align*}
Similarly, the denominator $\tildy^T(\boldsymbol{X}\boldsymbol{X}^T)^{-1}\tildy$ is
\begin{align*}
      \frac{\tildtilds + \normeta^2(\tildtilds s - \tilds^2) + \tilds^2t - \tildtilds st + 2\tildtilds h + \tildtilds h^2 + s\tildh h - 2\tildtilds \tildh h - 2\tilds \tildh}{s(\normeta^2 -t) + (h+1)^2}.
\end{align*}
Combining the two expressions above gives that we need to lower bound:
\begin{align}\label{pfnoise-ratiods01}
    \frac{\Big(\tilds\normeta^2-\tilds t+\tildh h+\tildh\Big)^2}{D_s\Big(s(\normeta^2 -t) + (h+1)^2\Big)},
\end{align}
where
$$D_s = \tildtilds + \normeta^2(\tildtilds s - \tilds^2) + \tilds^2t - \tildtilds st + 2\tildtilds h + \tildtilds h^2 + s\tildh h - 2\tildtilds \tildh h - 2\tilds \tildh.$$
Recall that in Appendix \ref{pfsec-erroriso}, we have lower bounded $\frac{\Big(s(\normeta^2 -t) +h^2 + h\Big)^2}{s \Big(s(\normeta^2 -t) + (h+1)^2 \Big)}$ and since $\tildtilds, \tilds, s$ are of the same order and $\tildh, h$ are also of the same order, we actually have the same bound for $ \frac{\Big(\tilds\normeta^2-\tilds t+\tildh h+\tildh\Big)^2}{s(\normeta^2 -t) + (h+1)^2}$ in \eqref{pfnoise-ratiods01}. The next step is to show that $D_s$ is close to $s_{cc}$. This is true since due to the assumption $p > C\max\{n\sqrt{\log(2n)}\normeta, n\normeta^2\}$ for a large constant $C$, the bounds for terms such as $\normeta^2s, st, h^2, \tildh$ are sufficiently small compared to $1$ (we also illustrate this under \eqref{pfeq-linkpfnoise}). Therefore, in $D_s$, $\tildtilds$ is the dominant term and we finally need to lower bound the term
\begin{align*}
    \frac{\Big(\tilds\normeta^2-\tilds t+\tildh h+\tildh\Big)^2}{\tildtilds\Big(s(\normeta^2 -t) + (h+1)^2\Big)}.
\end{align*}
This satisfies the same bound as $\frac{\Big(s(\normeta^2 -t) +h^2 + h\Big)^2}{s \Big(s(\normeta^2 -t) + (h+1)^2 \Big)}$ in Appendix \ref{pfsec-erroriso}. Since $p > Cn\normeta^2$ falls into the low-SNR regime in Corollary \ref{cor-eqvariso}, we can directly apply the results of low-SNR regime in Corollaries \ref{cor-eqvariso} and \ref{cor-linkandtester}, which gives the desired.

\subsection{Proofs of auxiliary lemmas}
We first prove Lemma \ref{lem-xinversenoise}.
\begin{proof}[Proof of Lemma \ref{lem-xinversenoise}]
The proof follows Appendix \ref{pflemsec-xinverse} except for in the last steps, we have
\begin{align}\label{pfnoise-matinverse01}
    &\tildy^T(\boldsymbol{X}\boldsymbol{X}^T)^{-1} = \tildy^T\boldu_0^{-1} - \begin{bmatrix}\normeta \tilds & \tildh & \tilds \end{bmatrix} \boldsymbol{A}^{-1}\begin{bmatrix}
    \normeta\boldsymbol{y}^T\\ 
    \boldsymbol{y}^T\\
    \boldsymbol{d}^T
    \end{bmatrix} \boldutau^{-1},
    %& = \tildy^T\boldu_0^{-1} - \begin{bmatrix}\normeta \tilds & \tildh & \tilds \end{bmatrix} \boldsymbol{A}^{-1}\begin{bmatrix}
    %\normeta\boldsymbol{y}^T\\ 
    %\boldsymbol{y}^T\\
    %\boldsymbol{d}^T
    %\end{bmatrix} \boldutau^{-1}
\end{align}
where $\boldsymbol{A}^{-1}$ is
\begin{align*}
    \frac{1}{D}\begin{bmatrix}
    (h+1)^2-st & \normeta(st-h-h^2) & -\normeta s \\
    -\normeta s & h+1+\normeta^2 s & -s \\
    \normeta(st-h-h^2) & \normeta^2h^2-t(1+\normeta^2s) & h+1+\normeta^2 s
    \end{bmatrix},
\end{align*}
with $D = s(\normeta^2 - t) + (h+1)^2$. Then plugging the expression above in \eqref{pfnoise-matinverse01} completes the proof.
\end{proof}

We now prove Lemma \ref{lem-ineqforunoise}. We first start from a lemma bounding $\Vert \tildy + \boldy \Vert_2^2$ and $\Vert \tildy - \boldy \Vert_2^2$.
\begin{lemma}
\label{lem-ynormnoise}
Assuming the probability $\gamma$ of a label flipping  is small enough such that $1-\gamma \ge 1-(1/C_0)$ for some large constant $C_0$, there exist large constants $C_1, C_2 > 1$ such that the event
\begin{align}
\label{eq:eventvnorm}
    \evente_y:=\Big\{ \Vert \tildy + \boldy \Vert_2^2 \ge 4(1-\frac{1}{C_1})n ~~ \text{and} ~~ \Vert \tildy - \boldy \Vert_2^2 \le \frac{4}{C_1}n\Big\},
\end{align}
holds  with probability at least $1-4e^{-\frac{n}{C_2}}$.
\end{lemma}
\begin{proof}
We first look at $(\tilde{y}_i + {y}_i)^2$, which evaluates to either $4$ or $0$. Since bounded, these variables are independent sub-Gaussians. The mean of $\Vert \tildy + \boldy \Vert_2^2$ is $4(1-\gamma)n$. Therefore, Hoeffding's bound \citet[Ch. 2]{wainwright2019high} gives
\begin{align*}
    \Prob \Big(|\Vert \tildy + \boldy \Vert_2^2 - 4(1-\gamma)n| \ge t\Big) \le 2\exp\Big(-\frac{t^2}{Cn}\Big).
\end{align*}
We complete the proof by setting $t = \frac{n}{C_3}$ for a large enough constant $C_3$. $(\tilde{y}_i - {y}_i)^2$ also evaluates to either $4$ or $0$ and the mean of $\Vert \tildy - \boldy \Vert_2^2$ is $4\gamma n$. Thus, we can repeat the previous derivation to obtain the advertised results.
\end{proof}
Now we are ready to prove Lemma \ref{lem-ineqforunoise}.
\begin{proof}[Proof of Lemma \ref{lem-ineqforunoise}.]
The bounds for $\tildh$, $\tildtilds$ and the upper bound for $\tilds$ follow exactly as in Lemma \ref{lem-ineqforu} since $\Vert \tildy \Vert_2^2=n$ same as $\|\y\|_2^2=n$. We now derive the lower bound for $\tilds$. We will need the following standard lemma (here adapted from \citet[Lemma 2]{muthukumar2020classification}) to bound quadratic forms of a Wishart matrix.
\begin{lemma}\label{lem:wishartineq}
Define $\dprimen := (p-n+1)$. Let matrix $\boldsymbol{M} \sim \text{Wishart}(p, \mathbf{I}_n)$. For any unit-frobenius norm vector $\boldsymbol{v}$ and any $t >0$, we have
\begin{align*}
    \Prob\Big(\frac{1}{\boldsymbol{v}^T\boldsymbol{M}^{-1}\boldsymbol{v}} > \dprimen + \sqrt{2t\dprimen}+2t\Big) &\le e^{-t}\\
    \Prob\Big(\frac{1}{\boldsymbol{v}^T\boldsymbol{M}^{-1}\boldsymbol{v}} < \dprimen - \sqrt{2t\dprimen}\Big) &\le e^{-t},
\end{align*}
provided that $\dprimen > 2\max\{t, 1\}$.
\end{lemma}
We use the parallelogram law to write
\begin{align*}
    \tildy^T\boldu_0^{-1}\boldy = \frac{1}{4}\Big((\tildy+\boldy)^T\boldu_0^{-1}(\tildy+\boldy) - (\tildy-\boldy)^T\boldu_0^{-1}(\tildy-\boldy)\Big).
\end{align*}
Let $t = \log n$ and recall that $\dprimen > Cn\log n$ for a sufficiently large constant $C$. To lower bound $\tilds$, conditioned on event $\evente_y$, we have with probability at least $1 - \frac{1}{n}$,
\begin{align*}
   &\tildy^T\boldu_0^{-1}\boldy \ge \frac{1}{4} \Big(\frac{4(1-1/C_1)n}{(\dprimen+\sqrt{2\log(n)\dprimen} + 2\log(n))} - \frac{(4/C_1)n}{(\dprimen-\sqrt{2\log(n)\dprimen})}\Big) \notag\\
    &\ge \frac{(1-1/C_1)n(\dprimen -\sqrt{2\log(n)\dprimen}) - (1/C_1)n(\dprimen + \sqrt{2\log(n)\dprimen} + 2\log(n))}{(\dprimen-\sqrt{2\log(n)\dprimen})(\dprimen+\sqrt{2\log(n)\dprimen})} \\
    & \ge \frac{(1-1/C_3)n\dprimen}{C_4\dprimen^2}\\
    & \ge \frac{n}{C_5p},
\end{align*}
where we replaced $\log(n)$ with $\dprimen/C$ using the fact that $\dprimen > Cn\log n$ for a sufficiently large constant $C$ above. Let $\evente$ be the desired event that $\tildy^T\boldu_0^{-1}\boldy \ge n/(Cp)$. We then complete the proof by adjusting the probability using $\Prob(\evente^c) \le \Prob(\evente^c|\evente_y)+\Prob(\evente_y^c) \le (1/n) + 4\exp(-n/c_1) \le c_2/n$.
\end{proof}

\section{\new{On linear separability of GMM}}\label{sec:lin_sep}
{The main result of this section Lemma \ref{lem:lin_sep_gmm} proves that GMM data are linearly separable with high-probability as long as $p>n+2$. The arguments presented  are pretty standard in the literature, but included here for completeness. Sharp separability thresholds for the GMM have been recently derived in \citet{deng2019model}.

We will first need the following technical lemma that lower bounds the minimum singular value of a non-zero mean isotropic Gaussian matrix. The result is a minor extension of the standard proof using Gordon's Gaussian min-max inequality for the case of a centered isotropic Gaussian matrix (e.g. see \citet[Exercise 7.3.4]{vershynin2018high}).

\begin{lemma}\label{eq:gordon}
Let $Q\in\R^{p \times n}$ a matrix with IID standard normal entries and $\mathbf{y}\in\R^n$, $\etab\in\R^p$ fixed vectors. Consider the matrix  $\A=\etab\mathbf{y}^T+\Q$. For every $t>0$ it holds that
\begin{align}
    \min_{\|\ub\|_2=1}\|\A\ub\|_2 \geq \sqrt{p-2} - \sqrt{n} - t
\end{align}
with probability at least $1-4e^{-t^2/8}$.
\end{lemma}
\begin{proof} We now prove the lemma using Gordon's Gaussian comparison inequality \citep{gordon1985some}. Specifically, we apply a version that appears in \citet{thrampoulidis2015regularized}. We start by writing 
\begin{align}
    \Phi(\A):=\min_{\|\ub\|_2=1}\|\A\ub\|_2  =
    \min_{\|\ub\|_2=1}\max_{\|\w\|_2=1} \w^T\Q\ub + (\w^T\etab)(\y^T\ub)\notag
\end{align}
Now, following \citet[Thm. 3(i)]{thrampoulidis2015regularized} we focus on the following auxiliary problem where $\g\in\R^n$ and $\h\in\R^d$ have iid standard normal entries:
\begin{align}
    \phi(\g,\h):=\min_{\|\ub\|_2=1}\max_{\|\w\|_2=1} \h^T\w + \g^T\ub + (\w^T\etab)(\y^T\ub).\notag
\end{align}
By decomposing $\w=\alpha \frac{\etab}{\|\etab\|_2} + \mathbf{P}_{\etab}^\perp\w$ for $\alpha:=\frac{\etab^T\w}{\|\etab\|_2}\in[0,1]$ and $\mathbf{P}_{\etab}^\perp=\mathbf{I}_d-\frac{\etab\etab^T}{\|\etab\|_2^2}$, we can see that
\begin{align}
    \phi(\g,\h)&= \min_{\|\ub\|_2=1}\max_{\alpha\in[0,1]}~ \|\mathbf{P}_{\etab}^\perp\h\|_2\sqrt{1-\alpha^2} + \alpha\frac{\etab^T\h}{\|\etab\|_2} + \g^T\ub + (\y^T\ub)\alpha\|\etab\|_2\\
    &\geq \min_{\|\ub\|_2=1} \|\mathbf{P}_{\etab}^\perp\h^T\|_2  + \g^T\ub\\ &=\|\mathbf{P}_{\etab}^\perp\h\|_2  -\|\g\|_2
\end{align}
But now from standard concentration arguments (e.g. see \citet[Lemma B.2]{oymak2013squared}, for all $t>0$ with probability at least $1-2e^{-t^2/2}$ it holds that $\|\mathbf{P}_{\etab}^\perp\h\|_2  -\|\g\|_2\geq \sqrt{p-2}-\sqrt{n}-2t$. We now invoke Gordon's inequality to complete the proof:
\begin{align}
    \Pr\left(\Phi(\A)\leq \sqrt{p-2}-\sqrt{n}-t\right) \leq 2\Pr\left(\phi(\g,\h)\leq \sqrt{p-2}-\sqrt{n}-t\right) \leq 4e^{-t^2/8}.
\end{align}
\end{proof}

We are now ready to state and prove the main result of this section.
\begin{lemma}\label{lem:lin_sep_gmm}
Let training data $\{(\x_i,y_i)\}_{i\in[n]}$ be generated from the GMM in Equation \eqref{eq-GM}. Assume $p>n+2+t$ for some $t>0$. Then with probability at least $1-4e^{-t^2/8}$ the following statements hold:

    \noindent(i) The min-norm interpolator is feasible, i.e. there exists  $\betab\in\R^d$ such that for all $i\in[n] : y_i= \x_i^T\betab$.
    
    \noindent(ii) The training data are linearly separable, i.e. there exists  $\betab\in\R^d$ such that for all $i\in[n] : y_i(\x_i^T\betab)\geq 1$.

\end{lemma}
\begin{proof}To prove the first statement we need to show that the feature matrix $\X\in\R^{n\times d}$ has full row-rank with high probability. Equivalently, we show that $\min_{\|\ub\|_2=1}\|\X^T\ub\|_2>0$ with high-probability. This is a direct application of Lemma \ref{eq:gordon} above for $\A=\X^T$.

Now we prove the second statement. From part (i), there exists $\betab$ such that $y_i=\x_i^T\betab, i\in[n]$. Since $y_i\in\{\pm1\}, i\in[n]$ it then holds that $y_i(\x_i^T\betab)=1, i\in[n]$. Thus, the same vector $\betab$ from part (i) that interpolates the data is also a linear separator.
\end{proof}
}

\end{document}